\documentclass[preprint,12pt,sort&compress]{elsarticle}

\usepackage[utf8]{inputenc}
\usepackage{bbold}
\usepackage{latexsym,amsmath,amscd,amssymb}
\usepackage{enumerate}
\usepackage{wrapfig}
\usepackage{graphicx}
\usepackage{framed}
\usepackage{comment}
\usepackage{booktabs}
\usepackage{url}
\usepackage{xcolor}
\usepackage[all]{xy}
\usepackage{makecell}




\newtheorem{theorem}{Theorem}

\newtheorem{definition}[theorem]{Definition}

\newtheorem{corollary}[theorem]{Corollary}

\newtheorem{lemma}[theorem]{Lemma}

\newtheorem{proposition}[theorem]{Proposition}

\newtheorem{remark}[theorem]{Remark}

\numberwithin{theorem}{section}

\newenvironment{proof}[1][Proof]{\textbf{#1.} }{\ \rule{0.5em}{0.5em}}

\newcommand{\pp}[2]{\frac{\partial #1}{\partial #2}}

\journal{Neural Networks}

\begin{document}
	\begin{frontmatter}
		
		\title{Latent Lie-Poisson Neural Networks (LLPNNs): Discovering the motion of Lie-Poisson systems through observable data and latent dynamics}
		
		\author[ua]{Vakhtang Putkaradze\corref{cor1}}
		\ead{vputkaradze@ua.edu}
		\cortext[cor1]{Corresponding author}
		
		\affiliation[ua]{organization={Department of Mathematics, The University of Alabama},
			addressline={Gordon Palmer Hall},
			city={Tuscaloosa},
			state={AL},
			postcode={35401},
			country={USA}}
		
		\begin{abstract}
		Structure-preserving neural networks are essential for the long-term prediction of Hamiltonian systems from data. Many important Hamiltonian systems in mechanics and control admit symmetry reduction to Lie--Poisson systems, including rigid bodies, underwater vehicles, fluids, plasmas, and optimal control problems. A fundamental challenge in learning such systems is that their dynamics evolve in momentum variables that are typically unobservable, while available data consist only of observable quantities such as configurations and velocities. In optimal control applications, the situation is further complicated because the latent variables contain unobservable co-states (Lagrange multipliers) and the reduced Hamiltonian may be degenerate, preventing the existence of a corresponding Lagrangian and rendering encoder-decoder approaches inapplicable.
		
		We introduce \emph{Latent Lie--Poisson Neural Networks} (LLPNNs), a structure-preserving framework for learning Lie--Poisson dynamics directly from observable data. The proposed approach exploits three geometric ingredients: (i) learning either a Hamiltonian decoder or a pseudo-Lagrangian encoder on the active variables, (ii) constructing latent trajectories through a universal Noether invariant arising from Lie--Poisson symmetry reduction, and (iii) reconstructing observable and latent dynamics through Lie--Poisson flows combined with Magnus-based Lie-group updates. The resulting method preserves the underlying geometric structure and remains applicable to both regular and degenerate Hamiltonian systems.
		
		We demonstrate the method on three representative examples: a generalized rigid body on $\mathrm{SO}(3)$, Kirchhoff's underwater vehicle on $\mathrm{SE}(3)$, and an optimal-control problem for interacting vehicles on $\mathrm{SE}(2)^N$. Numerical experiments show excellent long-term predictive accuracy, strong robustness to noise, and competitive performance using only modest datasets and lightweight neural-network architectures.
		\end{abstract}
		
		
	\end{frontmatter}
	\section{Introduction and Problem Statement}

	\subsection{Physics-Informed Learning of Hamiltonian Systems}

	\paragraph{Previous work on learning general dynamical systems}
	
	The goal of data-driven methods of modeling is to 
	approximate the governing dynamics directly from observations. 
	Physics-Informed Neural Networks (PINNs) incorporate differential equations,
	initial conditions, and boundary conditions into the training loss and use
	automatic differentiation to evaluate the corresponding residuals
	\citep{raissi2019physics,karniadakis2021physics,cuomo2022scientific}.
	This approach has been successfully applied to a broad range of scientific
	computing problems.
	
	Neural Ordinary Differential Equations (Neural ODEs)
	\citep{chen2018neural} represent the vector field
	$\dot{\mathbf u}=\mathbf f(\mathbf u,t)$ by a neural network and
	recover trajectories through numerical integration. By training the vector
	field directly, Neural ODEs provide a flexible framework for learning
	continuous-time latent dynamics and have inspired numerous extensions in
	scientific machine learning.
	
	A complementary paradigm is operator learning. Deep Operator Networks
	(DeepONets) \citep{lu2021learning} approximate nonlinear operators mapping
	input functions to solution fields and can generalize across large families
	of initial conditions, forcing terms, or model parameters. Their ability to
	perform rapid inference without retraining has made them attractive for
	many computational physics applications.
	
	Although these methods often achieve excellent predictive performance, they
	do not generally preserve the geometric structures underlying Hamiltonian and
	Poisson systems. Consequently, long-term predictions may violate conserved
	quantities and drift away from physically admissible trajectories. This
	observation has motivated the development of geometric machine-learning
	architectures that explicitly incorporate Hamiltonian, symplectic, or
	Poisson structure. 
	
	A more fundamental limitation arises in many control and symmetry-reduced
	mechanical systems. In these problems the governing equations are written
	for latent momentum variables, whereas the available data typically consist
	only of observable configuration and velocity variables. Consequently, the
	observable variables need not form a closed dynamical system. In particular,
	for degenerate Hamiltonians arising in optimal control, there may not exist
	an autonomous evolution equation for the observable variables alone, as we demonstrate on explicit examples in this paper. This
	hidden-state dynamics that is governing the evolution of the observable data motivates the framework developed in this paper.
	
	\paragraph{Previous work in Structure-Preserving Data-Driven Modeling: Continuous Approach} 
	One dominant strategy in this field aims to directly infer either the Hamiltonian function or the underlying Poisson bracket from experimental data. For canonical systems, this paradigm was established by \emph{Hamiltonian Neural Networks (HNNs)} \citep{greydanus2019hamiltonian}, which approximate $H(\mathbf{q},\mathbf{p})$ by training a network to satisfy the relations in \eqref{canonical_system_gen}. By embedding this geometric structure directly into the loss function, HNNs achieve significantly greater robustness and trajectory accuracy than standard black-box neural networks. This formulation has since been expanded to accommodate chaotic regimes and adaptive parameters \citep{han2021adaptable}, with formal proofs of the existence of such learned Hamiltonians detailed in \citep{david2021symplectic}. 
	
	An alternative formulation is offered by \emph{Lagrangian Neural Networks (LNNs)} \citep{cranmer2020lagrangian}, which model the Euler-Lagrange equations directly in coordinate-velocity space $(\mathbf{q}, \dot{\mathbf{q}})$, thereby eliminating the need for a Legendre transform to momentum space. For non-canonical structures, direct learning of the vector fields was investigated in \citep{vsipka2023direct}, though rigorously enforcing the Jacobi identity remains a major open challenge. More recently, kernel ridge regression was proposed to learn Hamiltonians of Poisson systems from noisy observations \citep{hu2025global}. However, a key limitation of these continuous methods is that the true Hamiltonian is only identifiable up to the addition of an arbitrary function of the Casimirs, which complicates its unique isolation using standard machine learning optimization.
	
	Furthermore, continuous-time methods assume that once the vector field is learned, trajectories can be obtained via standard numerical integration. However, simulating Hamiltonian systems over extended time horizons with generic numerical solvers typically violates fundamental conservation laws. Preserving these geometric properties requires the use of variational integrators \citep{marsden2001discrete,leok2012general,hall2015spectral}. While these specialized integrators preserve momentum-like maps to machine precision, they are often computationally expensive compared to standard explicit solvers.
	
	\paragraph{Learning Hamiltonian systems through structure-preserving discrete maps}
	
	A major research direction in geometric machine learning focuses on designing data-driven architectures that respect the underlying mathematical structures of physical systems. The majority of existing works address \emph{canonical} Hamiltonian systems, where the state vector is defined on a $2n$-dimensional phase space as $\mathbf{u} = (\mathbf{q}, \mathbf{p})$, with $\mathbf{q}$ and $\mathbf{p}$ representing the generalized coordinates and momenta, respectively. Under this formulation, the state's time evolution is governed by a Hamiltonian function $H(\mathbf{q},\mathbf{p})$ through the vector field $\dot{\mathbf{u}} = \mathbf{f}(\mathbf{u})$, where the operator $\mathbf{f}$ is expressed as:
	\begin{equation}
		\mathbf{f} = \mathbb{J} \nabla_{\mathbf{u}} H \, , \quad  
		\mathbb{J} = 
		\left( 
		\begin{array}{cc}
			0 & \mathbb{I}_n  
			\\
			- \mathbb{I}_n & 0
		\end{array}
		\right) \,, \quad \Leftrightarrow \quad \dot{\mathbf{q}} = \frac{\partial H}{\partial \mathbf{p}} \, , \quad 
		\dot{\mathbf{p}} = - \frac{\partial H}{\partial \mathbf{q}} \,,
		\label{canonical_system_gen}
	\end{equation}
	and $\mathbb{I}_n$ denotes the $n \times n$ identity matrix. For any arbitrary smooth function $F(\mathbf{q},\mathbf{p})$, its evolution along a trajectory of \eqref{canonical_system_gen} is dictated by the canonical Poisson bracket:
	\begin{equation}
		\frac{dF}{dt} = \{ F, H \} = 
		\frac{\partial F}{\partial \mathbf{q}} \cdot \frac{\partial H}{\partial \mathbf{p}}
		- 
		\frac{\partial H}{\partial \mathbf{q}} \cdot \frac{\partial F}{\partial \mathbf{p}}.
		\label{canonical_bracket}
	\end{equation}
	In addition to the canonical brackets \eqref{canonical_bracket}, there are also non-canonical Poisson brackets, see  Section~\ref{subsec:Lie-Poisson_Systems} for precise definition. The dynamical systems governed by these non-canonical brackets are known as Poisson systems. A key feature of non-canonical Poisson systems is the presence of a non-trivial null space, which yields conserved quantities known as \emph{Casimir invariants} (or \emph{Casimirs}). These invariants are intrinsic to the structure of the Poisson bracket itself and remain conserved regardless of the specific Hamiltonian governing the system. This work focuses on developing data-driven computational methods for a highly important class of these non-canonical Poisson systems.
	
	To overcome the challenges of continuous integration, we adopt a different perspective: learning discrete, structure-preserving phase space transformations directly. For canonical dynamics, Poincar\'e's theorem states that the flow $\boldsymbol{\phi}_t(\mathbf{u})$ mapping initial states to their configurations at time $t$ is a symplectic transformation \citep{arnol2013mathematical,MaRa2013,putkaradze2025concise}. Mathematically, for $\mathbf{u} = (\mathbf{q}, \mathbf{p})$, the flow mapping $\boldsymbol{\phi}_t(\mathbf{u})$ must satisfy:
	\begin{equation}
		\left( \frac{\partial \boldsymbol{\phi}}{\partial \mathbf{u}} \right)^T 
		\mathbb{J}
		\left( \frac{\partial \boldsymbol{\phi}}{\partial \mathbf{u}} \right) = 
		\mathbb{J} \, . 
		\label{symplectic_map_def}
	\end{equation}
	Instead of learning the vector fields and subsequently integrating them, several works propose learning these symplectic maps directly from data.
	
	Early frameworks in this domain include \emph{Symplectic Recurrent Neural Networks (SRNNs)} \citep{chen2020symplectic}, which enforce symplecticity by embedding discretized symplectic numerical schemes into the network architecture. This approach showed that structure preservation significantly mitigates sensitivity to observational noise. Similarly, \emph{Non-Separable Symplectic Neural Networks (NSSNNs)} \citep{xiong2020nonseparable} were developed to handle non-separable Hamiltonians by utilizing the explicit integration algorithms from \citep{tao2016explicit}. HNNs based on symplectic integrators were recently developed to compute the effective Hamiltonians from noisy data \cite{choudhary2026symplectic}. From a Lagrangian perspective, discrete mapping frameworks in the velocity-phase space $(\mathbf{q}, \dot{\mathbf{q}})$ were explored in \citep{sharma2024lagrangian} and recently extended to include non-conservative dissipative and thermodynamic systems \citep{eldred2025variational}.
	
	Further progress in learning symplectic maps has been driven by three main architectures \citep{jin2020sympnets,chen2021data}. First, \emph{SympNets} \citep{jin2020sympnets} construct symplectic maps by composing specialized, structure-preserving neural network layers. Second, \emph{Generating Function Neural Networks (GFNNs)} \citep{chen2021data} parameterize the discrete flow using generating functions, which provides rigorous bounds on long-term tracking errors. Third, \emph{H\'enonNets} \citep{burby2020fast} build symplectic maps by composing H\'enon-like transformations to learn Poincar\'e maps. These architectures have demonstrated exceptional performance in simulating both integrable systems (such as pendulums and orbital mechanics) and chaotic dynamics (such as the three-body problem).

	\paragraph{Previous work on data-based learning Poisson systems} 
	If the symmetry-reduced momenta $\mu$ are known (which is not the case for the problem we consider here), several approaches have been suggested to construct a structure-preserving learning procedure.  
	
	Recent work \cite{eldred2024lie,eldred2025clpnets} exploits the property that the phase flow of any Hamiltonian function constitutes a Poisson map, meaning that it preserves the Poisson bracket \cite{MaRa2013}. These architectures consider flows generated by ``test'' Hamiltonians of special forms, for which the equations of motion simplify and can be solved analytically, providing the mathematical foundation for Lie--Poisson neural network layers.
	
	By composing multiple Poisson maps of this type, the neural network in \cite{eldred2024lie} optimizes the parameters of these test Hamiltonians to match observed phase-space trajectories. This sequence enables the network to reconstruct the flow across the entire phase space with high accuracy, using small datasets and lightweight parameterizations. Crucially, this construction ensures that Casimir invariants are conserved to machine precision, independent of the training data or optimization quality. This work was recently extended to optimal-control problems for interacting autonomous vehicles \cite{huraka2026structure}, where Lie--Poisson transformations were generalized to incorporate individual neural networks depending on the symmetry-reduced momenta. This structure enabled a completeness result for networks built from Lie--Poisson transformations.
	
These previous works highlight the value of exploiting the exact Lie--Poisson structure rather than treating the dynamics as an unconstrained vector-field learning problem. In the present work, we adopt the symmetry group and its Lie--Poisson bracket as known geometric data, while learning the physical or control Hamiltonian from observations. This separation is natural in many physical and control systems, where the symmetry structure is fixed by the configuration space, whereas the Hamiltonian may be unknown or only approximately modeled. For a related discussion in the context of physical systems, see \cite{Tonti2013}.
	
	For non-canonical Poisson systems, the SympNet framework was extended under the name \emph{Poisson Neural Networks (PNNs)} \citep{jin2022learning}. By leveraging the Lie-Darboux theorem, PNNs locally map non-canonical coordinates to canonical ones under the assumption of a constant-rank Poisson structure. However, this strategy faces two major limitations: the Lie-Darboux transformation is typically local rather than global, and PNNs do not guarantee exact preservation of Casimir invariants. Because preserving Casimirs is vital for the statistical validity of long-term ensemble simulations \citep{Dubinkina2007}, there is a strong demand for methods that enforce Casimir conservation to machine precision.
	
An alternative formulation presented in \citep{vaquero2024designing,vaquero2024symmetry} constructs maps that exactly conserve momentum maps and Casimirs by solving approximate Hamilton--Jacobi equations for symmetry-reduced systems. However, solving Hamilton--Jacobi equations for high-dimensional, complex systems is often computationally and analytically intractable. More recently, \emph{LocSympNets} and \emph{SymLocSympNets} \citep{bajars2023locally} were introduced, modeling flows by composing volume-preserving mappings. While these models have proven highly effective for simulating rigid-body dynamics, linear advection, and charged particles in magnetic fields, their theoretical completeness remains an open question. 
	
	In a related study, \cite{mason2023learning} considered the use of purely visual streams to model 3D rotational dynamics by projecting images onto an $SO(3)$ orientation manifold while simultaneously learning the moment-of-inertia tensor. While their technical framework is different from ours, their work highlights the feasibility of learning underlying latent dynamics purely from observable pixel data without relying on full state measurements including the latent variables. 
	
	Another recent work \cite{hu2025global} addressed learning of a general Poisson system using a structure-preserving kernel ridge regression. The method allowed to find Hamilto-
	nian functions on general Poisson manifolds out of datasets made of noisy observations of
	Hamiltonian vector fields. 
	
	It is important to note that all these works assumed that the information about the flow on the momentum (latent) manifold is known. In contrast, we do not assume that knowledge as in practical realizations, especially in control systems, the latent variables are hidden, and only the information about the observable variables is available. This observation motivates our development of the \emph{Latent Lie--Poisson Neural Network
		(LLPNN)}, a framework that learns structure-preserving Lie--Poisson
	dynamics directly from observable variables while reconstructing the
	hidden momentum dynamics.

	\paragraph{Optimal Control via Neural Network Architectures and Symmetry Reduction to Lie-Poisson Systems}
Let us also briefly describe the efforts of using neural networks for optimal control problems. In addition to the paper \cite{huraka2026structure} we have mentioned earlier, several works considered the optimal control problem using neural networks. These models are mainly focused on finding the solutions to the optimal control problems, using neural networks as an alternative way to finding solutions to the equations of optimal control. 
	Neural networks have also been applied extensively to optimal control, although the focus of these works (planning) is different from the goal of this paper (learning from observations).  
	In particular, \cite{effati2013optimal} incorporated the state and co-state equations from Pontryagin's Minimum (or Maximum) Principle (PMP) \cite{gamkrelidze2013principles} as soft constraints within a loss function, utilizing a formulation that conceptually anticipated Physics-Informed Neural Networks (PINNs). To ensure control tractability, \cite{chen2018optimal} deployed input-convex neural networks. Neural network controllers have also been applied to practical engineering challenges, including robotic manipulation \cite{kim2000intelligent} and autonomous aeronautical landing maneuvers \cite{sanchez2018real}. For a comprehensive overview of machine learning paradigms in this domain, we refer the reader to the review by \cite{wang2026optimal}. Despite these notable advancements, existing methodologies generally overlook structure-preserving network designs, thereby failing to exploit the intrinsic geometric structure of Pontryagin's equations. Enforcing this underlying mathematical structure within the network architecture is essential for achieving accurate, long-term dynamical predictions \citep{jin2020sympnets,eldred2024lie}, which motivates the framework introduced in this work. However, the most important feature of the optimal control is that the observable variables are different from the dynamics latent variables carrying the dynamics, which makes direct machine learning challenging. Finding the latent dynamics from observable data is the main point of this paper.

	\subsection{Novelty of this paper} 
	Our paper is different from previous work on the subject as it presents the following novel results: 
	\begin{enumerate}
		\item We design a structure-preserving method (LLPNN) for learning the system based on the observable variables only, 
		\item Our method guarantees conservation of Casimirs with machine precision in the latent space, even though it only uses the data from the observable space. 
		\item LLPNNs are applicable to both regular and degenerate
		Lie--Poisson Hamiltonian systems, including optimal-control
		problems for which a corresponding Lagrangian formulation in terms of observables 
		may not exist. 
	\end{enumerate}
	
	\section{Mathematical Foundations of the Theory} 
	\label{sec:math_foundation} 

	\subsection{A brief introduction to Lie--Poisson systems}
	\label{subsec:Lie-Poisson_Systems}
	
	To develop our framework, we draw upon general properties of Lie--Poisson systems. Here, we present a brief summary; for a complete description, see~\ref{app:math_Lie_Poisson}, as well as \cite{Ho2011_pII,MaRa2013} for detailed expositions.
	
	Lie--Poisson systems are a special class of general Poisson systems. They arise naturally through symmetry reduction with respect to the action of a Lie group, where the resulting Poisson bracket inherits its structure from the underlying Lie algebra.
	
	A general Poisson manifold consists of a smooth manifold $\mathcal{M}$ equipped with a bracket $\{ \cdot , \cdot \}$ that maps a pair of smooth functions $F, G \in C^\infty(\mathcal{M})$ to another smooth function $\{F, G\} \in C^\infty(\mathcal{M})$. To constitute a valid Poisson bracket, it must satisfy the following properties for all $F, G, H \in C^\infty(\mathcal{M})$:
	\begin{enumerate}
		\item \emph{Antisymmetry:} $\{ F, G\} = -\{ G, F\}$, 
		\item \emph{Bilinearity:} $\{ a F + b G, H \} = a\{ F, H\} + b \{G, H\}$ for all real constants $a, b \in \mathbb{R}$, 
		\item \emph{Leibniz rule}: $\{ FG, H\} = G \{ F,H\} + F \{ G,H \}$,
		\item \emph{Jacobi identity}: $\{ \{ F, G\}, H\} + \{ \{ H, F\}, G\} + \{ \{ G, H\}, F\} = 0$. 
	\end{enumerate}
	
	For finite-dimensional systems, Poisson brackets can be expressed in terms of a Poisson tensor matrix $\mathbb{B}(\mu)$ \cite{MaRa2013}:
	\begin{equation}
		\{ F,G \} = \left( \pp{F}{\mu} \right)^T \mathbb{B}(\mu) \pp{G}{\mu} \, , 
	\end{equation}
	where $\mathbb{B}(\mu)$ is an antisymmetric matrix satisfying a system of non-linear first-order partial differential equations required by the Jacobi identity. Because these PDE constraints are generally non-trivial to solve analytically in closed form, one typically verifies whether a given matrix defines a valid Poisson structure rather than deriving a general parameterization.
	
	Lie--Poisson systems emerge when a dynamical system exhibits symmetry under a Lie group action. Suppose the configuration manifold of the system is a Lie group $G$, and the  Lagrangian or Hamiltonian is invariant with respect to the action of $G$. For instance, rigid body motion is modeled on the special orthogonal group $G=\mathrm{SO}(3)$, where rotational invariance holds with respect to the same group. Just as unreduced canonical systems admit both Lagrangian and Hamiltonian formulations, symmetry-reduced systems preserve this duality. The standard (unreduced) Lagrangian description depends on configuration coordinates and velocities $(g, \dot{g}) \in TG$, whereas the canonical Hamiltonian description uses coordinates and momenta $(g, p) \in T^*G$.
	
	The symmetry-reduced Lagrangian formulation is expressed in terms of Lie algebra-valued velocities $\xi \in \mathfrak{g}$, obtained by trivializing the tangent bundle. 
	For matrix Lie groups considered here, these velocities are defined through one of the two standard trivializations,
	\begin{equation}
		\dot{g}=g\xi
		\qquad\text{or}\qquad
		\dot{g}=\xi g .
		\label{eq_g_xi}
	\end{equation}
	The first convention corresponds to the body-frame velocity (left trivialization), $\xi=g^{-1}\dot g$, while the second corresponds to the spatial velocity (right trivialization), $\xi=\dot g g^{-1}$. The sign in the reduced equations depends on whether the Lagrangian is left- or right-invariant.
For a reduced Lagrangian $\ell(\xi)$, the symmetry-reduced equations of motion are given by the \emph{Euler--Poincar\'e equations} \cite{Ho2011_pII,MaRa2013}:
\begin{equation}
	\frac{d}{dt} \pp{\ell}{\xi} \mp \operatorname{ad}^*_\xi \pp{\ell}{\xi} = 0 ,
	\label{EP_eqs}
\end{equation}
where $\operatorname{ad}^*$ denotes the coadjoint operator for the Lie algebra $\mathfrak{g}$, and the minus (\emph{resp.}, plus) sign corresponds to left (\emph{resp.}, right) invariance.

	In Lie--Poisson dynamics, we define the symmetry-reduced momentum $\mu = g^{-1} p$ (for left-invariant systems) or $\mu = p g^{-1}$ (for right-invariant systems), where $\mu \in \mathfrak{g}^*$ belongs to the dual Lie algebra. For example, satellite attitude dynamics are formulated on $\mathrm{SO}(3)$, where $\mu \in \mathfrak{so}(3)^*$ is represented as a three-dimensional vector. The intrinsic equations of motion and Poisson bracket for $\mu$ are given by:
	\begin{equation}
		\frac{d}{dt} \mu \mp \operatorname{ad}^*_{\pp{h}{\mu}} \mu = 0 \, , 
		\label{LP_eqs} 
	\end{equation}
	where we focus primarily on left-invariant systems (right-invariant systems are treated analogously). Equation \eqref{LP_eqs} can be expressed using the Lie--Poisson bracket:
	\begin{equation} 
		\dot{\mu} = \{ \mu, h \} \quad \Leftrightarrow \quad \dot{\mu} = \Lambda(\mu) \nabla h(\mu) \, , \quad 
		\{ F, G \} = \mp \left\langle \mu, \left[ \frac{\partial F}{\partial \mu}, \frac{\partial G}{\partial \mu} \right] \right\rangle ,
		\label{LP_bracket_explicit} 
	\end{equation} 
where $[\cdot,\cdot]$ denotes the Lie algebra commutator. Again, the $-$ sign in the Poisson bracket \eqref{LP_bracket_explicit} corresponds to left-invariant systems, while the $+$ sign corresponds to right-invariant systems. 
	Here, the Poisson tensor $\Lambda(\mu)$ depends linearly on $\mu$, with its structural form determined entirely by the structure constants of the Lie algebra. These brackets often admit Casimir invariants, \emph{i.e.}, functions $C(\mu)$ satisfying $\{ C, h \} = 0$ for any Hamiltonian $h$, which play a central role in our structure-preserving approach.
	
	Optimal control problems with symmetry naturally inherit a Lie--Poisson structure \cite{krishnaprasad1993optimal,Bl2003}, including multi-agent formation control systems \cite{justh2010extremal,justh2015optimality}. Notably, symmetry reduction of controlled Hamiltonian systems yields a complete Lie--Poisson reduction onto the direct product of individual Lie algebras serving as reduced phase spaces. While symmetry reduction in multi-body systems is frequently incomplete, requiring relative Lie group configurations alongside algebra-level dynamics \cite{ellis2010symmetry}, this complexity is avoided in Lie--Poisson optimal control due to the specific algebraic structure of control Hamiltonians \cite{justh2010extremal,justh2015optimality}.
	
	For classical mechanical systems, Lagrangian and Hamiltonian formulations are \emph{mathematically} equivalent via the Legendre transform. However, they differ fundamentally regarding data accessibility, as illustrated in Figure~\ref{fig:observables}. Positions, orientations, and group configurations $g(t)$ are directly observable from measurements. If $g(t)$ is observed, its rate of change $\dot{g}(t)$ and the reduced velocity $\xi(t)$ defined by \eqref{eq_g_xi} are likewise observable.
	
	In contrast, the latent momentum $\mu$ cannot be measured directly without prior knowledge of the system's underlying functional form $\ell(\xi)$ or $h(\mu)$. Thus, empirical simulations of observable states naturally operate in velocity space $\xi$, whereas structure-preserving dynamics are intrinsically formulated in Hamiltonian space $\mu$, where the Poisson/symplectic geometric structure resides \cite{Ho2011_pII,MaRa2013,arnol2013mathematical}.
	
	\begin{figure}[htbp]
		\centering 
		\includegraphics[width=1\textwidth]{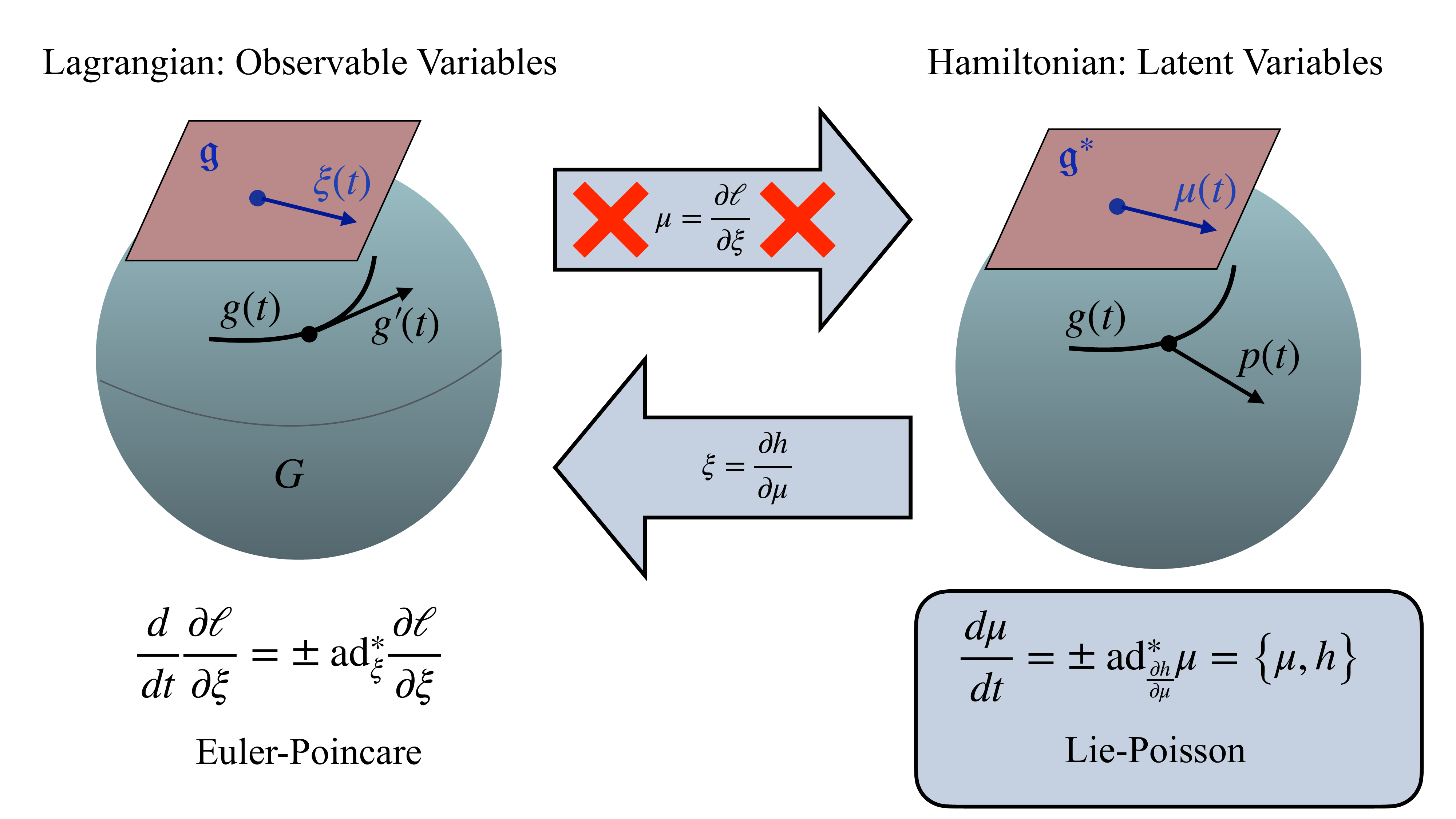}
		\caption{Illustration of the structural relationship between observable configuration/velocity variables and unobservable latent momentum variables.}
		\label{fig:observables} 
	\end{figure}
	
	In principle, when the system is non-degenerate, one can use the reduced Lagrangian and Hamiltonian to define encoder and decoder mappings between observable and latent spaces:
	\begin{equation}
		\mu = \mu(\xi) = \pp{\ell}{\xi} \quad \text{(encoder)}, \qquad 
		\xi = \xi(\mu) = \pp{h}{\mu} \quad \text{(decoder)} \, .  
		\label{encoder_decoder_structure} 
	\end{equation}
	This allows standard latent space architectures to learn bi-directional mappings between observables and latent variables before applying structure-preserving integrators. While this encoder-decoder paradigm \eqref{encoder_decoder_structure} works well for regular mechanical systems, it frequently fails for optimal control systems.
	
	Although symmetry-reduced optimal control problems reduce to Lie--Poisson equations of the form \eqref{LP_bracket_explicit}, they differ in a crucial aspect: optimal control problems frequently yield reduced Hamiltonians $h(\mu)$ that are \emph{degenerate}—specifically, linear in components of $\mu$ associated with drift controls. For example, in the ground-vehicle model studied in Section~\ref{subsec:SE2_drones}, vehicles drive forward at constant speed, cannot slide laterally, and control their turning rate. Consequently, $h(\mu)$ is linear in the forward-drive momentum, independent of the lateral momentum, and non-linear only in the angular momentum component. For the corresponding velocity components, $\xi_i=\mathrm{const}$. The Hessian $\nabla^2 h(\mu)$ is therefore singular, rendering the Legendre transform non-invertible and preventing the existence of a corresponding reduced Lagrangian $\ell(\xi)$. Under such degeneracy, only the decoder $\xi = \xi(\mu)$ exists; the encoder $\mu = \mu(\xi)$ does not. That is, no function $\ell(\xi)$ exists such that the Euler--Poincar\'e equations \eqref{EP_eqs} close autonomously in $\xi$-space. In contrast, the Lie--Poisson equations \eqref{LP_eqs} remain well-defined regardless of Hamiltonian degeneracy.
	
	One might consider regularizing a degenerate Hamiltonian by adding arbitrary functions of Casimir invariants, which do not alter momentum-space dynamics. For instance, in rigid body mechanics where $\boldsymbol{\mu} \in \mathbb{R}^3$, Casimirs are functions of $\|\boldsymbol{\mu}\|$. Adding a quadratic Casimir term $\kappa \|\boldsymbol{\mu}\|^2$ ($\kappa > 0$) can render the Hessian non-singular, locally restoring an invertible Lagrangian $\ell(\xi)$. However, this regularization alters the reconstructed velocity field $\xi$ for a given momentum $\mu$. Because data-driven models must fit observed velocity trajectories, modifying the Hamiltonian with Casimir terms to force invertibility is non-viable.
	
	This lack of an invertible mapping between observables and latent states represents a key challenge for structure-preserving machine learning. Moreover, Section~\ref{subsec:difference_Ham_Lagr} demonstrates that for some degenerate Hamiltonians that are of practical importance for control theory, the evolution in observable space cannot even be expressed as an autonomous differential equation.
	
	To overcome this, our framework demonstrates that latent Lie--Poisson dynamics can be reconstructed and predicted solely from observable configuration and velocity trajectories. By integrating geometric mechanics with structure-preserving deep learning, we propose \emph{Latent Lie-Poisson Neural Networks (LLPNNs)}.
	
	A foundational property of Lie--Poisson systems leveraged in our architecture is the existence of a conserved spatial momentum. By Noether's theorem, left-invariance guarantees that the spatial momentum $p_0 \in \mathfrak{g}^*$ is conserved for all time. This yields the coadjoint reconstruction relationship:
	\begin{equation}
		\frac{d}{dt} \left( \operatorname{Ad}^*_{g(t)^{-1}} \mu(t) \right) = 0 \implies \mu(t) = \operatorname{Ad}^*_{g(t)} p_0 \,,
		\label{eq:coadjoint_reconstruction}
	\end{equation}
	where $\operatorname{Ad}^*_g$ denotes the coadjoint action of $G$ on $\mathfrak{g}^*$ (see Definition~\eqref{Ad_star_def}). Equation \eqref{eq:coadjoint_reconstruction} provides an exact, structure-preserving bridge: it maps the unobservable, time-varying body momentum $\mu(t)$ to a constant spatial momentum vector $p_0$ using only the observed configuration trajectory $g(t)$.

	\begin{remark}[Momentum Gauge Ambiguity and Scaling]
		{\rm 
			When relying solely on observable data, the latent momentum $\mu$ is identifiable only up to an arbitrary multiplicative scaling constant $c > 0$. Rescaling the Hamiltonian as $h_{\text{scaled}}(\mu) = c h(\mu / c)$ preserves the exact velocity dynamics:
			\begin{equation}
				\xi = \frac{\partial}{\partial \mu} \left( c h\left(\frac{\mu}{c}\right) \right) = \nabla h\left(\frac{\mu}{c}\right) .
			\end{equation}
			Because $c h(\mu/c)$ generates identical velocity fields $\xi$, the absolute scale of $\mu$ cannot be uniquely determined from velocity observations alone. To evaluate learned latent states against physical ground-truth trajectories, predicted momenta are rescaled post-hoc using the Hessian trace ratio at the origin or a scale factor estimated from an initial test window.
		}
	\end{remark}

	\subsection{The fundamental difference between evolution in latent and observable variables} 
	\label{subsec:difference_Ham_Lagr} 
	For classical, physically based mechanical systems, one assumes that the Hessian of the Hamiltonian function $h(\mu)$ is non-degenerate. Under this condition, the (inverse) Legendre transform 
	\begin{equation} 
		\ell(\xi) = \langle \xi, \mu \rangle - h(\mu), \quad 
		\xi = \pp{h}{\mu} \leftrightarrow \mu = \mu(\xi)
		\label{inverse_Legendre_transform} 
	\end{equation} 
	maps the Hamiltonian representation $h(\mu)$ to the Lagrangian representation $\ell(\xi)$. The Euler--Poincar\'e equations \eqref{EP_eqs} then constitute a well-defined dynamical system in the observable space. 
	
	\paragraph{Degenerate Hamiltonians and the absence of closed observable dynamics}
	The situation is substantially more difficult when $h(\mu)$ is degenerate. The failure of a Lagrangian formulation for degenerate Hamiltonians is only part of the challenge; in some cases, the observable variables do not even admit a closed autonomous dynamical system.
	
	As an example, consider the Lie--Poisson equations on the dual Lie algebra $\mathfrak{se}(2)^*$ with Hamiltonian \cite{justh2010extremal,justh2015optimality}
	\begin{equation}
		h(\mu) = \mu_1 + \frac12 \mu_3^2 .
		\label{single_particle_SE2_Ham} 
	\end{equation}
	This control Hamiltonian describes the motion of an autonomous vehicle on a plane: motion occurs at constant speed along the main body axis (corresponding to $\mu_1$), no motion is allowed along the transverse axis (corresponding to the absence of $\mu_2$), and control is exerted via the rotation angle (expressed by $\mu_3$). We will study this multi-vehicle system in substantially more detail in Section~\ref{subsec:SE2_drones}. 
	
	The Lie--Poisson equations \eqref{LP_bracket_explicit} on $\mathfrak{se}(2)^*$ for the Hamiltonian \eqref{single_particle_SE2_Ham} take the form
	\begin{equation}
		\dot{\mu}_1 = -\mu_2 \mu_3,
		\qquad
		\dot{\mu}_2 = \mu_1 \mu_3,
		\qquad
		\dot{\mu}_3 = \mu_2 .
	\end{equation}
	These equations are well defined; in fact, \eqref{LP_bracket_explicit} defines $\dot\mu = \mathbf{f}(\mu)$ for any time-independent Hamiltonian $h(\mu)$. Moreover, the solutions for $\mu$ are readily mapped into the observable domain via
	\begin{equation}
		\xi = \pp{h}{\mu}
		=
		\begin{pmatrix}
			1 \\
			0 \\
			\mu_3
		\end{pmatrix},
	\end{equation}
	reflecting constant velocity along the forward axis, zero motion along the transverse axis, and control over the rotation angle. This mapping exists for every Hamiltonian, whether degenerate or not. However, we can readily prove that for the degenerate Hamiltonian \eqref{single_particle_SE2_Ham}, no autonomous differential equation exists for $\xi$, as illustrated by the following lemma. 
	
\begin{lemma}
	\label{lem:non_existence_xi}
	For the Hamiltonian \eqref{single_particle_SE2_Ham}, there is no well-defined autonomous first-order system
	\[
	\dot{\xi}=\mathbf{g}(\xi)
	\]
	describing the evolution in the purely observable velocity space.
\end{lemma}
	\noindent
	\textbf{Proof.} 
	Since $\xi_3 = \mu_3$, we obtain
	\begin{equation}
		\dot{\xi}_3 = \dot{\mu}_3 = \mu_2 .
	\end{equation}
	However, the value of $\mu_2$ cannot be recovered from the observable variables $\xi = (1,0,\mu_3)$. Consequently, two different momentum states may produce identical observable velocities while possessing different time derivatives. For example,
	\begin{equation}
		(\mu_1,\mu_2,\mu_3) = (1,0,2) \quad \text{and} \quad (\mu_1,\mu_2,\mu_3) = (1,1,2)
	\end{equation}
	both yield
	\begin{equation}
		\xi = (1,0,2),
	\end{equation}
	but lead to
	\begin{equation}
		\dot{\xi}_3 = 0 \quad \text{and} \quad \dot{\xi}_3 = 1,
	\end{equation}
	respectively. Therefore, there does not exist a first-order autonomous system $\dot{\xi} = \mathbf{g}(\xi)$ describing the evolution of the observable variables alone. $\blacksquare$
	
	\paragraph{Consequences of non-existence of a unique vector field in observable space} 
	While the result of Lemma~\ref{lem:non_existence_xi} may seem surprising, it stems from the fundamental fact that observable velocity is not a state variable, but merely an output of the latent momentum dynamics. This distinction is critical for control problems involving degenerate Hamiltonians. Standard data-driven methods, such as Neural ODEs or DeepONets, attempt to learn the vector field directly in terms of observable variables. However, to close the dynamics in $\xi$ and resolve state ambiguity, one must incorporate additional structural knowledge—such as restricting motion to a specific symplectic leaf defined by a fixed Casimir value. Consequently, while Neural ODEs and DeepONets perform reasonably well when the Hamiltonian is non-degenerate, as illustrated by our benchmark examples of generalized rigid body dynamics in Section~\ref{subsec:rigidbody_so3} and the generalized Kirchhoff underwater vehicle problem in Section~\ref{subsec:kirchhoff_se3}, they fail rapidly when the Hamiltonian is degenerate, as demonstrated for the autonomous vehicles studied in Section~\ref{subsec:SE2_drones}. In contrast, the LLPNN framework developed in this work reconstructs the latent Lie--Poisson state, where dynamics obeys a well-defined differential equation regardless of Hamiltonian degeneracy, relying solely on observable data and intrinsic geometric properties of Lie--Poisson systems.
	
		\subsection{Observable Data, System Input and Problem Formulation} 
	To design a neural network for learning latent dynamics, we must account for the fact that, in general, no invertible mapping exists between observable states (positions and velocities) and latent momenta. Consequently, we must infer the functional form of the Hamiltonian $h(\mu)$ directly from observable measurements of group elements and velocities. More precisely, we assume access to a dataset comprising $M$ short trajectories. For each trajectory $j=1,\dots,M$, we observe configuration group elements $(g_1^j,\dots,g_N^j)$, alongside their corresponding Lie algebra-valued velocities $(\xi_1^j,\dots,\xi_N^j)$ across $N$ discrete time instances. 
	
	For each trajectory, we treat this dataset as two separate collections: one containing configuration states $g_i \in G$, and another containing Lie algebra-valued velocities $\xi_i \in \mathfrak{g}$. In practice, velocities can often be reconstructed directly from group elements—for instance, the Cayley transform provides a map from the special orthogonal group $\mathrm{SO}(3)$ to its Lie algebra $\mathfrak{so}(3)$ \cite{iserles2001cayley}, while matrix logarithms can be employed for general Lie groups. In such cases, velocity errors stem from both deterministic discretization errors and stochastic measurement noise in the configuration states. Alternatively, velocity data $\xi$ may be acquired independently of orientation data, such as via onboard telemetry from gyroscopes. Thus, in what follows, we assume that both group element measurements $g^j(t_i) \simeq g_i^j$ and velocity measurements $\xi^j(t_i) \simeq \xi_i^j$ are available for $i = 1, \dots, N$ and $j=1, \dots, M$. When analyzing noise effects, we assume the error distributions in these two datasets to be independent.
	
	Given these available measurements, the core problem addressed in this work is formulated as follows. 
	
	\paragraph{Problem Formulation}
	Consider a dataset $(g_1^j,\dots,g_N^j)$ of configuration group elements $g_i^j \in G$ and corresponding symmetry-reduced velocities $\xi_1^j, \dots, \xi_N^j \in \mathfrak{g}$ sampled across $N$ discrete time points for $M$ trajectories. Assume that the dynamics in the latent space $\mu \in \mathfrak{g}^*$ is governed by the Lie--Poisson equations \eqref{LP_bracket_explicit} with a known Lie--Poisson matrix $\Lambda(\mu)$ but an unknown Hamiltonian function $H(\mu)$. Our goal is to learn this underlying dynamics and reconstruct both the observable trajectories $(\xi, g)$ and the latent momentum trajectories $\mu$.

	\section{Latent Lie-Poisson Neural Networks: two approaches} 
	\subsection{LLPNNs: The decoder (Hamiltonian) approach} 
	\label{subsec:decoder} 
	
	\subsubsection{Mathematical background} 
	In this first approach, we formulate the learning method for the Hamiltonian function $h(\mu)$ based on Noether's theorem. The primary advantage of this procedure is that it operates directly in the dynamical space of the momenta $\mu$. Because we only learn the decoder $h(\mu)$, we never need to deduce a corresponding encoder $\ell(\xi)$, thus bypassing the need to consider whether a Lagrangian formulation exists.
	
	We begin by identifying degenerate directions, which yield constant velocities. These degenerate directions correspond to components of $\xi = \pp{h}{\mu}$ that are either identically zero or constant. Identifying these directions directly from data is a natural task in systems with time-independent Hamiltonians. We formalize this distinction as follows:
	
	\begin{definition}
		Suppose coordinate component $k$ exhibits a constant velocity, i.e., $\xi_k = \pp{h}{\mu_k} = C_k$ (which may be zero) for all $\mu$. We define such coordinates as \emph{passive}, while all remaining coordinates are defined as \emph{active}. 
	\end{definition} 
	
\begin{remark}
	{\rm
		Throughout this paper, we assume that the reduced Hamiltonian $h(\mu)$ is time-independent. One could also consider time-dependent Hamiltonians for which passive velocity components $\xi_k^{\mathrm{pass}}=\xi_k^{\mathrm{pass}}(t)$ are prescribed functions of time. Identifying such components from data would introduce an additional model-selection problem, and we do not consider this extension here.
	}
\end{remark}
	A standard method for discovering constant linear combinations of functions from data is Singular Value Decomposition (SVD) \cite{brunton2022data}. If a linear combination of coordinates remains constant during evolution, we can perform a change of variables to align with these new coordinates, placing the passive coordinates at the end of the coordinate vector. In practice, degenerate coordinates can often be identified directly; for instance, in the ground vehicle model discussed in Section~\ref{subsec:SE2_drones}, they are readily observed by computing the body-frame angular velocity from orientation measurements.
	
	Because the Hamiltonian depends linearly on passive variables, the neural network parameterizing the Hamiltonian needs to depend only on the active variables. Suppose the passive components are defined by $\xi_k^{\text{pass}} = a_k \in \mathbb{R}$, which are directly observed from data. We posit the following structure for the Hamiltonian: 
	\begin{equation}
		h_{NN}(\mu) = \sum_k a_k \mu_k^{\text{pass}} + F_{NN}(\mu^{\text{act}}) \, . 
		\label{h_act_pass} 
	\end{equation}
	Learning the Hamiltonian in \eqref{h_act_pass} proceeds in three steps, while dynamic trajectory reconstruction requires an additional step utilizing a Magnus expansion combined with Noether's theorem.
	
	\paragraph{Step 1.} 
	Given observable configuration data $(g_1^j, \ldots, g_N^j)$ for trajectory $j \in \{ 1, \ldots, M\}$, conservation law \eqref{eq:coadjoint_reconstruction} guarantees the existence of a trajectory-specific constant vector $p_0^j$ such that the latent momentum trajectory satisfies: 
	\begin{equation} 
		\mu_i^j = \operatorname{Ad}^*_{g_i^j} p_0^j \, . 
		\label{eq:mu_integral_j} 
	\end{equation}
	We therefore assign a constant vector $p_0^j$ to each trajectory $j = 1, \ldots, M$. These vectors, optimized during network training, allow exact reconstruction of latent trajectories. Note that in \eqref{eq:mu_integral_j}, $\mu$ includes all components, both active and passive. 
	
	\paragraph{Step 2} 
	Given observable velocity data $(\xi_1^j, \ldots, \xi_N^j)$ for the same trajectories, we parameterize the Hamiltonian using the neural network $h = h_{NN}(\mu)$ and compute predicted velocities using a projection operator $\mathbb{P}^{\text{act}}$ onto the active components: 
	\begin{equation}
		\mathbb{P}^{\text{act}} \widetilde{\xi}_i^j = \mathbb{P}^{\text{act}} \pp{h_{NN}}{\mu} \left( \operatorname{Ad}^*_{g_i^j} p_0^j \right) \, . 
		\label{eq:xi_i_approx} 
	\end{equation}
	Projecting onto active components via $\mathbb{P}^{\text{act}}$ is sufficient, as passive velocity components are satisfied automatically by the design of \eqref{h_act_pass}. 
	
	\paragraph{Step 3} 
	We define the loss function with respect to the network weights $\mathbf{W}$ and the set of trajectory constants $\bar{p}_0 = (p_0^1, \ldots, p_0^M)$: 
	\begin{equation}
		L(\mathbf{W}, \bar{p}_0) = \sum_{j=1}^M \sum_{i=1}^N \left\| \mathbb{P}^{\text{act}} \left( \xi_i^j - \widetilde{\xi}_i^j \right) \right\|^2 \, , \quad \text{with } \widetilde{\xi}_i^j \text{ defined by \eqref{eq:xi_i_approx}} . 
		\label{eq:latent_loss} 
	\end{equation}
	Optimizing loss function \eqref{eq:latent_loss} yields an approximation for the Hamiltonian $h(\mu)$. While optimization also determines the constant vectors $p_0^j$ for training trajectories, these values are dataset-specific and not used directly for predicting new trajectories. 
	
	\subsubsection{Reconstruction Algorithm: Hamiltonian Approach} 
	
	\paragraph{Step 1: Finding Noether's constant} 
	Once $h_{NN}$ has been learned, we use it to reconstruct system trajectories. In momentum space, trajectory reconstruction is straightforward, requiring forward integration of \eqref{LP_eqs} using structure-preserving integrators. However, our primary goal is to reconstruct dynamics in observable space. When $h(\mu)$ is regular with no passive directions, one can use the Legendre transform to map Hamiltonian dynamics into Lagrangian dynamics. For degenerate Hamiltonians, however, constructing a Lagrangian that directly governs velocity evolution is generally impossible.
	
	Rather than attempting to find a Lagrangian that may not exist, we operate entirely within the Hamiltonian framework. Given initial observations $\xi_0$ and $g_0$ for a trajectory to be reconstructed, we determine the conserved spatial momentum $p_0$ satisfying: 
	\begin{equation} 
		\mathbb{P}^{\text{act}} \xi_0 = \mathbb{P}^{\text{act}} \pp{h_{NN}}{\mu} \left( \operatorname{Ad}^*_{g_0} p_0 \right) .
		\label{init_cond_reconstruction} 
	\end{equation}
	Although solving \eqref{init_cond_reconstruction} from a single snapshot is theoretically possible for regular Hamiltonians, fitting $p_0$ over a short initial window of $N_{\text{fit}}$ states significantly improves robustness against measurement noise. Because accurate estimation of $p_0$ is crucial for long-term trajectory prediction, relying on a single data point often leads to numerical degradation. We thus determine $p_0$ from a short history window rather than a single point to improve the stability and accuracy of the method. 
	
	\paragraph{Step 2: Finding the update for the group element} 
	Once $p_0$ is determined from Step 1, let $\Delta_k = g_k^{-1} g_{k+1}$ denote the relative group evolution element, yielding the configuration update $g_{k+1} = g_k \Delta_k$. Using Noether's conservation law and the coadjoint anti-homomorphism property ($\operatorname{Ad}^*_{gh} = \operatorname{Ad}^*_{h} \operatorname{Ad}^*_{g}$), we obtain: 
	\begin{equation} 
		\widetilde{\mu}_{k+1} (\Delta_k) = \operatorname{Ad}^*_{g_{k+1}} p_0 = \operatorname{Ad}^*_{g_k \Delta_k} p_0 = \operatorname{Ad}^*_{\Delta_k} \operatorname{Ad}^*_{g_k} p_0 = \operatorname{Ad}^*_{\Delta_k} \mu_k \, . 
		\label{eq:mu_update_k} 
	\end{equation}
	Simultaneously, the Lie algebra velocity at step $k+1$ must satisfy $\widetilde{\xi}_{k+1} = \pp{h_{NN}}{\mu} (\widetilde{\mu}_{k+1})$. For the time step $\Delta t$, the relative step $\Delta_k$ is computed via a second-order Magnus expansion:
	\begin{equation}
		\Theta^{(2)}(\Delta t, \Delta_k) = \frac{\Delta t}{2}\left(\xi_k + \widetilde{\xi}_{k+1}(\Delta_k)\right) + \frac{(\Delta t)^2}{12}\left[ \xi_k, \widetilde{\xi}_{k+1}(\Delta_k) \right] ,
		\label{eq:theta2}
	\end{equation}
	which yields an implicit algebraic relation for $\Delta_k$. Note that the Lie algebra commutator in \eqref{eq:theta2} involves all components of $\xi_k$ and $\widetilde{\xi}_{k+1}$, not just active ones. Setting the residual to zero creates the equation: 
	\begin{equation}
		\text{res}(\Delta_k) = \Delta_k - \exp\left( \Theta^{(2)}(h, \Delta_k) \right) = 0 .
		\label{magnus_eq_Delta} 
	\end{equation}
	Because $\Delta_k$ is close to the identity, we parameterize $\Delta_k$ via its Lie algebra counterpart $U_k$ as $\Delta_k = \exp(h U_k)$. Equation \eqref{magnus_eq_Delta} can be solved using standard root-finding algorithms or non-linear least squares; we employ the Newton--Raphson method for computational efficiency. Solving for $U_k$ at each time step yields the updated relative group element $\Delta_k$. 
	
	\paragraph{Step 3: Solving for observable and latent variables} 
	Obtaining $\Delta_k$ allows us to update the dynamic state variables at step $k+1$:
	\begin{equation}
		g_{k+1} = g_k \Delta_k, \quad \mu_{k+1} = \operatorname{Ad}^*_{\Delta_k} \mu_k, \quad \xi_{k+1} = \pp{h_{NN}}{\mu}(\mu_{k+1}) \, . 
		\label{reconstruction_eq_final} 
	\end{equation}
	
	\paragraph{Step 4: Full reconstruction} 
	With $p_0$ fixed from Step 1 for the entire trajectory, repeat Steps 2 and 3 sequentially across all time steps to complete trajectory reconstruction.
	
	\subsection{LLPNNs: The encoder (Lagrangian) approach} 
	\label{subsec:encoder}
	
	\subsubsection{Mathematical background} 
	In this second approach, we formulate the learning procedure using a
	Lagrangian-like function $\ell(\xi)$. Care must be taken here: unlike in
	the Hamiltonian formulation, we cannot assume that a physically meaningful
	Lagrangian always exists. We can, however, assume that, after excluding the
	passive coordinates, the mapping between active momentum and velocity
	variables is locally one-to-one. Specifically, the key assumption of the
	Lagrangian LLPNN approach is that the restriction of the map
	$\Phi(\mu)=\xi(\mu)=\partial h/\partial\mu$ to the active coordinates has a
	non-singular Hessian:
	\begin{equation}
		\det
		\frac{\partial^2 h}
		{\partial\mu^{\rm act}\partial\mu^{\rm act}}
		\neq 0
		\quad\Longrightarrow\quad
		\Psi(\xi^{\rm act})
		=
		\frac{\partial \ell}{\partial \xi^{\rm act}}
		=
		\mu^{\rm act},
		\qquad
		\Psi=\Phi^{-1}.
		\label{inverse_Psi}
	\end{equation}
	Indeed, by the inverse function theorem, the non-singularity condition in
	\eqref{inverse_Psi} guarantees the existence of a local inverse
	$\Psi=\Phi^{-1}$. Furthermore, the Jacobian of $\Psi$ is the inverse Hessian
	of $h$ in active components, which is symmetric. Consequently, $\Psi$ is locally a gradient
	mapping and therefore admits a scalar potential $\ell(\xi)$ satisfying
	$\mu^{\rm act}=\partial\ell/\partial\xi^{\rm act}$.
	
	We emphasize that $\ell(\xi)$ is generally not a true physical Lagrangian
	unless all coordinates are active and the Hamiltonian is non-degenerate.
	Condition \eqref{inverse_Psi} merely defines an invertible change of
	coordinates between active momentum and velocity variables and does not imply
	the existence of an Euler--Poincar\'e formulation. To reflect its purely
	formal role as a partial encoder, we refer to $\ell(\xi)$ as a
	\emph{pseudo-Lagrangian}.
	
	The main advantage of this formulation is that it operates directly on observable velocities $\xi$. Here, we only learn the encoder in active variables $\ell(\xi^{\text{act}})$ and do not require a corresponding decoder $h(\mu)$. Learning the phase space dynamics proceeds in two steps.
	
	\paragraph{Step 1} 
	Using conservation law \eqref{eq:coadjoint_reconstruction} for observable configurations $(g_1^j, \ldots, g_N^j)$ along trajectory $j \in \{ 1, \ldots, M\}$, there exists a constant spatial momentum $p_0^j$ such that the active momentum components satisfy: 
	\begin{equation} 
		\mathbb{P}^{\text{act}} \mu_i^j (\xi_i^j) =  \pp{\ell}{\xi^{\text{act}}} \left( \xi_i^{j,\text{act}} \right) = \mathbb{P}^{\text{act}} \operatorname{Ad}^*_{g_i^j} p_0^j \, . 
		\label{eq:mu_integral_j_Lagr} 
	\end{equation}
	Note that \eqref{eq:mu_integral_j_Lagr} is formulated in terms of velocity coordinates $\xi$, unlike the momentum formulation in \eqref{eq:mu_integral_j}. 
	
	\paragraph{Step 2} 
	We seek to approximate the pseudo-Lagrangian $\ell(\xi) \approx \ell_{NN}(\xi)$ using a neural network while simultaneously determining the constant vectors $p_0^j$ for each trajectory. A key technical difference arises when determining $p_0^j$ in the Lagrangian setting versus the Hamiltonian setting. Unlike in \eqref{eq:mu_integral_j}, the constants $p_0^j$ appear \emph{outside} the neural network evaluation in \eqref{eq:mu_integral_j_Lagr}. Consequently, for a given network $\ell_{NN}(\xi)$, \eqref{eq:mu_integral_j_Lagr} forms a \emph{linear} system for $p_0^j$, which can be solved efficiently via linear least squares. Thus, for network weights $\mathbf{W}$, the optimal constants $p_0^j(\mathbf{W})$ are determined directly from data. The predicted active momentum components are then given by:
	\begin{equation}
		\widetilde{\mu}_i^{j, \text{act}} = \mathbb{P}^{\text{act}} \pp{\ell_{NN}}{\xi} \left( \xi_i^j \right) \, . 
		\label{eq:mu_i_approx_Lagr}
	\end{equation}
	
	\paragraph{Step 3} 
	We construct a loss function depending solely on the network weights $\mathbf{W}$:
	\begin{equation}
		L(\mathbf{W}) = \sum_{j=1}^M \sum_{i=1}^N \left\| \mathbb{P}^{\text{act}} \operatorname{Ad}^*_{g_i^j} p_0^j (\mathbf{W}) - \widetilde{\mu}_i^{j,\text{act}} \right\|^2 , \quad \text{with } \widetilde{\mu}_i^{j,\text{act}} \text{ defined by \eqref{eq:mu_i_approx_Lagr}} . 
		\label{eq:latent_loss_Lagr} 
	\end{equation}
	Optimizing loss function \eqref{eq:latent_loss_Lagr} yields the learned pseudo-Lagrangian $\ell_{NN}(\xi^{\text{act}})$, providing the forward encoder mapping to active momentum components. This optimization also determines the Noether constants $p_0^j(\mathbf{W})$ for each training trajectory. 
	
	\begin{remark}[On non-constant time steps in the learning procedure]
		{\rm 
		Note that equations \eqref{eq:mu_integral_j} and \eqref{eq:xi_i_approx} in the Hamiltonian formulation, along with their Lagrangian counterparts \eqref{eq:mu_integral_j_Lagr} and \eqref{eq:mu_i_approx_Lagr}, depend only on configuration states and velocities at discrete points. Consequently, loss functions \eqref{eq:latent_loss} and \eqref{eq:latent_loss_Lagr} remain valid regardless of whether time steps between data points are uniform or non-uniform. While we focus on constant time step data in our numerical experiments, the framework extends naturally to irregularly sampled trajectories.
	}
	\end{remark} 
	
	\subsubsection{Reconstruction Algorithm: Pseudo-Lagrangian Approach} 
	
	\paragraph{Step 1: Finding Noether's constant}
	
	Similar to the Hamiltonian approach, suppose we observe a short window of
	initial conditions $(g_i,\xi_i)$ for $0 \le i \le N_{\mathrm{fit}}$ along a
	trajectory to be reconstructed. We seek the conserved spatial momentum
	$p_0$ satisfying
	\begin{equation}
		\frac{\partial \ell_{NN}}{\partial \xi^{\text{act}}}\left(\xi_i^{\text{act}}\right)
		=
		\mathbb{P}^{\mathrm{act}}
		\operatorname{Ad}^{*}_{g_i} p_0,
		\qquad
		0 \le i \le N_{\mathrm{fit}} .
		\label{init_cond_reconstruction_Lagr}
	\end{equation}
	
	Since $\ell_{NN}$ has already been learned, equation
	\eqref{init_cond_reconstruction_Lagr} is linear in the unknown vector
	$p_0$. This should be contrasted with its Hamiltonian counterpart
	\eqref{init_cond_reconstruction}, where $p_0$ appears nonlinearly inside
	the learned Hamiltonian $h_{NN}$.
	
	Because $\ell_{NN}$ acts only on the active components, whereas $p_0$
	contains all momentum components in the dual Lie algebra, a single
	observation is generally insufficient to uniquely determine $p_0$.
	Instead, a short fitting window of $N_{\mathrm{fit}}$ observations
	produces an overdetermined linear system that can be solved efficiently
	in the least-squares sense.
	
	\paragraph{Step 2: Determining velocity at the next step using Magnus residual} 
	Once $p_0$ is determined, we also know full initial momentum $\mu_0 = Ad^*_{g_0} p_0$ since we know the initial Lie group element $g_0$. Suppose we know $\mu_k$, $g_k$ and $\xi_k$ on the step $k$ and would like to find out these values at the step $k+1$. 
	
	Let $\Delta_k = g_k^{-1} g_{k+1}$ denote the relative configuration update $g_{k+1} = g_k \Delta_k$. Given current velocity $\xi_k$ and a candidate velocity $\widetilde{\xi}_{k+1} = \left(\xi^{\text{pass}},\widetilde{\xi}_{k+1}^{\text{act}} \right)$ at step $k+1$, we obtain: 
	\begin{equation} 
		\widetilde{\mu}_{k+1}^{\text{act}}  (\Delta_k) = \pp{\ell_{NN}}{\xi^{\text{act}}} (\widetilde{\xi}_{k+1}^{\text{act}} ) =\mathbb{P}^{\text{act}}  \operatorname{Ad}^*_{\Delta_k} \mu_k \, . 
		\label{eq:mu_update_k_Lagr} 
	\end{equation}
	As before, we relate $\Delta_k$ to velocities $\xi_k$ and $\widetilde{\xi}_{k+1}$ using the second-order Magnus expansion \eqref{eq:theta2}. Equating the two representations of the momentum at time step $k+1$
	defines the residual equation:
	\begin{equation}
		\text{res}(\widetilde{\xi}_{k+1}) = \pp{\ell_{NN}}{\xi^{\text{act}}} (\widetilde{\xi}_{k+1}^{\text{act}}) - \mathbb{P}^{\text{act}}  \operatorname{Ad}^*_{\Delta_k(\xi_k, \widetilde{\xi}_{k+1})} \mu_k = 0 . 
		\label{magnus_eq_Delta_Lagr} 
	\end{equation}
	Since $\Delta_k$ lies near the identity, we parameterize it through its Lie-algebra element. Equation \eqref{magnus_eq_Delta_Lagr} is then solved for
	\[
	\xi_{k+1}^{\mathrm{act}}=\widetilde{\xi}_{k+1}^{\mathrm{act}}
	\]
	using a root-finding or nonlinear least-squares procedure. The full velocity $\xi_{k+1}$ is then obtained by combining the solved active components with the fixed passive components.
	
	\paragraph{Step 3: Prediction of observable and latent variables at the next time step} 
	Once the full velocity $\xi_{k+1}$ is determined from \eqref{magnus_eq_Delta_Lagr}, we compute the relative configuration step $\Delta_k$ and update the latent momentum at step $k+1$ as
	\begin{equation}
		\mu_{k+1} = \operatorname{Ad}^*_{\Delta_k} \mu_k .
		\label{reconstruction_mu_k_eq_final} 
	\end{equation}
	Note that \eqref{reconstruction_mu_k_eq_final}, like its Hamiltonian counterpart \eqref{reconstruction_eq_final}, iterates the full latent momentum forward using coadjoint group actions, thereby preserving Casimir invariants to machine precision.	
	\paragraph{Step 4: Full reconstruction} 
	With $p_0$ determined from Step 1 for the entire trajectory, repeat Steps 2 and 3 sequentially across all time steps to complete trajectory reconstruction. 
	
	A schematic overview comparing the Lagrangian and Hamiltonian approaches is shown in Figure~\ref{fig:schematic}. 
	
	\begin{figure}[htbp]
		\centering
		\includegraphics[width=1\textwidth]{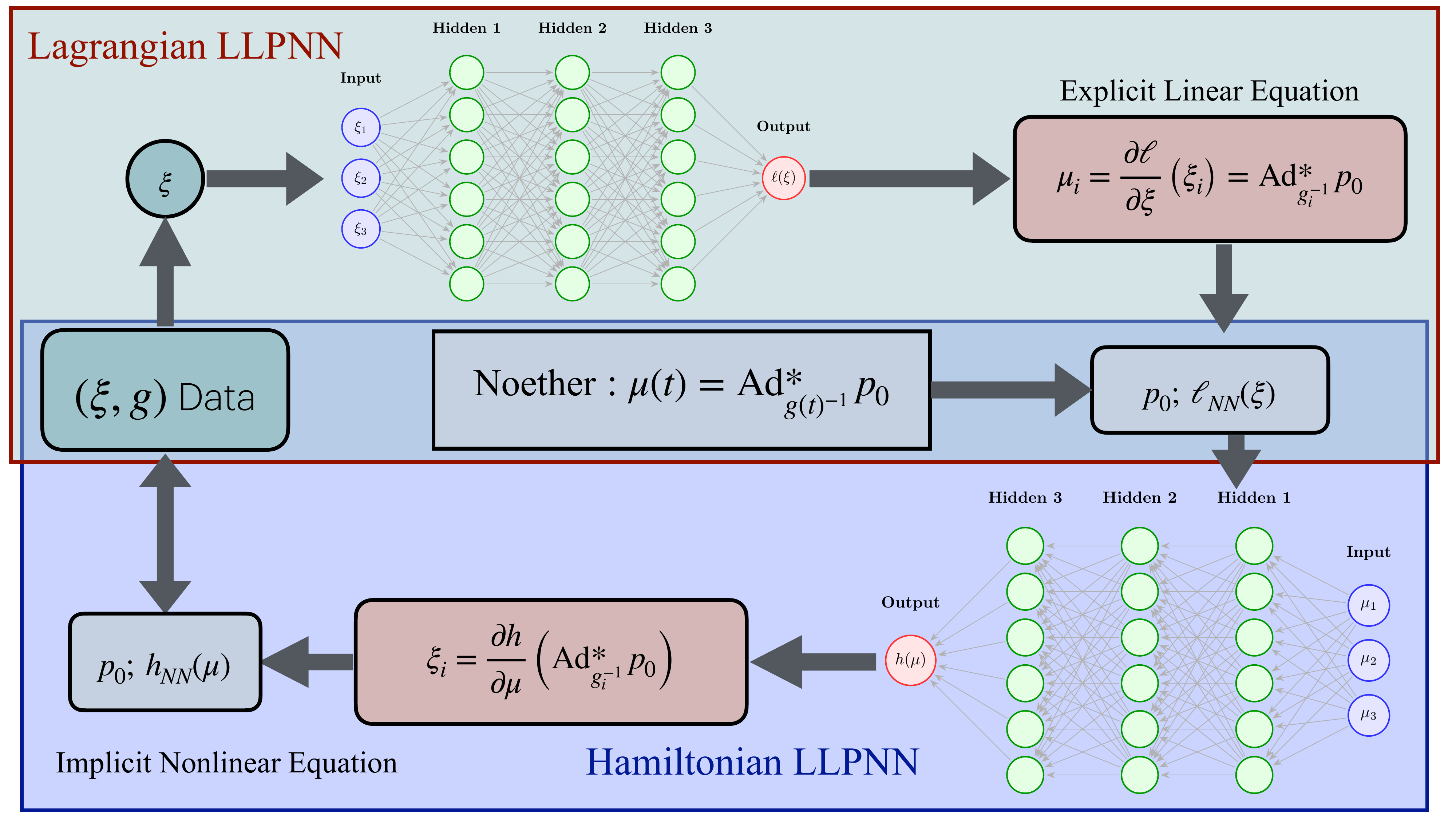}
		\caption{Schematic overview of the proposed LLPNN framework in both Lagrangian and Hamiltonian settings.}
		\label{fig:schematic}
	\end{figure}
	\subsection{Warm Start for the Learning Process} 
	\label{subsec:warm_start}
	To improve robustness and convergence during optimization, we initialize our algorithm with a physically and computationally viable warm start. This is particularly crucial for the Hamiltonian approach, where we must simultaneously train the Hamiltonian neural network and estimate the Noether constants $p_0$ from data. When Noether constants enter non-linearly into network evaluations, optimization initialized from random weights can encounter complex loss landscapes and converge to spurious local minima. An inaccurate initial state for the residual network can similarly cause convergence to false minima. Although random initialization often yields correct results, implementing a warm-start procedure significantly improves optimization reliability for practical applications. We design this procedure to provide accurate initial approximations for common control Hamiltonians, specifically those exhibiting quadratic behavior near the origin in active variables.
	
	To construct our warm-start procedure, we assume that the Hamiltonian and pseudo-Lagrangian in active coordinates are dominated by quadratic terms, with neural networks modeling non-quadratic residuals:
	\begin{align}
		\ell\left( \xi^{\text{act}} \right) &= \frac{1}{2} \xi^{\text{act}} \cdot \mathbb{B} \xi^{\text{act}} + \ell^{\text{res}}_{NN}\left(\xi^{\text{act}}\right) \, , 
		\label{warm_start_Lagr} 
		\\ 
		h(\mu) &= a \cdot \mu^{\text{passive}} + \frac{1}{2} \mu^{\text{act}} \cdot \mathbb{K} \mu^{\text{act}} + h^{\text{res}}_{NN}\left(\mu^{\text{act}}\right) \, , 
		\label{warm_start_Ham} 
	\end{align}
	where $\mathbb{B}$ and $\mathbb{K}$ are trainable symmetric matrices of size $K \times K$, with $K$ denoting the number of active variables. During the warm-start phase, we freeze the residual neural networks in \eqref{warm_start_Lagr} and \eqref{warm_start_Ham} to zero and solve for purely quadratic representations. Under this assumption, parameter initialization for the Hamiltonian, pseudo-Lagrangian, and Noether constants simplifies dramatically.
	
	In the Hamiltonian setting, \eqref{eq:xi_i_approx} simplifies to a biquadratic equation in $p_0$ and matrix $\mathbb{K}$, which can be solved efficiently using Newton iterations. In the Lagrangian setting, \eqref{eq:mu_i_approx_Lagr} reduces to a homogeneous linear system in $\mathbb{B}$ and $p_0$, solvable as an overdetermined least-squares problem. A unique solution is obtained by requiring, for example, $\operatorname{tr} \mathbb{B} = 1$, which yields a Karush--Kuhn--Tucker (KKT) linear system \cite{nocedal2006numerical}. This formulation establishes a unique scale for matrix $\mathbb{B}$ in \eqref{warm_start_Lagr} and all Noether constants $p_0$.
	
	Once initial parameter estimates are established, we unfreeze the neural network weights and proceed with full training. For the Hamiltonian formulation, this step also refines the estimated Noether constants $p_0$. This warm-start strategy provides a fast, robust initialization for the subsequent optimization phase.
	
	\begin{remark}[Approximation capacity of LLPNNs]
		{\rm 
		Expressions \eqref{warm_start_Ham} and \eqref{warm_start_Lagr} parameterize the reduced Hamiltonian $h(\mu)$ and, when it exists, the reduced pseudo-Lagrangian $\ell(\xi^{\mathrm{act}})$. Under the modeling assumption that the observed system is generated by Lie--Poisson Hamiltonian dynamics, the reduced Hamiltonian $h(\mu)$ is well defined. The pseudo-Lagrangian formulation $\ell(\xi^{\text{act}})$ exists locally on the active variables whenever the corresponding active velocity--momentum map is locally invertible. Since the networks use smooth activation functions, standard universal approximation results imply that these smooth functions, and their derivatives on compact subsets, can be approximated to arbitrary prescribed accuracy by sufficiently expressive networks \cite{hornik1990universal}.
	}
	\end{remark}
		
	\subsection{Energy Definition and Regularity}
	\label{subsec:energy_regularity}
	For autonomous Hamiltonian systems, the Hamiltonian function evaluated on momenta is a conserved quantity. Consequently, when using the Hamiltonian LLPNN formulation, energy conservation along reconstructed trajectories can be verified directly, as both the Hamiltonian function and state trajectories in observable and latent spaces are explicitly reconstructed.
	
	The situation differs in the Lagrangian setting. For a regular Lagrangian system, the conserved energy-like function $E$ is obtained through the Legendre transform:
	\begin{equation}
		E(\xi)
		=
		\left\langle \xi, \frac{\partial \ell}{\partial \xi} \right\rangle
		-
		\ell(\xi),
		\label{energy_like_xi}
	\end{equation}
	where $\langle \cdot,\cdot\rangle$ denotes the natural pairing between the Lie algebra $\mathfrak{g}$ and its dual $\mathfrak{g}^*$.
	Note that energy function \eqref{energy_like_xi} is formulated in terms of velocity $\xi$ rather than momentum $\mu$; it represents the quantity that becomes Hamiltonian $h(\mu)$ upon substituting $\xi = \xi(\mu)$.
	
	If both $\ell(\xi)$ and $h(\mu)$ are regular (i.e., the Hessians $\nabla^2 \ell$ and $\nabla^2 h$
	are non-singular in the region of interest), the Legendre transformation is locally invertible. Consequently, the energy functional $E(\xi)$ in \eqref{energy_like_xi} exists at least locally, a property that fails for degenerate or constrained systems. If the Hamiltonian is singular, the observable dynamics need
	not admit an Euler--Poincaré representation \eqref{EP_eqs}, as a valid Lagrangian satisfying the observed velocity data may not exist, and there may not even exist a closed autonomous dynamical
	system in the observable variables $\xi$, as we discussed above in Section~\ref{subsec:difference_Ham_Lagr}. 
	
	Accordingly, we compute energy in the Lagrangian formulation only when the underlying Hamiltonian is regular. In this work, regularity holds for the Kirchhoff underwater vehicle problem on $SE(3)$ detailed in Section~\ref{subsec:kirchhoff_se3}. In contrast, planar drone control on $SE(2)^N$
	(Section~\ref{subsec:SE2_drones}) involves a control
	Hamiltonian with two singular directions and one regular
	direction per agent. Consequently, the energy coming from the Legendre-transform
	\eqref{energy_like_xi} is not defined and no
	corresponding energy-conservation test is available in the
	pseudo-Lagrangian formulation.

	\subsection{Overall Comparison Between Hamiltonian and Lagrangian LLPNNs}
	\label{subsec:comparison}
	
	Both Lagrangian and Hamiltonian formulations offer distinct trade-offs, making each suitable for specific problem classes. A comparative summary is provided in Table~\ref{tab:llpnn_comparison}.
	
	\begin{table}[htbp]
		\centering
		\renewcommand{\arraystretch}{1.5}
		\begin{tabular}{|c|c|c|}
			\toprule
			Quantity & Hamiltonian LLPNN & Lagrangian LLPNN \\
			\midrule
			Phase space & Latent & Observable \\
			\hline 
			Applicability & Any Hamiltonian & \makecell[l]{Legendre invertible \\ in active variables }\\
			\hline 
			\makecell[c]{Noether Integral \\ $p_0$ computation} 
			& \makecell[l]{Implicit root-finding \\ inside NN evaluation} & \makecell[c]{Overdetermined \\ linear system} \\
			\hline 
			\makecell[c]{Gauge-invariant \\ latent space} & Yes & Yes \\
			\hline
			Energy conservation &   \makecell[l]{Directly available \\ for any Hamiltonian}  & Regular Hamiltonians \\
			\bottomrule
		\end{tabular}
		\caption{Summary of differences between Hamiltonian and Lagrangian LLPNN approaches.}
		\label{tab:llpnn_comparison}
	\end{table}
	
	The Hamiltonian formulation operates directly in the latent state space
	where the dynamics is closed, while the pseudo-Lagrangian formulation
	operates in the observable variables through an invertible active-sector
	encoding. The Hamiltonian LLPNN therefore applies to arbitrary Lie--Poisson
	systems, whereas the Lagrangian LLPNN benefits from substantially simpler
	reconstruction whenever the active velocity--momentum map is invertible.
	Consequently, both formulations are relevant in practical applications:
	the Hamiltonian approach offers the greatest generality and remains
	applicable to a wide class of degenerate systems, while the Lagrangian approach provides
	a computationally simpler learning and reconstruction procedure whenever a suitable
	active-variable inverse mapping exists.
	
	Having established the theoretical framework, we now turn to practical applications: rigid-body motion, Kirchhoff's underwater vehicle problem, and interacting planar drone dynamics. An interesting feature of the proposed warm-start procedure is that, when the underlying Hamiltonian and Lagrangian are quadratic, the initialization already recovers the exact solution, and further neural-network training provides little or no improvement. Consequently, purely quadratic systems, whether degenerate or non-degenerate, are of limited interest as neural-network benchmarks. While classical formulations of these problems typically employ quadratic Hamiltonians, we consider generalized non-quadratic models in order to demonstrate the full expressive power of the LLPNN framework.
	
	\subsection{Comparison with other methods}
	To the best of our knowledge, no existing structure-preserving framework is capable of learning arbitrary latent Hamiltonian dynamics directly from observable variables $(\xi, g)$ like our approach. We therefore evaluate alternative baseline methods from the literature that can be adapted to provide a meaningful benchmark comparison.
	
	If no additional information about the system is provided, one could use
	non-structure-preserving approaches such as Neural ODEs
	\cite{chen2018neural} and DeepONets \cite{lu2021learning}. However, these
	methods operate purely on the observable variables and therefore implicitly
	assume that the observables constitute a closed dynamical system. As
	discussed in Section~\ref{subsec:difference_Ham_Lagr}, this assumption may
	fail for degenerate Lie--Poisson Hamiltonians. In such cases, the observable
	variables do not uniquely determine the latent momentum state and may not
	satisfy any autonomous evolution equation. Consequently, Neural ODEs and
	DeepONets are forced to approximate an ill-defined dynamics without the knowledge of the latent states,
	whereas LLPNNs reconstruct the latent state space in which the dynamics is
	actually closed.
	
	Since we work with velocity data, an alternative would be to consider purely Lagrangian methods \cite{cranmer2020lagrangian} adapted to the symmetry-reduced Euler-Poincar\'e equations \eqref{EP_eqs}. However, as previously discussed, one cannot in general expect the motion of a degenerate symmetry-reduced Hamiltonian to reduce to a corresponding Euler-Poincar\'e system for a Lagrangian, as a Lagrangian representation may not even exist. 
	
	Instead, we adopt a compromise: rather than searching for a structure-preserving algorithm that predicts latent-variable Hamiltonian dynamics using only observable variables, we generate latent-space data sequences to train an alternative structure-preserving baseline. Specifically, we choose an algorithm based on Lie-Poisson Neural Networks (LPNets) \cite{eldred2024lie}, extended to interacting objects \cite{eldred2025clpnets,huraka2026structure}. This comparison should be interpreted with care, since LPNet training
	requires latent momentum trajectories that are not directly available from
	the observations. Thus, the LPNet baseline relies heavily on our method's learned Hamiltonian, which performs the preliminary work required to generate those latent trajectories.

	In summary, we compare LLPNNs against three representative alternatives.
	Neural ODEs and DeepONets require only observable data but do not exploit
	the latent Lie--Poisson structure. LPNets preserve the Lie--Poisson geometry
	exactly, but require access to latent momentum trajectories that are not
	available in practical applications. Together, these baselines provide a
	useful comparison between observable-space learning, latent-space
	structure-preserving learning, and the proposed latent reconstruction
	framework.
	
	The details of the learning procedure, which is common for all cases, is summarized on Table~\ref{tab:hyperparams}.
	\begin{table}[htbp]
		\centering
		
		\begin{tabular}{|c|c|c|c|c|c|}
			\toprule
			\textbf{Phase} & \textbf{Network} & \textbf{Epochs} & \textbf{Initial LR} & \textbf{Decay Rate} & \textbf{Decay Steps} \\
			\midrule
			L-Net      & $L_{\text{net}}$ & $20,000$        & $0.005$             & $0.90$              & $500$        \\
			H-Net   & $H_{\text{net}}, p_0$ & $20,000$   & $0.010$             & $0.90$              & $500$             \\
			LPNets   & $\text{LPNet}$   & $10,000$        & $0.005$             & Fixed LR        & N/A                 \\
			\bottomrule
		\end{tabular}
		\caption{Training Schedules and Optimization Hyperparameters 	\label{tab:hyperparams}}
		
	\end{table}

	\section{Experimental results} 
	
	\subsection{Ground Truth Data Generation}
	\label{subsec:ground_truth_integrator}
	
To generate physically faithful long-horizon reference trajectories without introducing numerical dissipation or artificial manifold drift, we employ a structure-preserving geometric integrator formulated for a general Lie group $G$ with dual Lie algebra $\mathfrak{g}^*$ (e.g., $SE(3)$, $SO(3)$, or product groups such as $SE(2)^N$). The ground-truth dynamics are integrated in two coupled stages: a symplectic implicit midpoint update for the body momentum $\mu \in \mathfrak{g}^*$, followed by a second-order Magnus expansion to reconstruct group configurations $g \in G$.
	
	\paragraph{1. Body Momentum Step (Implicit Midpoint)}
	The reduced Lie--Poisson dynamics on $\mathfrak{g}^*$ obey $\dot{\mu} = \Lambda(\mu) \nabla h(\mu)$, where $\Lambda_{ij}(\mu) = -\sum_k \mu_k \gamma^k_{ij}$ represents the Poisson structure tensor induced by the Lie algebra structure constants $\gamma^k_{ij}$. To advance the body momentum from $\mu_k$ to $\mu_{k+1}$ across an internal step size $\Delta t_{\text{int}}$, we solve the implicit midpoint scheme:
	\begin{equation}
		\mu_{k+1} - \mu_k - \Delta t_{\text{int}} \, \Lambda\left(\frac{\mu_k + \mu_{k+1}}{2}\right) \nabla h\left(\frac{\mu_k + \mu_{k+1}}{2}\right) = \mathbf{0}\,.
		\label{eq:implicit_midpoint}
	\end{equation}
	Equation \eqref{eq:implicit_midpoint} is solved numerically at each step via a high-precision multi-dimensional root finder with a residual tolerance of $10^{-12}$. The body velocity vector $\xi = \nabla h(\mu) \in \mathfrak{g}$ is obtained via exact gradient evaluations of the Hamiltonian $h(\mu)$, which may be non-quadratic. This method preserves all quadratic Casimirs (which comprise all Casimirs in the systems considered here) to machine precision.
	
	\paragraph{2. Configuration Reconstruction (Magnus Expansion)}
	Given momentum states $\mu_k$ and $\mu_{k+1}$, their associated body velocities $\xi_k = \nabla h(\mu_k)$ and $\xi_{k+1} = \nabla h(\mu_{k+1}) \in \mathfrak{g}$ are used to advance the group configuration $g_k \in G$. The algebra update element $\Omega_k \in \mathfrak{g}$ is constructed via a second-order Magnus expansion:
	\begin{equation}
		\Omega_k = \frac{\Delta t_{\text{int}}}{2} (\xi_k + \xi_{k+1}) + \frac{(\Delta t_{\text{int}})^2}{12} [\xi_k, \xi_{k+1}]_{\mathfrak{g}}\,,
		\label{eq:magnus_expansion}
	\end{equation}
	where $[\cdot, \cdot]_{\mathfrak{g}}$ denotes the Lie bracket on $\mathfrak{g}$. The group state is then updated via the Lie exponential map:
	\begin{equation}
		g_{k+1} = g_k \exp\left(\widehat{\Omega}_k\right)\,,
		\label{eq:group_update}
	\end{equation}
	where $\widehat{\Omega}_k$ denotes the matrix Lie algebra representation of $\Omega_k$.
	
	\paragraph{3. Spatial Momentum Reconstruction}
	By Noether's theorem, spatial momentum $p \in \mathfrak{g}^*$ is conserved under $G$-invariance. At each recorded step, spatial momentum $p_k$ is computed from body momentum $\mu_k$ via the coadjoint action:
	\begin{equation}
		p_k = \operatorname{Ad}^*_{g_k^{-1}} \mu_k\,.
		\label{eq:spatial_momentum}
	\end{equation}
	The spatial momenta $p_k$, conserved by Noether's theorem, are saved throughout integration to monitor and verify numerical accuracy.
	
	In our data generation pipeline, integrations are executed using a fine sub-stepping time step $\Delta t_{\text{int}} = 0.01\,\text{s}$ and downsampled to an observational resolution of $\Delta t = 0.1\,\text{s}$. This dual-stage scheme guarantees exact geometric preservation of coadjoint orbits, high-accuracy conservation of spatial momentum $p$, preservation of system Casimirs, and long-term stability of the total energy across general Lie group phase spaces.
	
	The color scheme and line styles used across all experimental figures in subsequent sections adhere to the conventions summarized in Table~\ref{tab:method_styles}.
	
	\begin{table}[htbp]
		\centering
		\begin{tabular}{|c|c|c|}
			\toprule
			Method & Color & Line Style \\
			\midrule
			Ground truth & \textcolor{blue}{Blue} & Solid (\texttt{-}) \\
			Standard LPNet & \textcolor{red}{Red} & Dashed (\texttt{--}) \\
			Hamiltonian LLPNN & \textcolor{green}{Green} & Dash-Dotted (\texttt{-.}) \\
			Lagrangian LLPNN & \textcolor{magenta}{Magenta} & Dotted (\texttt{:}) \\
			DeepONet & \textcolor{orange}{Orange} & Dashed (\texttt{--}) \\
			Neural ODE & \textcolor{cyan}{Cyan} & Dash-Dotted (\texttt{-.}) \\
			\bottomrule
		\end{tabular}
		\caption{Summary of universal color and line style codes used in figure plots reporting experimental results.}
		\label{tab:method_styles}
	\end{table}
	
	\subsection{A Generalized Rigid Body: $SO(3)$ dynamics}
	\label{subsec:rigidbody_so3}
	
	\subsubsection{Problem formulation} 
	As a first numerical application, we consider rigid body dynamics modeled via Lie--Poisson equations on $SO(3)$. Although classical governing relations for rigid body motion simplify under standard symmetries, they continue to exhibit rich dynamical behavior. In this benchmark, we evaluate a non-quadratic modification of the Hamiltonian to demonstrate the learning capability and geometric fidelity of the LLPNN framework.
	
	The phase-space geometry of the rigid body is dictated by the Lie--Poisson bracket associated with the special orthogonal group $SO(3)$, encapsulating three-dimensional rotations. For $SO(3)$, Lie algebra velocities and dual Lie algebra momenta are identified with vectors in $\mathbb{R}^3$, denoted by bold symbols $\boldsymbol{\mu}$ and $\boldsymbol{\xi}$, respectively. For an inertia matrix $\mathbb{I} = \operatorname{diag}(1, 2, 3)$, the classical quadratic kinetic energy is:
	\begin{equation} 
		H_{\text{quad}}(\boldsymbol{\mu}) = \frac{1}{2} \boldsymbol{\mu}\cdot \mathbb{I}^{-1} \boldsymbol{\mu} = \frac{1}{2} \left( \mu_1^2 + \frac{1}{2} \mu_2^2 + \frac{1}{3} \mu_3^2 \right) \,.
		\label{Rigidbody_Hamiltonian_classical}
	\end{equation} 
	For smooth functions $F$ and $H$ on $\mathfrak{so}(3)^*$, the Lie--Poisson bracket is defined through a standard cross-product:
	\begin{equation}
		\{ F, H \} = - \boldsymbol{\mu} \cdot \left( \frac{\partial F}{\partial \boldsymbol{\mu}} \times \frac{\partial H}{\partial\boldsymbol{\mu}} \right) \,.
		\label{SO3_bracket}
	\end{equation} 
	
	Applying Lie--Poisson bracket \eqref{SO3_bracket} to state variables yields the equations of motion: 
	\begin{equation} 
		\dot{\boldsymbol{\mu}} = \boldsymbol{\mu} \times \frac{\partial H}{\partial\boldsymbol{\mu}} \,.
		\label{Rigidbody_eqs} 
	\end{equation} 
	Beyond total energy conservation, system \eqref{Rigidbody_eqs} inherently preserves a Casimir invariant linked to $\mathfrak{so}(3)^*$, given by the squared momentum magnitude:
	\begin{equation}
		C = \|\boldsymbol{\mu}\|^2 = \mu_1^2 + \mu_2^2 + \mu_3^2 \, .   
		\label{Casimir_SO3} 
	\end{equation}
	
	As noted previously, purely quadratic Hamiltonians can be learned with minimal neural network complexity. However, in general optimal control applications, quadratic Hamiltonians cannot be assumed. We therefore consider the non-quadratic Hamiltonian:
	\begin{equation}
		H(\boldsymbol{\mu}) = H_{\text{quad}}(\boldsymbol{\mu}) \left(1 + B (\mu_1^2 + \mu_2^2 + 2 \mu_3^2 ) \right) \, , 
		\label{Rigidbody_Hamiltonian_nonquad}
	\end{equation}
	where $B = 0.1$ in all simulations. The coefficient $2$ multiplying $\mu_3^2$ ensures that the non-quadratic scaling factor is not simply a function of the Casimir invariant \eqref{Casimir_SO3}.
	
	Even with non-quadratic Hamiltonian \eqref{Rigidbody_Hamiltonian_nonquad}, the two conserved quantities (the Casimir and the Hamiltonian) define one-dimensional invariant curves at their intersection surfaces in three-dimensional momentum space $\boldsymbol{\mu}$ (and correspondingly in velocity space $\boldsymbol{\xi}$). Consequently, this generalized rigid body system remains integrable for any smooth $H(\boldsymbol{\mu})$, although the solutions, in general, cannot be expressed in terms of classical elliptic functions as is the case with the standard rigid body  \cite{whittaker1964treatise,lawden2013elliptic}.

	\subsubsection{Experimental Results: $SO(3)$ Non-Quadratic Rigid Body}
	
\paragraph{Ground truth data generation for training and testing}
For training, we generate 200 trajectories, each containing 101 time points. The time step is fixed at $\Delta t = 0.1$ throughout. The initial conditions are sampled from the standard normal distribution in the three-dimensional momentum space, namely
\[
\mu^{\mathrm{train}}_0 \sim \mathcal{N}(0,I_3),
\]
where $I_3$ is the $3\times 3$ identity matrix. For the test trajectories, the initial momenta are sampled from a distribution with half the standard deviation,
\[
\mu^{\mathrm{test}}_0 \sim \mathcal{N}(0,0.25 I_3).
\]
The initial orientations $g_0 \in SO(3)$ are sampled independently from the uniform distribution on $SO(3)$. In total, we generate 50 test trajectories, each containing 2001 time points.

	\paragraph{Neural Network Architectures and Residual Formulations}
	Both Hamiltonian and Lagrangian neural networks are defined over active coordinates and initialized using quadratic warm-start baselines:
	\begin{enumerate}
		\item \textbf{Hamiltonian Model ($H_{\text{NN}}$):} 
		Residual network $H_{\text{net}}$ consists of $3$ dense hidden layers with $32$ units per layer, utilizing $\tanh$ activation functions. The output layer is a single linear unit without bias. All layers are zero-initialized. To eliminate gauge ambiguity at $\boldsymbol{\mu} = \mathbf{0}$, the network subtracts its value and gradient evaluated at the origin. 
		
		\item \textbf{Lagrangian Model ($L_{\text{NN}}$):} 
		Similarly, the direct Lagrangian model operates on active body velocities. Residual network $L_{\text{net}}$ mirrors $H_{\text{net}}$ with $3$ hidden layers of $32$ units, $\tanh$ activations, and identical zero-point ambiguity removal at $\boldsymbol{\xi} = \mathbf{0}$. 
	\end{enumerate}
	
Both Lagrangian and Hamiltonian networks are trained for $20{,}000$ epochs. During training, the loss functions \eqref{eq:latent_loss} and \eqref{eq:latent_loss_Lagr} decrease by approximately four orders of magnitude, dropping from about $10^{-3}$ to  $\sim10^{-6}-10^{-7}$, or lower depending on the run.
	
	Figure~\ref{fig:rigidbody_xi_components} compares trajectory tracking for observable velocity components $\boldsymbol{\xi}$ across evaluated methods for a representative trajectory. The Hamiltonian LLPNN achieves near-perfect tracking, remaining visually indistinguishable from ground truth across the full time horizon. The Lagrangian LLPNN and standard LPNet also maintain close agreement, whereas non-structure-preserving baselines (DeepONet and Neural ODE) diverge significantly over long integration times.
	
	\begin{figure}[htbp]
		\centering
		\includegraphics[width=1\textwidth]{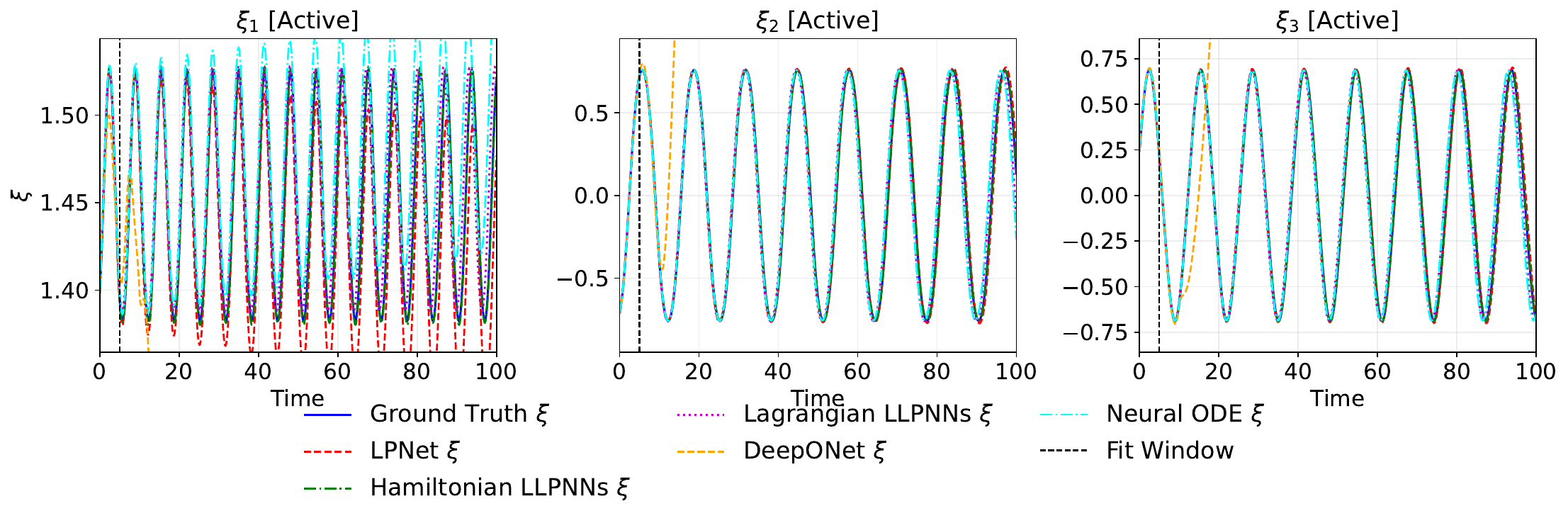}
		\caption{Predicted versus ground-truth velocity components $\boldsymbol{\xi}(t)$ for non-quadratic $SO(3)$ rigid body dynamics. Hamiltonian LLPNN (green dash-dotted line) remains visually indistinguishable from ground truth (blue solid line), whereas non-structure-preserving baselines diverge over time.}
		\label{fig:rigidbody_xi_components}
	\end{figure}
	
	Figure~\ref{fig:rigidbody_mu_components} displays corresponding latent momentum trajectories $\boldsymbol{\mu}$ reconstructed across momentum-capable formulations.
	
	\begin{figure}[htbp]
		\centering
		\includegraphics[width=1\textwidth]{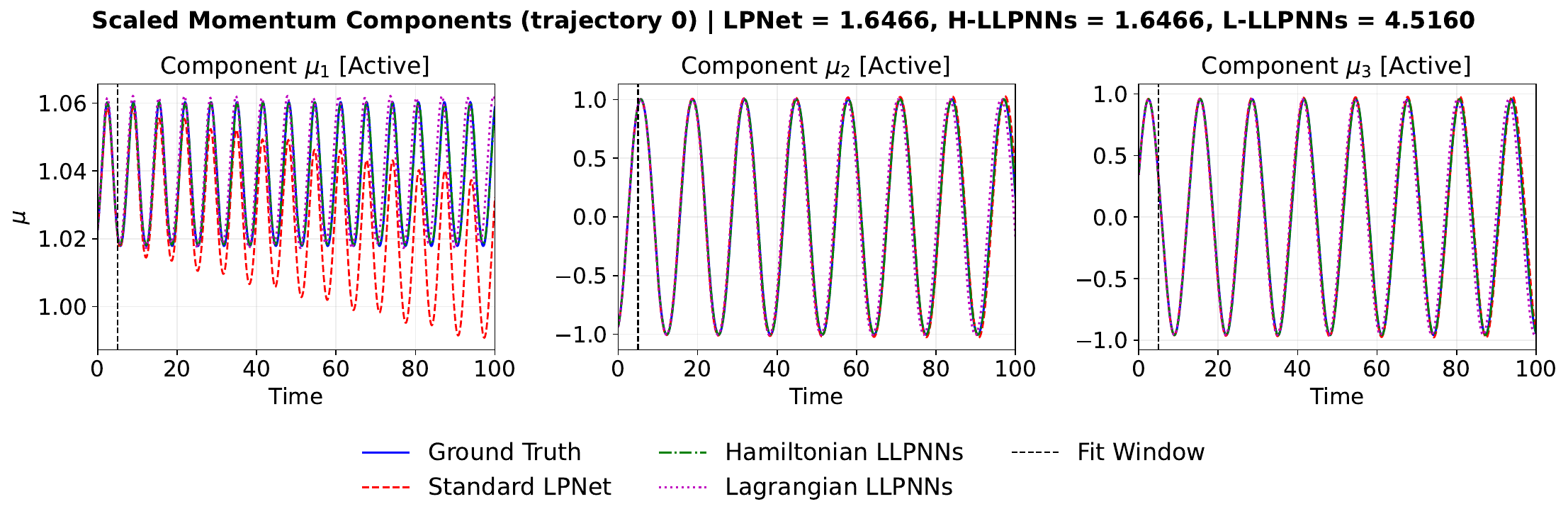}
		\caption{Reconstructed scaled latent momentum trajectories $\boldsymbol{\mu}(t) = (\mu_1,\mu_2,\mu_3)$ for the non-quadratic $SO(3)$ rigid body, shown for all models capable of predicting momentum evolution.}
		\label{fig:rigidbody_mu_components}
	\end{figure}
	
	To evaluate preservation of integrability, Figure~\ref{fig:rigidbody_PhaseSpace} displays 3D phase-space trajectories corresponding to Figures~\ref{fig:rigidbody_xi_components} and \ref{fig:rigidbody_mu_components}. Because the 3D phase space possesses two independent constants of motion (energy and Casimir), non-singular trajectories form closed orbits. The proposed Hamiltonian and Lagrangian LLPNNs reconstruct closed orbits matching ground truth data with high precision. Conversely, standard LPNets exhibit gradual orbit drift, while Neural ODEs and DeepONets rapidly depart from invariant manifolds in velocity space.
	
	\begin{figure}[htbp]
		\centering
		\includegraphics[width=1\textwidth]{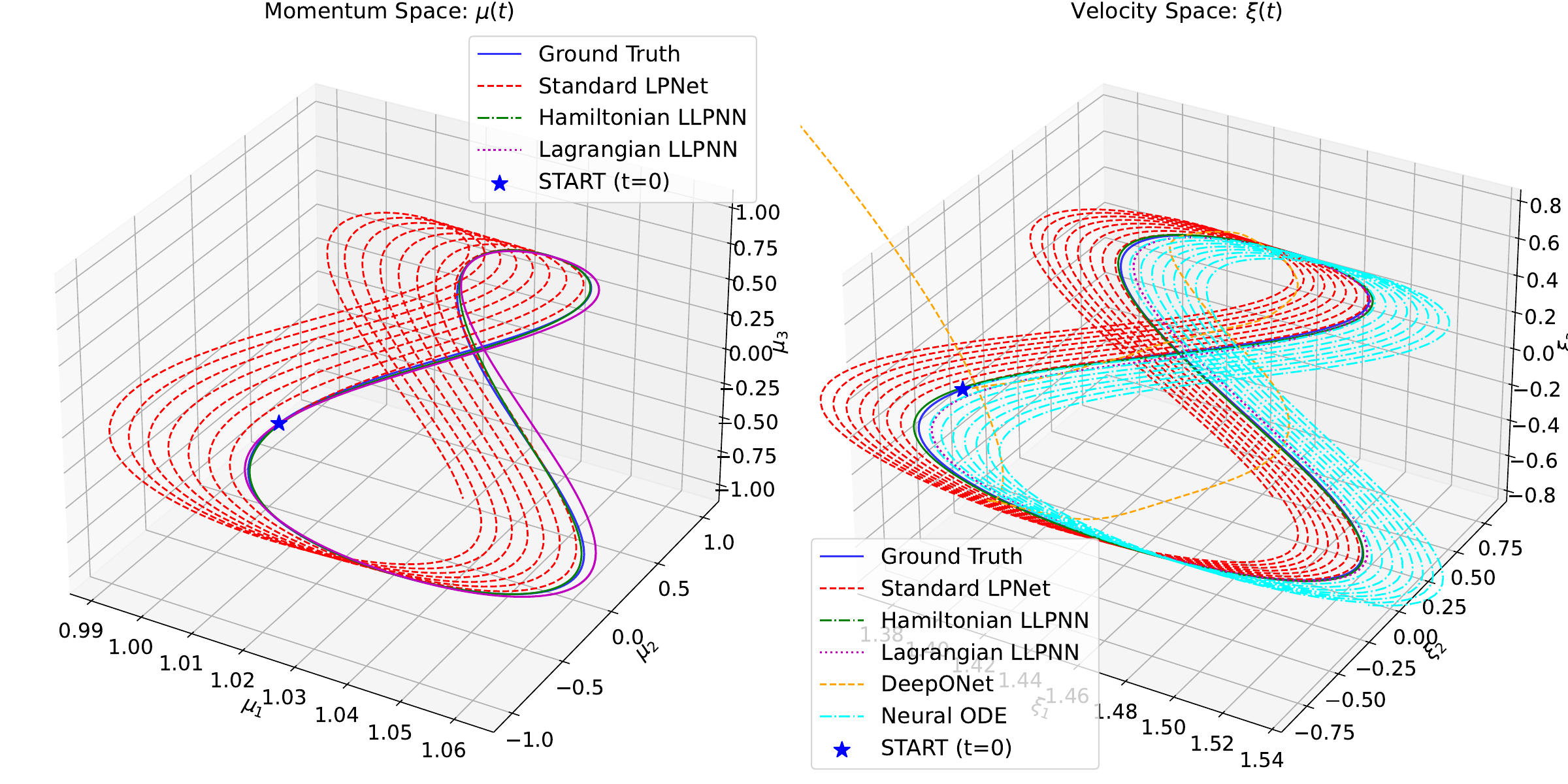}
		\caption{3D phase-space trajectories in latent momentum space $\boldsymbol{\mu}$ (left) and observable velocity space $\boldsymbol{\xi}$ (right) for non-quadratic $SO(3)$ rigid body dynamics. The proposed LLPNNs maintain closed invariant orbits matching the ground truth, while the standard LPNet (dashed red line) and Neural ODE (dashed blue line) solutions drift slowly away from the true trajectory; the DeepONet baseline (dashed orange line) diverges rapidly.}
		\label{fig:rigidbody_PhaseSpace}
	\end{figure}
	
	Figure~\ref{fig:rigidbody_mae_combined} generalizes single-trajectory observations by plotting Mean Absolute Error (MAE) evaluated across $50$ test trajectories with random initial conditions. The Hamiltonian LLPNN consistently achieves the lowest error in both latent ($\boldsymbol{\mu}$) and observable ($\boldsymbol{\xi}$) spaces, followed closely by the Lagrangian LLPNN, confirming superior accuracy in long-term reconstruction.
	
	\begin{figure}[htbp]
		\centering
		\includegraphics[width=1\textwidth]{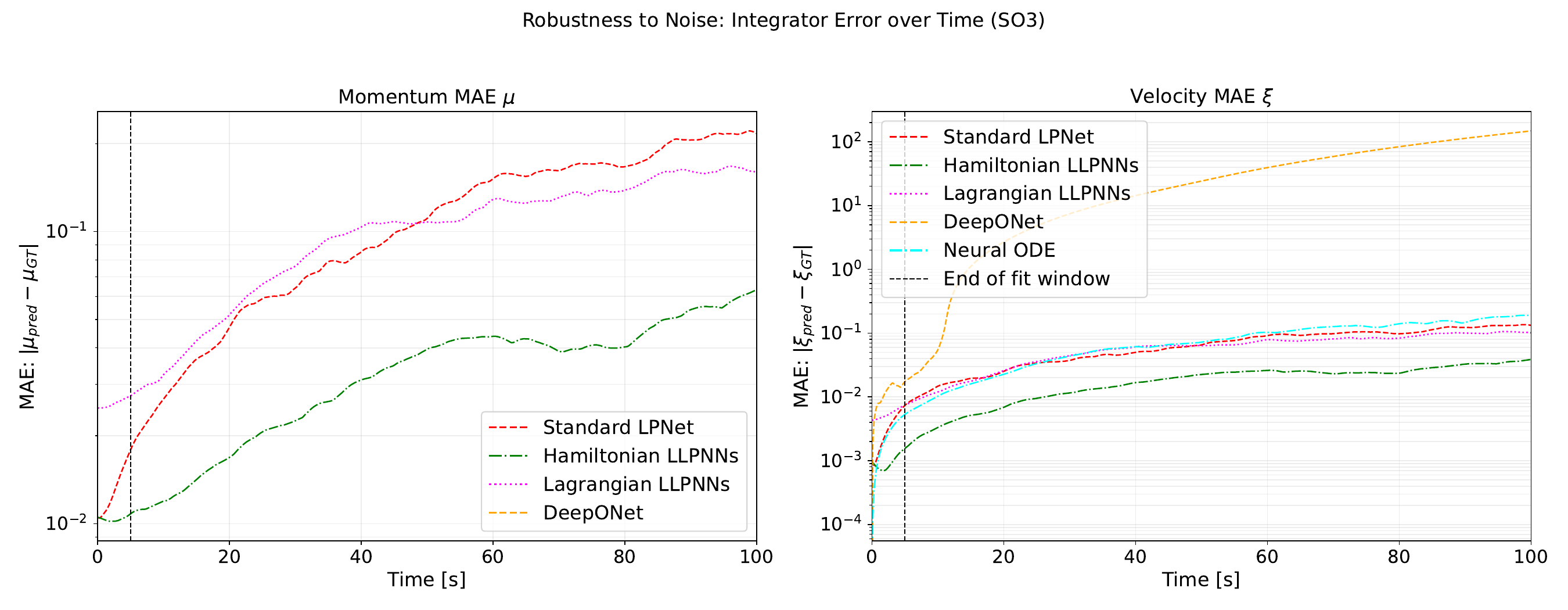}
		\caption{Time-evolution of Mean Absolute Error (MAE) for momentum $\boldsymbol{\mu}$ (left) and velocity $\boldsymbol{\xi}$ (right), averaged over 50 test trajectories with random initial conditions. Hamiltonian LLPNN achieves superior long-term accuracy.}
		\label{fig:rigidbody_mae_combined}
	\end{figure}
	
	Figure~\ref{fig:rigidbody_Casimir_energy} evaluates the preservation of conserved quantities across all evaluated methods. As shown in the left panel, all Lie--Poisson formulations (Hamiltonian LLPNN, Lagrangian LLPNN, and LPNet) preserve Casimir invariants to machine precision. In contrast, non-structure-preserving baselines that lack momentum representations (Neural ODE and DeepONet) cannot compute Casimir or energy values. The right panel demonstrates energy conservation error over time. Here, the Hamiltonian LLPNN achieves near-zero energy drift, whereas the Lagrangian LLPNN exhibits a small constant offset due to numerical differentiation in the inverse Legendre transformation, though it still maintains excellent long-term energy stability.
	
	\begin{figure}[htbp]
		\centering
		\includegraphics[width=0.48\textwidth]{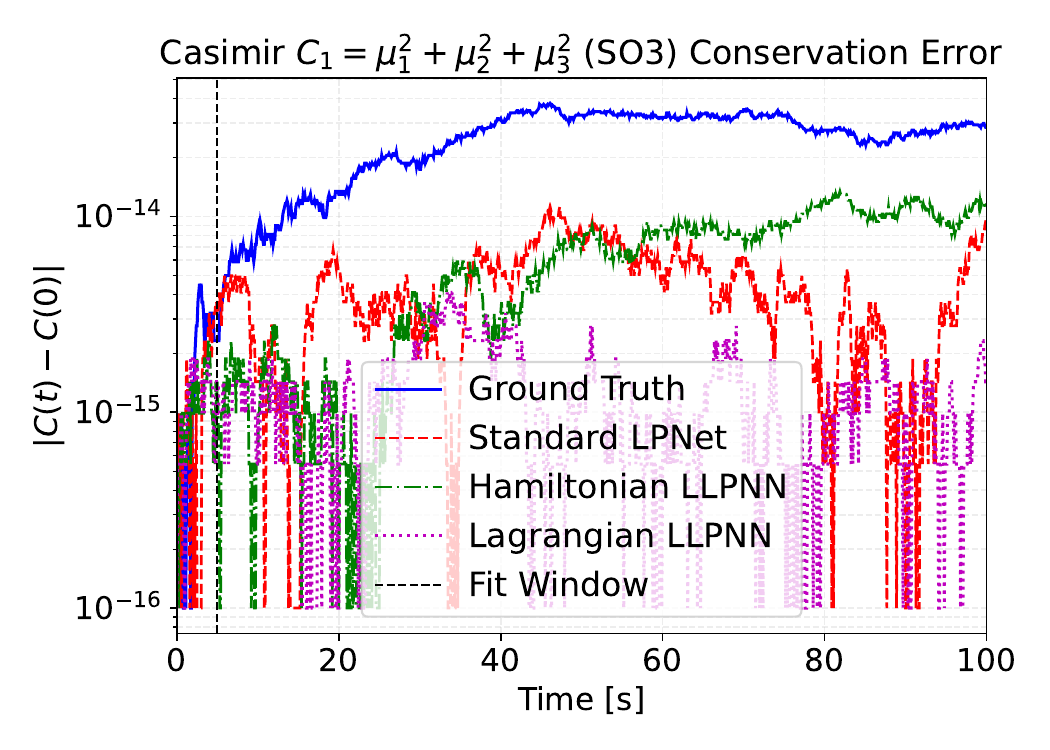}
		\hfill
		\includegraphics[width=0.48\textwidth]{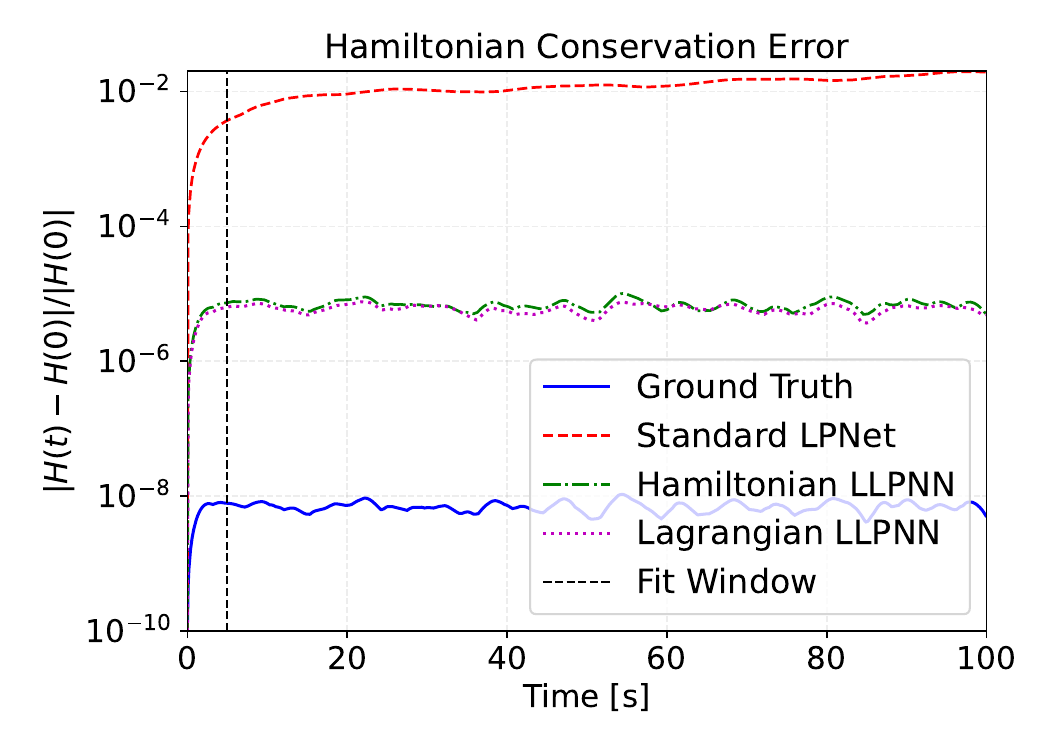}
		\caption{Conservation properties for non-quadratic $SO(3)$ rigid body dynamics averaged over 50 test trajectories. Left: Casimir conservation error over time, showing machine-precision preservation for all Lie--Poisson methods. Right: Relative energy drift over time, highlighting highly accurate energy conservation for both the Hamiltonian and Lagrangian LLPNNs.}
		\label{fig:rigidbody_Casimir_energy}
	\end{figure}
	
	\subsection{Kirchhoff's Underwater Vehicle: $SE(3)$ Dynamics}
	\label{subsec:kirchhoff_se3}
	
	\subsubsection{Problem formulation} 
	As a second application of our framework, we consider the motion of a submerged, neutrally buoyant rigid body governed by Kirchhoff's equations. A detailed exposition on their Lie--Poisson structure and geometric stability analysis can be found in \citep{leonard1997stability,LeMa1997stability}. When the vehicle's center of gravity coincides with its center of buoyancy, the governing relations simplify while exhibiting rich dynamical behavior, spanning both integrable regimes and chaotic dynamics \citep{holmes1998dynamics}. This system represents a significant step up in difficulty compared to $SO(3)$ rigid body dynamics, providing a demanding test of our learning methodology due to the complex coupling between rotational and translational modes.
	
	The motion of Kirchhoff's underwater vehicle is described by Lie--Poisson equations on $SE(3)$ with dual Lie algebra momentum state $\mu = (\mathbf{J}, \mathbf{P}) \in \mathfrak{se}(3)^*$, where $\mathbf{J}$ represents angular momentum and $\mathbf{P}$ denotes linear momentum. The classical kinetic energy Hamiltonian $H_0(\mathbf{J}, \mathbf{P})$ splits into rotational energy with inertia tensor $\mathbb{I}$ and translational energy with mass tensor $\mathbb{M}$:
	\begin{equation} 
		H_0(\mu) = H_0 (\mathbf{J},\mathbf{P})= \frac{1}{2} \mathbf{J} \cdot \mathbb{I}^{-1} \mathbf{J} + \frac{1}{2} \mathbf{P} \cdot \mathbb{M}^{-1} \mathbf{P}\,.
		\label{Underwater_body_Hamiltonian_classical}
	\end{equation} 
	In standard formulations, the inertia matrix $\mathbb{I}$ and generalized mass matrix $\mathbb{M}$ are taken to be diagonal. The standard physical Hamiltonian for this problem \eqref{Underwater_body_Hamiltonian_classical} is quadratic in momentum.
	
	The phase-space geometry of the vehicle is governed by the Lie--Poisson bracket associated with the special Euclidean group $SE(3)$, encapsulating three-dimensional rotations and translations. For arbitrary smooth functions $F$ and $H$ on $\mathfrak{se}(3)^*$, this bracket reflects the underlying semidirect product structure \citep{holm1998euler,leonard1997stability,holm2009geometric}:
	\begin{equation}
		\{ F, H \} = - \mathbf{J} \cdot \left( \frac{\partial F}{\partial \mathbf{J}} \times \frac{\partial H}{\partial \mathbf{J}} \right) - \mathbf{P} \cdot \left( \frac{\partial F}{\partial \mathbf{J}} \times \frac{\partial H}{\partial \mathbf{P}} - \frac{\partial H}{\partial \mathbf{J}} \times \frac{\partial F}{\partial \mathbf{P}} \right) \,.
		\label{SE3_bracket}
	\end{equation} 
	
	Applying Lie--Poisson bracket \eqref{SE3_bracket} to state variables yields Kirchhoff's equations of motion:
	\begin{equation} 
		\begin{aligned} 
			\dot{\mathbf{J}} & = - \frac{\partial H}{\partial \mathbf{J}} \times \mathbf{J} - \frac{\partial H}{\partial \mathbf{P}} \times \mathbf{P}\,, 
			\\
			\dot{\mathbf{P}} & = - \frac{\partial H}{\partial \mathbf{J}} \times \mathbf{P}\,.
		\end{aligned} 
		\label{Underwater_body_eqs} 
	\end{equation} 
	In addition to total energy conservation, system \eqref{Underwater_body_eqs} preserves two Casimir functions associated with the dual algebra $\mathfrak{se}(3)^*$: the squared linear momentum magnitude $C_1 = \|\mathbf{P}\|^2$ and the momentum pairing $C_2 = \mathbf{J} \cdot \mathbf{P}$.
	
	To test the learning capabilities of our framework on non-quadratic Hamiltonians, which naturally arise in controlled or fluid-coupled body dynamics, we modify the kinetic Hamiltonian to include non-quadratic state dependencies:
	\begin{equation}
		H(\mu) = H (\mathbf{J},\mathbf{P})= \left( \frac{1}{2} \mathbf{J} \cdot \mathbb{I}^{-1} \mathbf{J} + \frac{1}{2} \mathbf{P} \cdot \mathbb{M}^{-1} \mathbf{P} \right) \left(1 + B \left( \| \mathbf{J}\|^2 + \|\mathbf{P}\|^2\right) \right) \, , 
		\label{Underwater_body_Hamiltonian}
	\end{equation}
	where $B = 0.1$ in all numerical simulations. Unlike lower-dimensional examples, all six degrees of freedom in active coordinates are coupled non-linearly.

	\subsubsection{Experimental Results: $SE(3)$ Underwater Vehicle}
	
\paragraph{Ground truth data generation for training and testing}
For training, we generate 200 trajectories, each containing 101 time points. The initial momenta are sampled from the standard normal distribution in the six-dimensional momentum space, namely
\[
\mu^{\mathrm{train}}_0 \sim \mathcal{N}(0,I_6),
\]
where $I_6$ is the $6\times 6$ identity matrix. For testing, we generate 50 trajectories, each containing 2001 time points, from the same initial momentum distribution,
\[
\mu^{\mathrm{test}}_0 \sim \mathcal{N}(0,I_6).
\]
The time step is fixed at $\Delta t=0.1$ throughout. The initial orientations $R_0 \in SO(3)$ are sampled independently from the uniform distribution on $SO(3)$, while the initial positions are sampled from the standard normal distribution in $\mathbb{R}^3$.

	\paragraph{Neural Network Architectures and Residual Formulations}
	Both Hamiltonian and Lagrangian formulations operate on all six active coordinates ($K=6$) and are initialized using warm-start baselines:
	\begin{enumerate}
		\item \textbf{Hamiltonian Model ($h_{\text{NN}}$):}
		The Hamiltonian $h(\mu)$ is decomposed into a quadratic baseline core and a neural network residual:
		\begin{equation}
			h(\mu) = \frac{1}{2} \mu^\top \mathbb{K} \mu + h_{\text{net}}(\mu),
		\end{equation}
		where $\mathbb{K} = F F^\top \in \mathbb{R}^{6 \times 6}$ is a positive-definite matrix initialized via Cholesky factorization of the warm-start estimate for the inverse inertia matrix. The residual network $h_{\text{net}}: \mathbb{R}^6 \to \mathbb{R}$ comprises $3$ hidden layers with $32$ units per layer and $\tanh$ activation functions. The output layer is a single linear unit without bias, initialized with zero weights. To enforce strict consistency at the origin ($h_{\text{net}}(\mathbf{0}) = 0$ and $\nabla h_{\text{net}}(\mathbf{0}) = \mathbf{0}$), a linear correction is applied to raw network output $r_h$:
		\begin{equation}
			h_{\text{net}}(\mu) = r_h(\mu) - r_h(\mathbf{0}) - \nabla r_h(\mathbf{0})^\top \mu.
		\end{equation}
		The predicted velocity field is thus $\xi_{\text{pred}} = \nabla_\mu h(\mu) = A_0 \mu + \nabla h_{\text{net}}(\mu)$.
		
		\item \textbf{Lagrangian Model ($\ell_{\text{NN}}$):}
		Similarly, the direct Lagrangian model $\ell_{\text{NN}}: \mathbb{R}^6 \to \mathbb{R}$ acts on body velocities $\xi = (\boldsymbol{\Omega}, \mathbf{v}) \in \mathfrak{se}(3)$:
		\begin{equation}
			\ell_{\text{NN}}(\xi) = \frac{1}{2} \xi^\top \mathbb{B} \xi + 	\ell_{\text{net}}(\xi),
		\end{equation}
		where $\mathbb{B} = \frac{1}{2}(\mathbb{B}_{\text{core}} + \mathbb{B}_{\text{core}}^\top) \in \mathbb{R}^{6 \times 6}$ is a trainable symmetric mass matrix initialized to $M_{\text{warm}}$. Residual network $	\ell_{\text{net}}$ mirrors $h_{\text{net}}$ with identical hidden layer sizing, zero-initialized output weights, and origin constraints:
		\begin{equation}
			\ell_{\text{net}}(\xi) = r_	\ell(\xi) - r_	\ell(\mathbf{0}) - \nabla r_	\ell(\mathbf{0})^\top \xi.
		\end{equation}
	\end{enumerate}

	\paragraph{Optimization and Training Performance}
	All models are implemented in 64-bit precision (\texttt{tf.float64}) and trained for $20{,}000$ epochs using the Adam optimizer, with training parameters detailed in Table~\ref{tab:hyperparams}. Each trajectory reconstruction uses $N_{\text{fit}} = 50$ fit steps to determine the Noether constant $p_0$ for that trajectory. Throughout optimization, the loss functions \eqref{eq:latent_loss} and \eqref{eq:latent_loss_Lagr} typically drop by approximately four orders of magnitude, decaying from initial values around $10^{-3}$ to the range $10^{-6}$--$10^{-7}$, with some runs converging to even lower loss values.
	
	\paragraph{Energy and Legendre Regularity}
	For the learned Lagrangian model $\ell(\xi)$, the total energy $E(\xi)$ is evaluated via the Legendre transform:
	\begin{equation}
		E(\xi)
		=
		\left\langle \xi, \frac{\partial \ell}{\partial \xi} \right\rangle
		-
		\ell(\xi),
	\end{equation}
	where
	\[
	\left\langle
	(\boldsymbol{\Omega},\mathbf{v}),
	(\mathbf{J},\mathbf{P})
	\right\rangle
	=
	\boldsymbol{\Omega}\cdot\mathbf{J}
	+
	\mathbf{v}\cdot\mathbf{P}
	\]
	is the natural pairing between the Lie algebra $\mathfrak{se}(3)$ and its dual $\mathfrak{se}(3)^*$.
	In this formulation, both $	\ell(\xi)$ and $h(\mu)$ are regular (their Hessians $\nabla_\xi^2 	\ell$ and $\nabla_\mu^2 h$ are positive-definite). Consequently, the Legendre mapping $\xi \mapsto \mu = \frac{\partial 	\ell}{\partial \xi}$ is locally invertible, ensuring that energy $E(\xi)$ is well-defined across the domain.
	
	\paragraph{Trajectory and Conservation Analysis}
	Figure~\ref{fig:kirchhoff_xi_components} demonstrates the time-series for the six observable velocity components $\xi = (\boldsymbol{\Omega}, \mathbf{v}) = (\Omega_x, \Omega_y, \Omega_z, v_x, v_y, v_z)$ on a representative test trajectory. The Hamiltonian LLPNN captures both high-frequency rotational oscillations ($\boldsymbol{\Omega}$) and coupled translational velocity dynamics ($\mathbf{v}$), maintaining near-perfect alignment with ground-truth curves. Lagrangian LLPNN and standard LPNet models follow closely, whereas DeepONet and Neural ODE predictions degrade rapidly as integration time increases.
	
	\begin{figure}[htbp]
		\centering
		\includegraphics[width=1\textwidth]{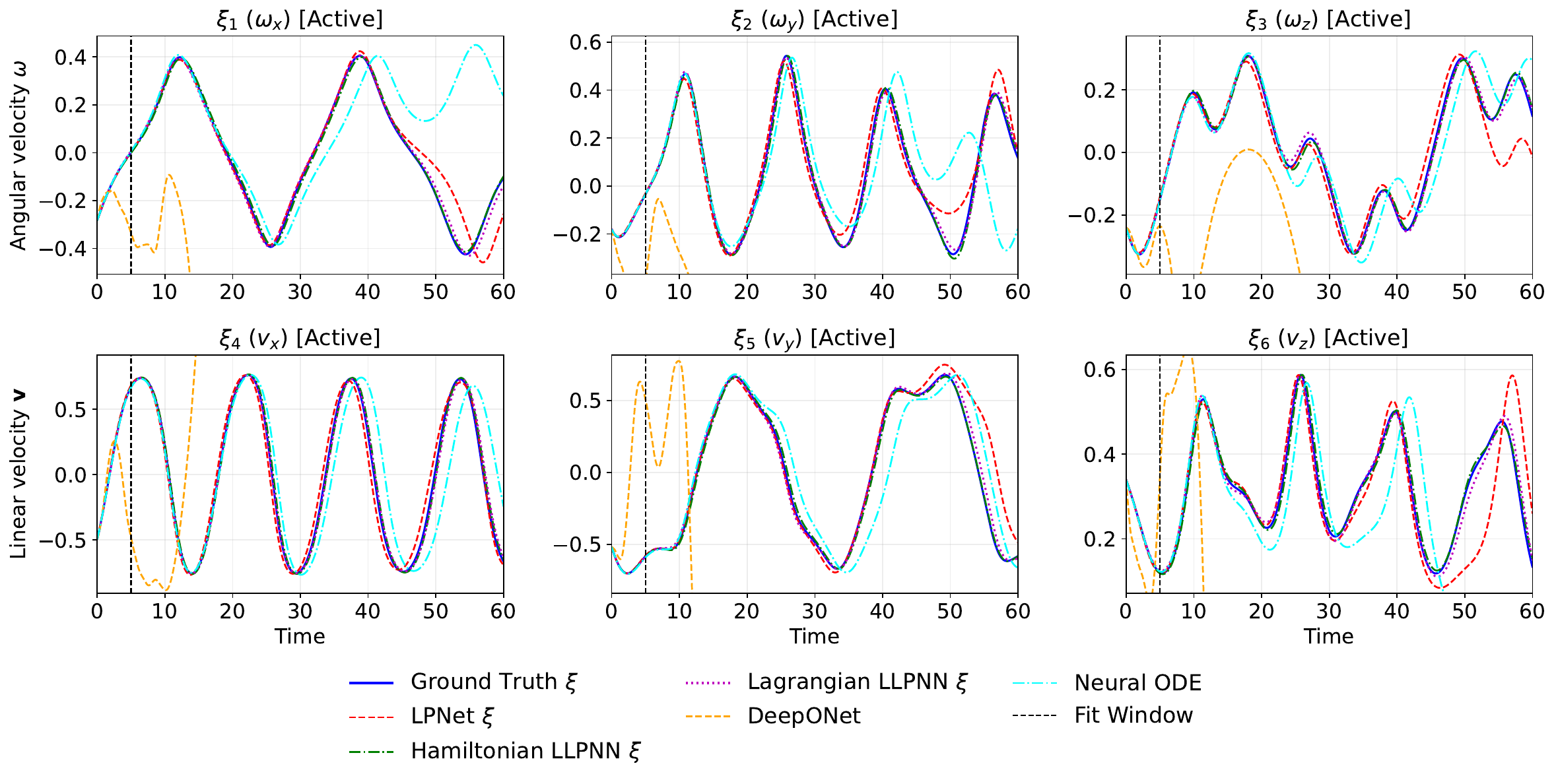}
		\caption{Predicted versus ground-truth body velocity components $\xi = (\boldsymbol{\Omega}, \mathbf{v}) = (\Omega_x, \Omega_y, \Omega_z, v_x, v_y, v_z)$ for non-quadratic $SE(3)$ Kirchhoff vehicle dynamics over time. Proposed Hamiltonian LLPNN (green dash-dotted line) stays closely aligned with ground truth (blue solid line) across all six components.}
		\label{fig:kirchhoff_xi_components}
	\end{figure}
	
	Figure~\ref{fig:kirchhoff_mu_components} presents the reconstructed latent body momentum states $\mu = (\mathbf{J}, \mathbf{P}) = (J_x, J_y, J_z, P_x, P_y, P_z)$ corresponding to the velocity trajectories shown in Figure~\ref{fig:kirchhoff_xi_components}. The Hamiltonian and Lagrangian LLPNN models maintain phase accuracy across all momentum coordinates.
	
	\begin{figure}[htbp]
		\centering
		\includegraphics[width=1\textwidth]{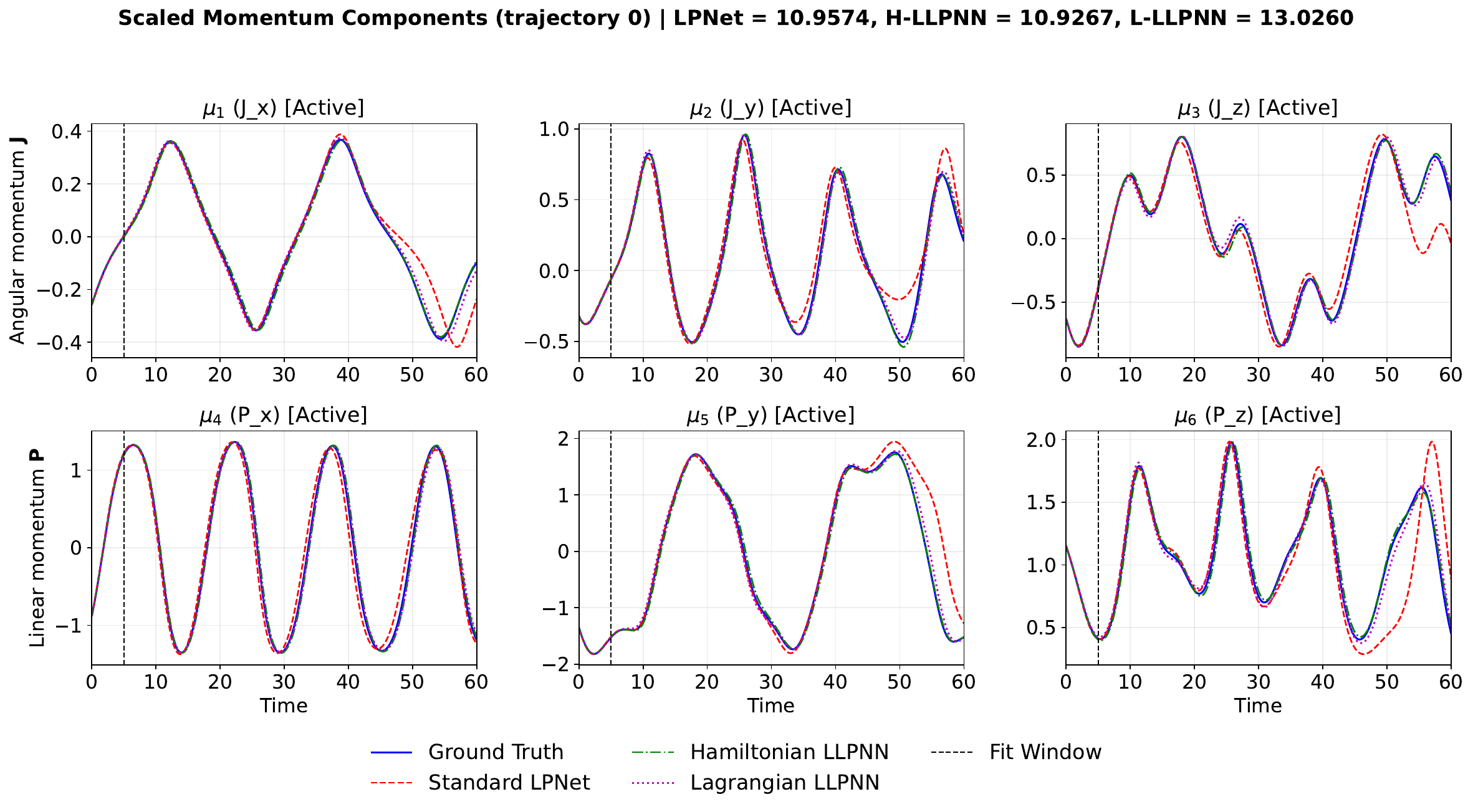}
		\caption{Reconstructed latent body momentum trajectories $\mu = (\mathbf{J}, \mathbf{P}) = (J_x, J_y, J_z, P_x, P_y, P_z)$ for the non-quadratic $SE(3)$ Kirchhoff vehicle across momentum-capable formulations.}
		\label{fig:kirchhoff_mu_components}
	\end{figure}
	
	Figure~\ref{fig:kirchhoff_mae_combined} summarizes generalization performance across $50$ test trajectories with randomized initial conditions. Mean Absolute Error (MAE) is shown separately for latent momentum ($\mu$) and observable velocity ($\xi$) coordinates. The Hamiltonian LLPNN achieves the lowest error growth over long horizons, followed by the Lagrangian LLPNN. On observable velocities, Neural ODE performs comparably to standard LPNet but accumulates error significantly faster than the proposed LLPNN models, while DeepONet diverges rapidly.
	
	\begin{figure}[htbp]
		\centering
		\includegraphics[width=1\textwidth]{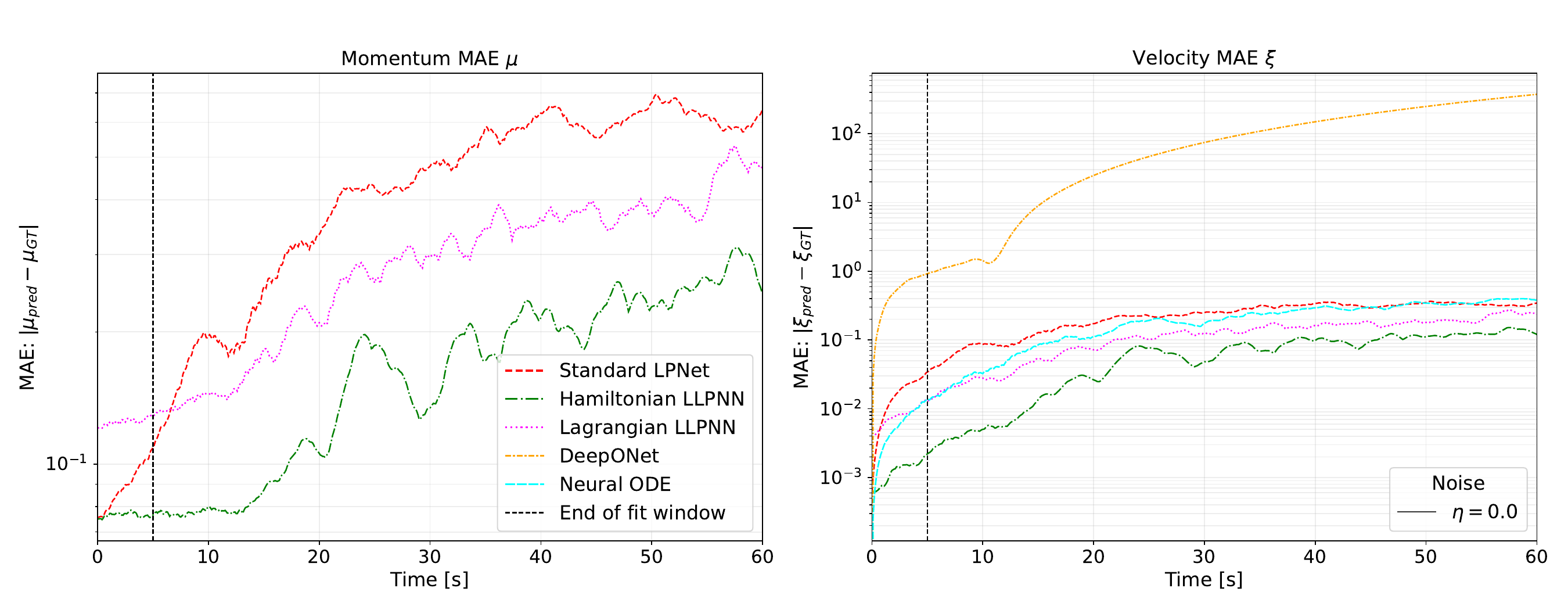}
		\caption{Time evolution of Mean Absolute Error (MAE) for body momentum $\mu$ (left) and body velocity $\xi$ (right), averaged over 50 test trajectories with random initial conditions for $SE(3)$ Kirchhoff dynamics.}
		\label{fig:kirchhoff_mae_combined}
	\end{figure}
	
	Figure~\ref{fig:kirchhoff_Casimir} evaluates conservation properties across $50$ test trajectories. The left and middle panels demonstrate that all Lie--Poisson structure-preserving formulations (Hamiltonian LLPNN, Lagrangian LLPNN, and LPNet) preserve both quadratic Casimir invariants $C_1 = \|\mathbf{P}\|^2$ and $C_2 = \mathbf{J} \cdot \mathbf{P}$ to machine precision. The right panel evaluates total energy conservation error over time. The Hamiltonian LLPNN exhibits minimal energy drift, while the Lagrangian LLPNN maintains long-term energy stability despite a minor initial offset caused by numerical gradients in the inverse Legendre map.
	
	\begin{figure}[htbp]
		\centering
		\includegraphics[width=1\textwidth]{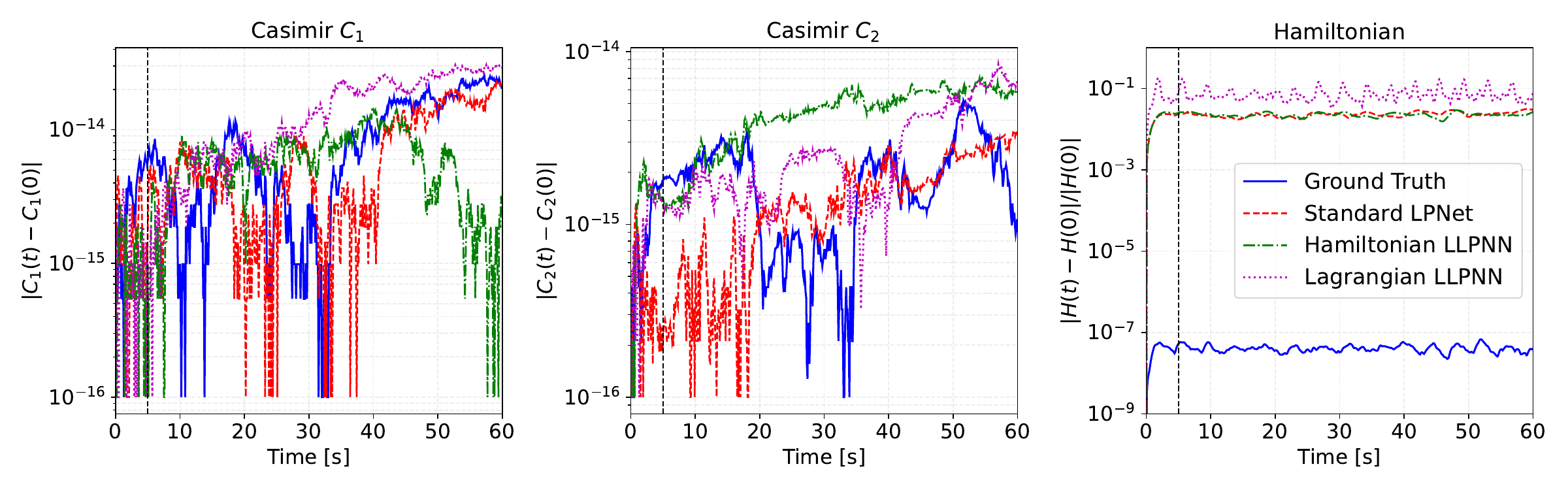}
		\caption{Conservation error metrics for non-quadratic $SE(3)$ Kirchhoff vehicle dynamics averaged over 50 test trajectories. Left and middle panels: Casimir invariant errors ($C_1 = \|\mathbf{P}\|^2$ and $C_2 = \mathbf{J} \cdot \mathbf{P}$), demonstrating machine-precision preservation for all Lie--Poisson methods. Right panel: Mean Absolute Error in energy conservation over time.}
		\label{fig:kirchhoff_Casimir}
	\end{figure}
	
	\subsection{Multi-Vehicle Optimal Control:  $\mathrm{SE}(2)^N$ Dynamics}
	\label{subsec:SE2_drones}
	
	\subsubsection{Background: Optimal Control and Symmetry Reduction}
	
	As a third and crucial benchmark demonstrating the capability of our framework, we consider interacting autonomous agents evolving on a Lie group. Specifically, we study a system of planar autonomous vehicles evolving on $SE(2)$, the special Euclidean group of rotations and translations in the plane. This example demonstrates the performance of our method for degenerate Hamiltonians, a structure typical of optimal-control formulations but rarely encountered in purely mechanical systems. We briefly summarize the underlying geometric framework to make the problem setting clear; detailed derivations can be found in \cite{justh2010extremal,justh2015optimality}.
		
	\paragraph{Pontryagin's Maximum Principle}
	For a controlled dynamical system $\dot{q} = a(q, u, t)$, Pontryagin's Maximum Principle (PMP) converts an optimal control problem into a Hamiltonian system by introducing an auxiliary co-state (or adjoint variable) $p$. The resulting Hamiltonian generates coupled canonical evolution equations for state $q$ and co-state $p$, accompanied by an optimality condition for control input $u$ through $\nabla_u H = \mathbf{0}$. In this sense, PMP plays the same fundamental role in optimal control as Hamilton's equations play in analytical mechanics.
	
	Crucially, co-states are mathematical artifacts generated by the optimization formulation. In practical applications, they are generally unobservable unless full knowledge of the optimal control model is available; they are completely absent in our observational setting, where only configuration $g$ or velocity $\xi$ on trajectories are recorded. Consequently, neither the explicit control law nor co-state dynamics are assumed known during learning.
	
	\paragraph{Symmetry Reduction}
	When agents evolve on a Lie group $G$, the configuration manifold of an $N$-agent ensemble is $G^N$. Direct application of PMP yields dynamics on the full state/co-state space, whose dimension scales rapidly with $N$. However, when the underlying optimal control problem exhibits Lie-group invariance, PMP dynamics can be reduced via geometric symmetry reduction theory \cite{krishnaprasad1993optimal,Ho2011_pII,marsden2013introduction}.
	
	For general interacting systems, symmetry reduction yields momentum variables coupled to relative spatial configurations. A key result of \cite{justh2010extremal,justh2015optimality} establishes that for the class of optimal control problems considered here, the reduced dynamics close entirely on symmetry-reduced momenta. Consequently, the reduced evolution is governed by a Lie--Poisson system on $(\mathfrak{g}^*)^N$, eliminating the need to explicitly track relative agent orientations.
	
	This reduction offers substantial practical savings: rather than learning dynamics on a high-dimensional state/co-state space, the evolution simplifies to a lower-dimensional Lie--Poisson system structured by group symmetries. Throughout this study, we assume the symmetry class of the configuration group is known, while control laws, cost functionals, and co-state trajectories remain entirely unknown.
	\paragraph{Group Configuration and Inactive Momentum Variables}
	For each vehicle $k \in \{1, \dots, N\}$, the planar configuration matrix $g_k \in \mathrm{SE}(2)$ and its body-fixed momentum vector $\mu_k \in \mathfrak{se}(2)^* \cong \mathbb{R}^3$ are given by:
	\begin{equation}
		g_k = \begin{pmatrix} \cos\theta_k & -\sin\theta_k & x_k \\ \sin\theta_k & \cos\theta_k & y_k \\ 0 & 0 & 1 \end{pmatrix}, \quad \mu_k = \begin{pmatrix} \mu_{k1} \\ \mu_{k2} \\ \mu_{k3} \end{pmatrix},
	\end{equation}
	where $\mu_{k1}$ and $\mu_{k2}$ denote longitudinal and lateral linear body momenta, respectively, and $\mu_{k3}$ denotes angular body momentum. The total momentum state of the $N$-vehicle ensemble is the concatenated vector $\mu = (\mu_1^\top, \dots, \mu_N^\top)^\top \in \mathbb{R}^{3N}$.
	
	The Poisson structure matrix $\Lambda(\mu)$ for this system is block-diagonal:
	\begin{equation}
		\Lambda(\mu) = \operatorname{diag}\bigl(\Lambda_1(\mu_1), \dots, \Lambda_N(\mu_N)\bigr) ,
	\end{equation}
	where each $3 \times 3$ block $\Lambda_k(\mu_k)$ is given by:
	\begin{equation}
		\Lambda_k(\mu_k) = \begin{pmatrix} 0 & 0 & -\mu_{k2} \\ 0 & 0 & \mu_{k1} \\ \mu_{k2} & -\mu_{k1} & 0 \end{pmatrix}.
	\end{equation}
	Consequently, the dynamics of the $k$-th vehicle satisfy:
	\begin{equation}
		\dot{\mu}_k = \Lambda_k(\mu_k) \nabla_{\mu_k} h(\mu),
		\label{SE2_particle_dynamics}
	\end{equation}
	or, expressed explicitly in scalar components:
	\begin{align}
		\dot{\mu}_{k1} &= -\mu_{k2} \frac{\partial h}{\partial \mu_{k3}}, \\
		\dot{\mu}_{k2} &= \mu_{k1} \frac{\partial h}{\partial \mu_{k3}}, \\
		\dot{\mu}_{k3} &= \mu_{k2} \frac{\partial h}{\partial \mu_{k1}} - \mu_{k1} \frac{\partial h}{\partial \mu_{k2}}.
		\label{SE2_particle_dynamics_components}
	\end{align}
	
	The collective multi-vehicle Hamiltonian $h(\mu)$ governing target ensemble dynamics is defined as \cite{justh2015optimality}:
	\begin{equation}
		h(\mu) = \sum_{k=1}^N \mu_{k1} + \frac{1}{2} \sum_{k,m=1}^N \Psi_{km} \mu_{k3} \mu_{m3},
		\label{H_true} 
	\end{equation}
	where $\Psi \in \mathbb{R}^{N \times N}$ represents a symmetric steering coordination matrix. We consider a fully coupled network topology (the \emph{Democracy} topology). For a three-agent ensemble ($N=3$), the graph Laplacian $B$ and coordination matrix $\Psi$ (evaluated with coupling parameter $\chi = 0.5$) are:
	\begin{equation}
		B = \begin{bmatrix} 2 & -1 & -1 \\ -1 & 2 & -1 \\ -1 & -1 & 2 \end{bmatrix}, 
		\quad 
		\Psi = (I_3 + 2\chi B)^{-1} = (I_3 + B)^{-1} = \begin{bmatrix} \frac{1}{2} & \frac{1}{4} & \frac{1}{4} \\[2pt] \frac{1}{4} & \frac{1}{2} & \frac{1}{4} \\[2pt] \frac{1}{4} & \frac{1}{4} & \frac{1}{2} \end{bmatrix}.
		\label{Psi_def} 
	\end{equation}
	
	A distinguishing property of this Lie--Poisson system is the existence of  passive variables  in momentum space. Because networking terms in $h(\mu)$ depend exclusively on angular momentum components $\mu_{k3}$, translational components $\mu_{k1}$ and $\mu_{k2}$ remain uncoupled from network interactions. Evaluating energy gradients yields Lie algebra velocities $\xi_k = \nabla_{\mu_k} h(\mu) \in \mathfrak{se}(2)$:
	\begin{equation}
		\xi_k = \begin{pmatrix} \xi_{k1} \\ \xi_{k2} \\ \xi_{k3} \end{pmatrix} = \begin{pmatrix} 1.0 \\ 0.0 \\ \sum_{m=1}^N \Psi_{km} \mu_{m3} \end{pmatrix}.
	\end{equation}
	The longitudinal momentum coordinates $\mu_{k1}$ generate uniform, constant forward drive ($\xi_{k1} = 1.0$), while lateral components $\mu_{k2}$ receive zero energetic drive ($\xi_{k2} = 0.0$). Thus, translational terms enter linearly or are absent in the Hamiltonian, isolating active coordination dynamics to the angular momentum subspace.
	\paragraph{Non-Linear Generalized Benchmark Hamiltonian}
	To make the learning task more demanding and represent complex control costs, we introduce a non-quadratic modification to \eqref{H_true}:
	\begin{equation}
		H_{\text{true}}(\mu) = \sum_{k=1}^N \mu_{1k} + \frac{1}{2} \left( \sum_{k,m=1}^3 \Psi_{km} \mu_{k3} \mu_{m3} \right) \left( 1 + B \sum_{k=1}^3 \mu_{k3}^2 \right) \, , 
		\label{H_nonlinear_SE2} 
	\end{equation}
	where $B = 0.01$ in all numerical experiments.

	\subsubsection{Ground truth data generation for training and testing}
	
For training, we generate 200 trajectories, each containing 101 time points. Since the system consists of $N=3$ interacting vehicles on $SE(2)$, the latent momentum variable has dimension $3N=9$. The initial momenta are sampled from the standard normal distribution,
\[
\mu^{\mathrm{train}}_0 \sim \mathcal{N}(0,I_9),
\]
where $I_9$ is the $9\times 9$ identity matrix. For testing, we generate 50 trajectories, each containing 2001 time points, from the same initial momentum distribution,
\[
\mu^{\mathrm{test}}_0 \sim \mathcal{N}(0,I_9).
\]
Just as in the other examples, the time step is fixed at $\Delta t=0.1$ throughout. For each vehicle, the initial orientation angle is sampled independently from the uniform distribution on $[0,2\pi)$, while the initial position is sampled from the standard normal distribution in $\mathbb{R}^2$.
	
	\subsubsection{Robustness to Observational Noise}
	To assess practical applicability under sensory degradation, observation data are corrupted using noise scale parameter $\eta$. For each vehicle, additive noise is added to Lie algebra velocities: $\xi_{k,\text{noisy}} = \xi_{k} + \mathcal{N}(\mathbf{0}, \eta^2 \mathbf{I}_3)$. Concurrently, configuration matrices are perturbed within their body frame via multiplicative intrinsic perturbations:
	\begin{equation}
		g_{k,\text{noisy}} = g_k \cdot \exp(\widehat{\omega}_k), \quad \widehat{\omega}_k = \begin{pmatrix} 0 & -d\theta_k/\sqrt{2} & 0 \\ d\theta_k/\sqrt{2} & 0 & 0 \\ 0 & 0 & 0 \end{pmatrix},
	\end{equation}
	where $d\theta_k \sim \mathcal{N}(0, \eta^2)$ represents stochastic orientation noise. This noise model evaluates the framework's ability to recover true underlying mechanics and hidden initial momenta from degraded observations. We highlight results for $\eta = 0$ (clean data) and $\eta = 0.005$ (noisy data); evaluations across intermediate noise levels $\eta \in [0.001, 0.005]$ yield consistent trends.

	\subsubsection{Experimental Results}
	
	The learning procedure processes trajectory data generated on $\mathrm{SE}(2)^N$ under Democracy coupling \cite{huraka2026structure} (all particles coupled to each other) for $N=3$. Dynamic variable isolation is performed by identifying coordinates with non-zero variance across trajectories, partitioning velocity and momentum states into active and passive components: $\xi = (\xi^{\text{act}}, \xi^{\text{pass}})$ and $\mu = (\mu^{\text{act}}, \mu^{\text{pass}})$.
	
	A global quadratic warm start is executed across all training trajectories to establish an initial quadratic baseline and estimate initial spatial co-states $p_0$. This step solves a constrained linear Karush--Kuhn--Tucker (KKT) system \cite{nocedal2006numerical} to yield a symmetric positive-definite metric matrix $\mathbb{B}$ (subject to $\operatorname{tr}\mathbb{B} = 1$) along with initial spatial momenta $p_0$ per trajectory. As in previous examples, training considers two LLPNN architectures:
	\begin{enumerate}
		\item \textbf{Lagrangian LLPNN:} The residual Lagrangian $\ell(\xi^{\text{act}}) = \frac{1}{2}\xi^{\text{act}} \cdot \mathbb{B}\xi^{\text{act}} + \mathcal{R}_L(\xi^{\text{act}})$ is optimized via gradient descent by matching predicted coadjoint-projected spatial momenta against $\frac{\partial \ell}{\partial \xi^{\text{act}}}$.
		\item \textbf{Hamiltonian LLPNN:} The Hamiltonian $h(\mu)$ is initialized using warm-start inverse metric $\mathbb{K} = \mathbb{B}^{-1}$ and optimized alongside Lie--Poisson flow parameter networks (LPNet) to enforce structure preservation.
	\end{enumerate}
	
	Figure~\ref{fig:xi_components_SE2} compares the reconstructed body velocity components $\xi(t)$ across evaluated methods against the ground truth for a representative trajectory. In this system, only the angular velocity components ($\xi_{k3}$) are active and evolve dynamically, whereas the translational velocity components ($\xi_{k1}, \xi_{k2}$) remain constant passive variables. The Lagrangian LLPNN accurately tracks angular velocity dynamics over a long time horizon ($t=100$), whereas the Hamiltonian LLPNN exhibits minor phase drift after $t \approx 40$. As expected from the discussion in Section~\ref{subsec:difference_Ham_Lagr}, non-structure-preserving baselines (DeepONet and Neural ODE) diverge rapidly from the ground truth.
	
	\begin{figure}[htbp]
		\centering
		\includegraphics[width=1\textwidth]{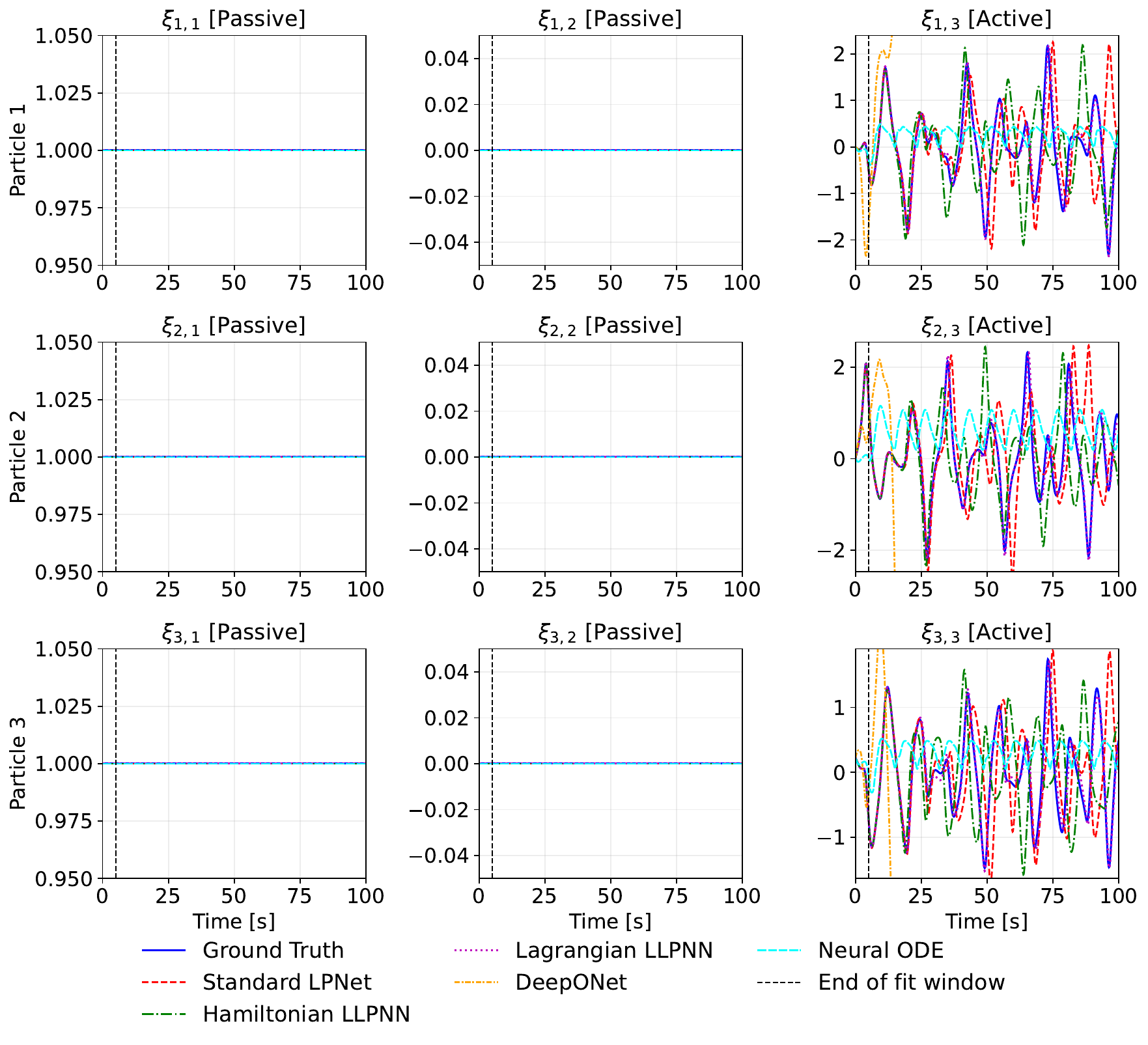}
		\caption{Time evolution of body velocity components $\xi = (\xi_{k1}, \xi_{k2}, \xi_{k3})$ for each vehicle ($N=3$). Passive translational velocities ($\xi_{k1}, \xi_{k2}$) remain constant, while active steering velocities ($\xi_{k3}$) evolve dynamically. Proposed Lagrangian LLPNN (dotted magenta line) maintains high-precision alignment with ground truth (blue solid line) over long integration horizons.}
		\label{fig:xi_components_SE2}
	\end{figure}
	
	Figure~\ref{fig:mu_components_SE2} presents corresponding reconstructed momentum components $\mu(t)$. Unlike velocity space, passive momentum components evolve over time due to geometric cross-coupling in the Lie--Poisson bracket. The Lagrangian LLPNN maintains close agreement with true momentum trajectories throughout the evaluated horizon.
	
	\begin{figure}[htbp]
		\centering
		\includegraphics[width=1\textwidth]{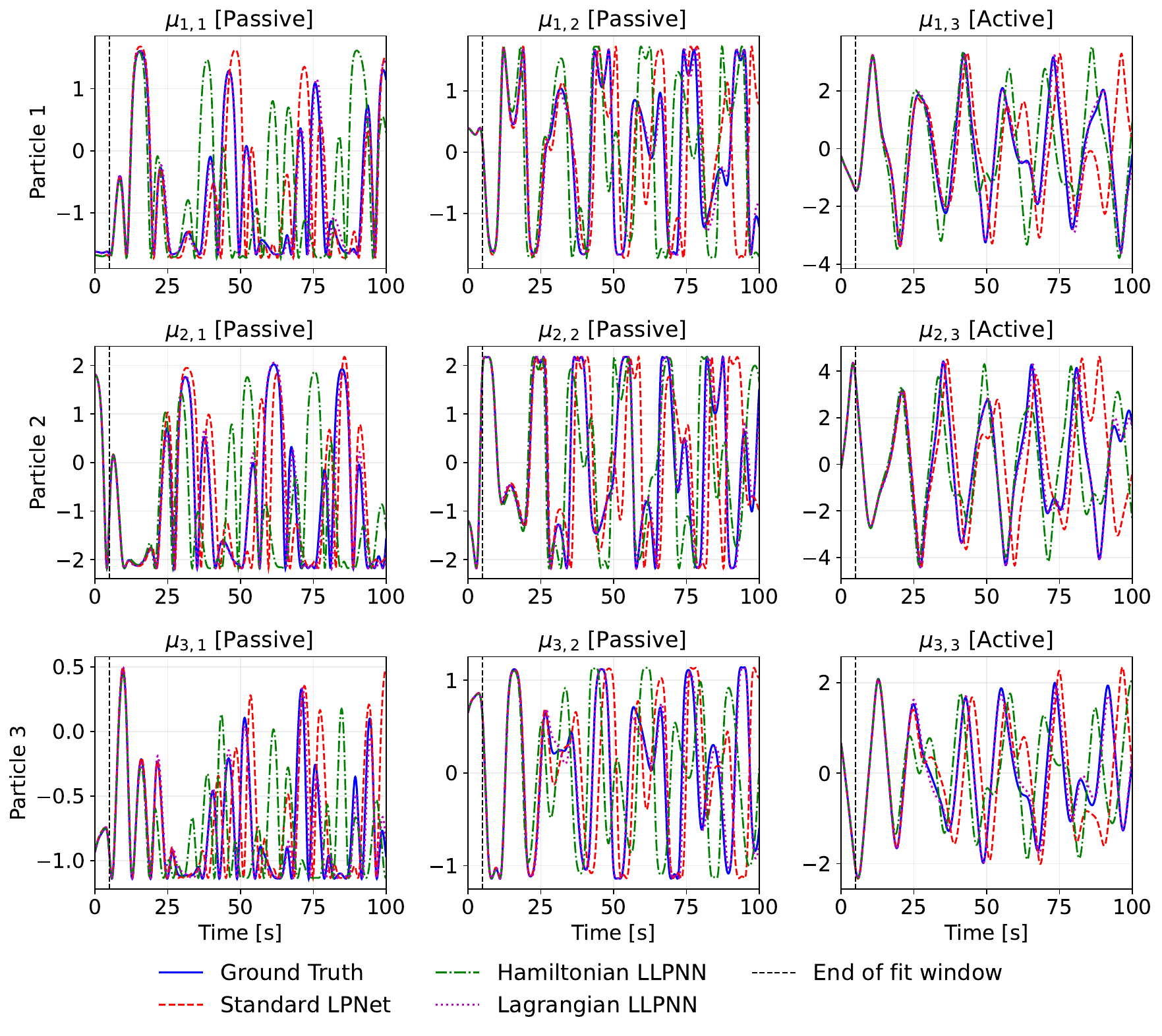}
		\caption{Time evolution of scaled body momentum components $\mu(t)$ for each vehicle ($N=3$). Both active ($\mu_{k3}$) and passive ($\mu_{k1}, \mu_{k2}$) momentum channels evolve dynamically due to Lie--Poisson bracket coupling. The Lagrangian LLPNN (dotted magenta line) closely tracks the ground-truth trajectory across the entire evaluation horizon.} 
		\label{fig:mu_components_SE2}
	\end{figure}
	
	Figure~\ref{fig:se2_individual} illustrates reconstructed 2D spatial trajectories $(x_k(t), y_k(t))$ obtained by integrating predicted velocities for the trajectory shown in Figure~\ref{fig:xi_components_SE2}. The Lagrangian LLPNN reproduces intricate multi-vehicle flight paths over long integration windows ($t=100$) with negligible spatial drift, while standard LPNet exhibits trajectory divergence over shorter time intervals.
	
	\begin{figure}[htbp]
		\centering
		\includegraphics[width=1\textwidth]{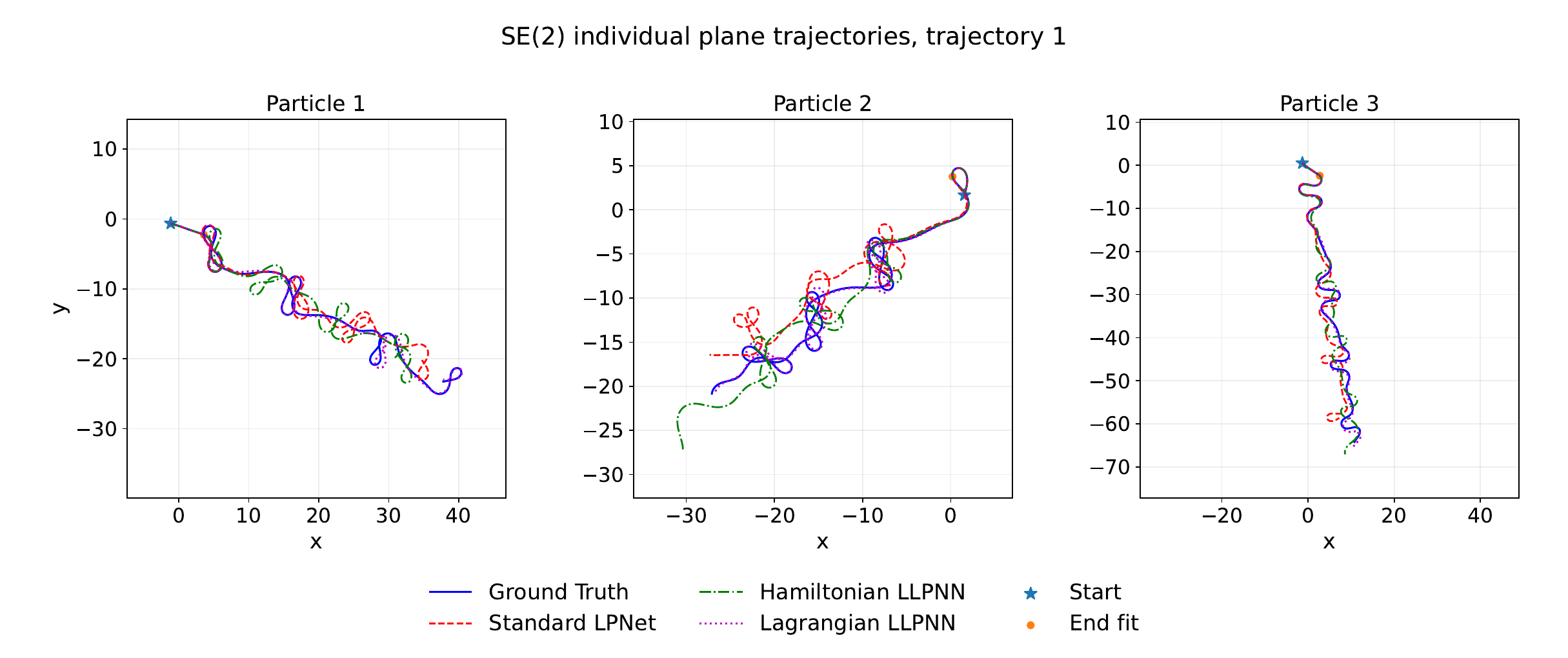}
		\caption{Reconstructed 2D spatial trajectories $(x_k(t), y_k(t))$ for $N=3$ vehicles on the $SE(2)$ plane corresponding to the predictions in Figure~\ref{fig:xi_components_SE2}. The Lagrangian LLPNN (dotted magenta line) remains visually indistinguishable from ground-truth paths (solid blue line) up to $t=100$.}
		\label{fig:se2_individual}
	\end{figure}
	
	To assess overall quality of learning, Figure~\ref{fig:mae_combined_SE2} presents Mean Absolute Error (MAE) for momentum $\mu$ and velocity $\xi$ averaged across $50$ test trajectories with random initial conditions, under both noiseless ($\eta=0$) and noisy ($\eta=0.005$) conditions. On average, the Hamiltonian LLPNN achieves the lowest mean error growth across the ensemble, closely followed by the Lagrangian LLPNN. Non-geometric baselines struggle on this task due to singular velocity-space vector fields induced by Hamiltonian degeneracy.
	
	\begin{figure}[htbp]
		\centering
		\includegraphics[width=1\linewidth]{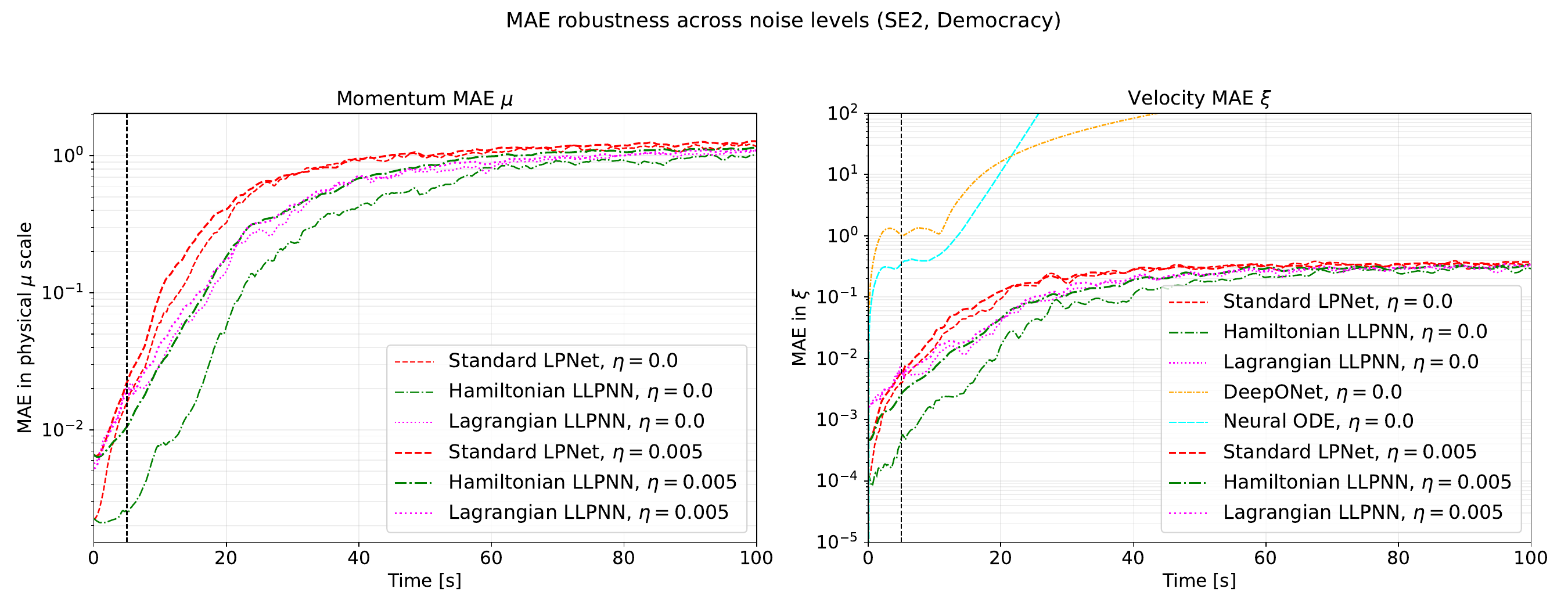}
		\caption{Mean Absolute Error (MAE) for scaled momentum $\mu$ (left) and velocity $\xi$ (right) averaged over 50 test trajectories, evaluated under clean ($\eta=0$) and noisy ($\eta=0.005$) observational conditions. Non-structure-preserving baselines, DeepONet (dashed orange line) and Neural ODEs (dash-dotted cyan line), diverge rapidly over time. On average, the Hamiltonian LLPNN (dashed green line) achieves the highest predictive accuracy across both latent momenta and observable velocities.}
		\label{fig:mae_combined_SE2}
	\end{figure}
	
	For each vehicle $k$, the individual Casimir invariant is given by $C_k = \mu_{k1}^2 + \mu_{k2}^2$. Figure~\ref{fig:casimir_errors_SE2} confirms that all Lie--Poisson methods preserve individual particle Casimirs to machine precision throughout long-horizon integration.
	
	\begin{figure}[htbp]
		\centering
		\includegraphics[width=1\linewidth]{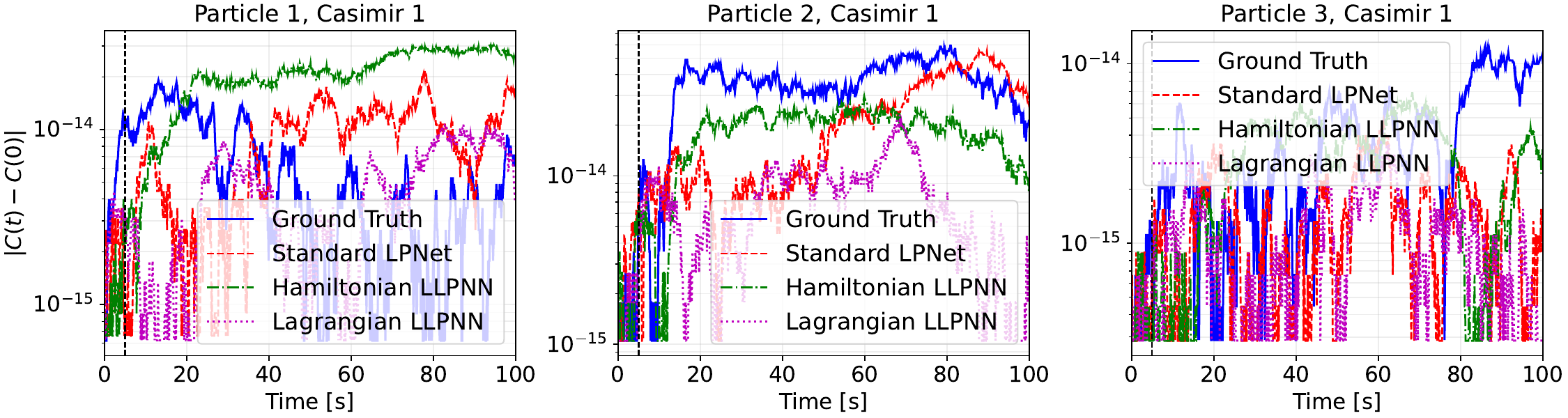}
		\caption{Casimir invariant conservation error $|C_k(t) - C_k(0)|$ evaluated across individual vehicle momentum components, confirming machine-precision preservation for Lie--Poisson formulations.}
		\label{fig:casimir_errors_SE2}
	\end{figure}
	
	Figure~\ref{fig:hamiltonian_drift_SE2} illustrates energy conservation performance for the Hamiltonian LLPNN. Even under non-quadratic, degenerate Hamiltonian dynamics \eqref{H_nonlinear_SE2}, the Hamiltonian LLPNN maintains long-term energy preservation with low relative drift across clean and noisy training regimes. Energy errors for the Lagrangian model are omitted here due to the ill-posedness of standard Legendre transforms on degenerate momentum directions.
	
	\begin{figure}[htbp]
		\centering
		\includegraphics[width=1\linewidth]{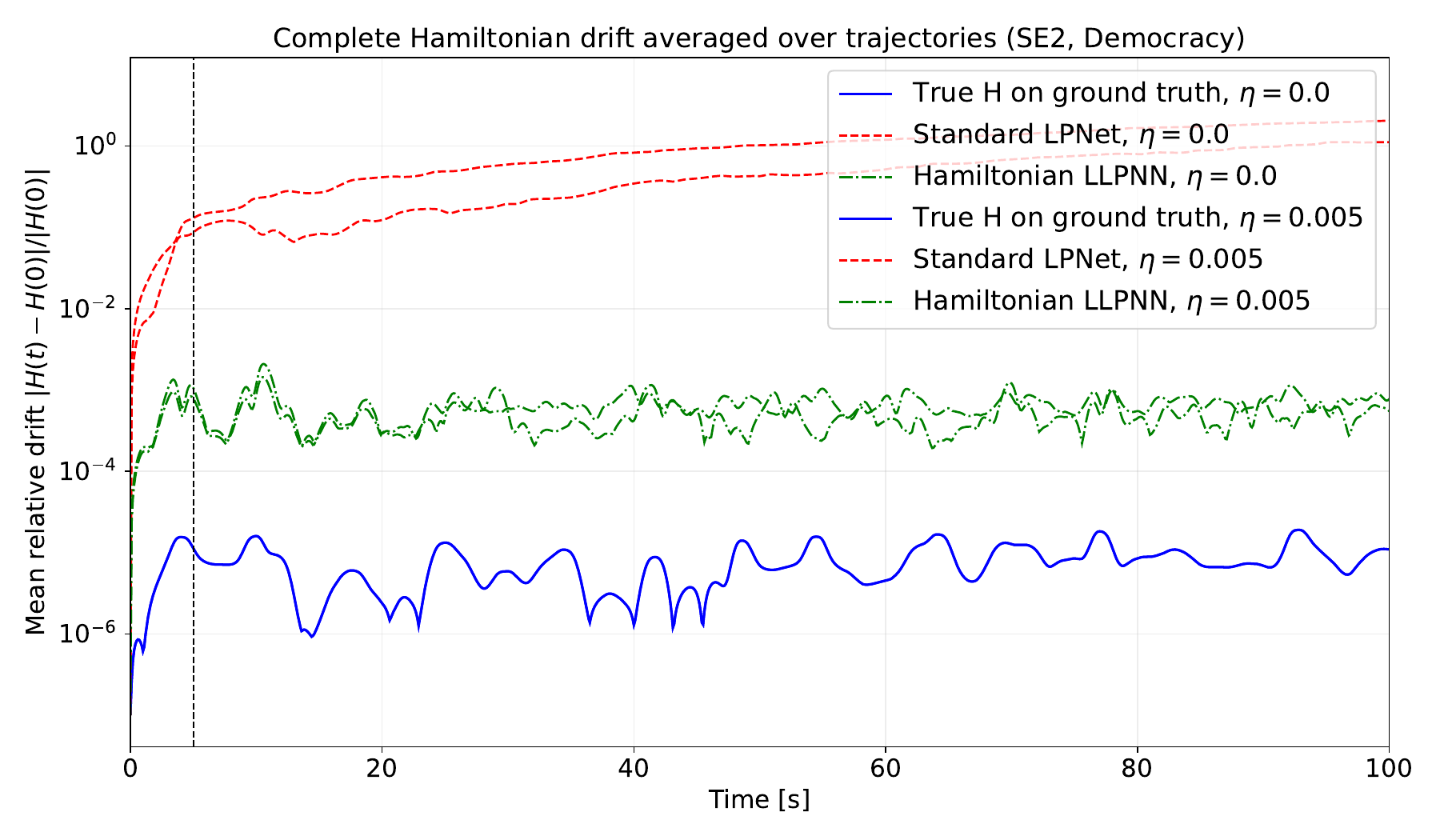}
		\caption{Relative Hamiltonian conservation drift $|H(t) - H(0)| / |H(0)|$ for the Hamiltonian LLPNN under clean ($\eta=0$) and noisy ($\eta=0.005$) conditions.}
		\label{fig:hamiltonian_drift_SE2}
	\end{figure}
	
	Finally, Figure~\ref{fig:hamiltonian_contours} directly compares level sets of learned versus true Hamiltonian functions in active momentum space. The left panel shows contour lines sliced at $\mu_{1,3} = 0.5$, demonstrating excellent agreement between the predicted and true Hamiltonian. The right panel displays absolute error contours, confirming that energy function discrepancies across active momentum regions remain below $10^{-3}$.
	
	\begin{figure}[htbp]
		\centering
		\includegraphics[width=1\linewidth]{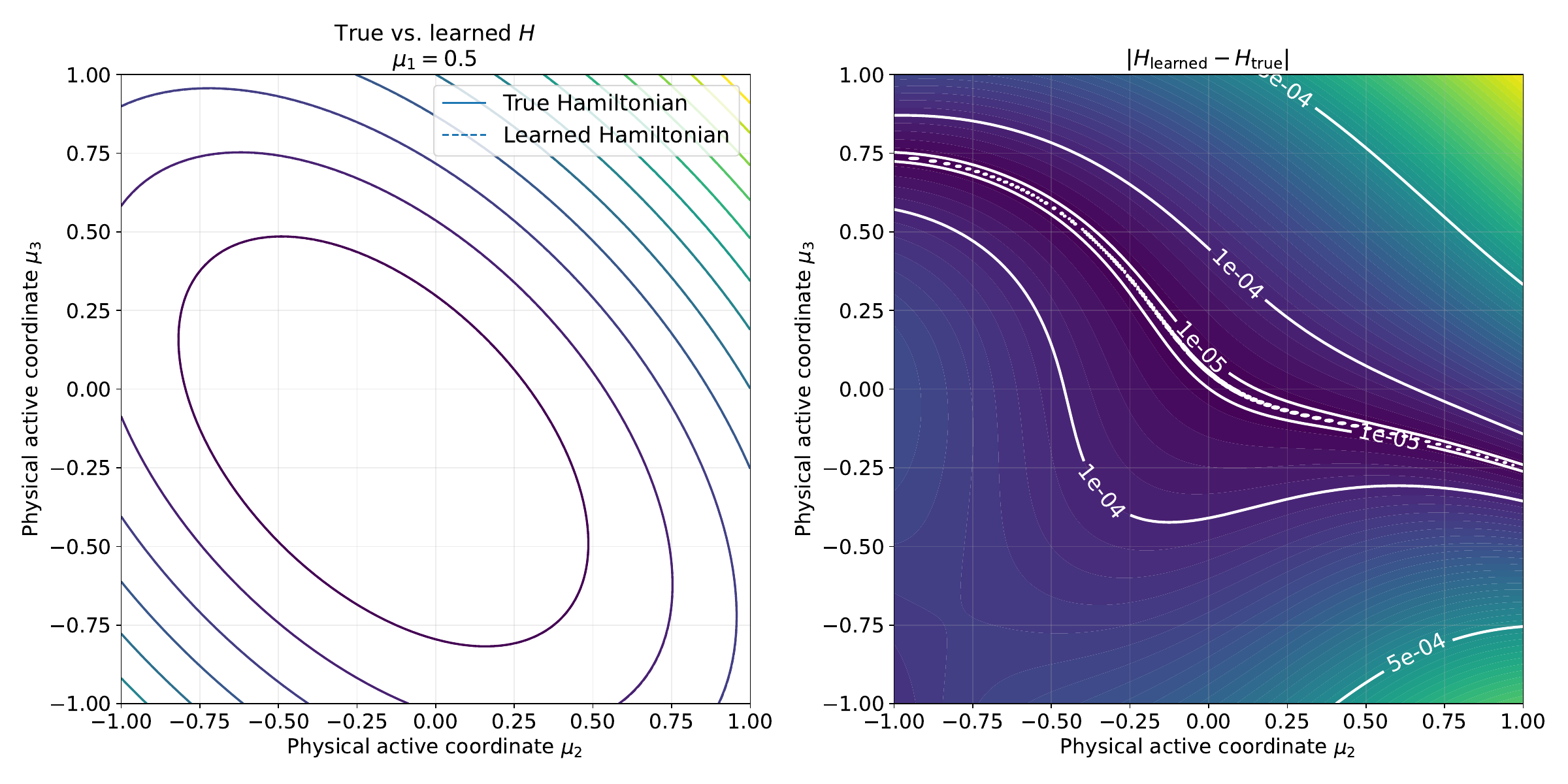}
		\caption{Hamiltonian level-set comparison in active momentum space ($\mu_{23}$ vs. $\mu_{33}$) at slice $\mu_{13} = 0.5$. Left: Overlay of true (solid) and learned (dashed) Hamiltonian level sets. Right: Absolute approximation error $|H_{\text{learned}} - H_{\text{true}}|$, demonstrating uniform accuracy with errors bounded near $\sim 10^{-4}$.}
		\label{fig:hamiltonian_contours}
	\end{figure}
	
	\section{Limitations and Computational Considerations}
	\label{sec:discussion}
	
	\paragraph{Computational efficiency} The proposed LLPNN framework achieves high sample and parameter efficiency by operating on a compact network architecture. With each constituent network comprising three hidden layers of width 32, the parameter footprint remains modest ($\approx 2\,\text{K}$ parameters per network, supplemented by a few parameters specifying the quadratic core for the warm-start initialization). Due to the relatively low-dimensional structure-preserving parameterization, training remained computationally modest. For the Kirchhoff problem on $SE(3)$, a complete training cycle of $20{,}000$ epochs required approximately $400$ s for the Lagrangian LLPNN and approximately $525$ s for the Hamiltonian LLPNN on standard consumer hardware (Apple M2 Max); the timings for the other examples were comparable.
	
	Despite these practical advantages, the framework possesses certain theoretical boundary conditions and practical limitations that warrant discussion.
	
	\paragraph{Limitation 1: Assumption of known symmetry} 
	Our formulation explicitly assumes that the underlying Lie group symmetry $G$ of the physical system is known \emph{a priori}. While this assumption holds for a broad class of conservative mechanical, fluid, and rigid-body control systems, it may be restrictive in certain optimal control settings. For instance, multi-agent vehicles actively searching for a target fixed in space exhibit symmetry-breaking cost functions \cite{justh2010extremal,justh2015optimality}, preventing complete symmetry reduction. Extending the LLPNN framework to account for such symmetry-breaking dynamics \cite{bloch2017optimal} remains an important open problem.
	
	\paragraph{Limitation 2: Extension to general Poisson manifolds} 
	While LLPNN is engineered specifically for Lie--Poisson dynamics, one may naturally inquire about its applicability to general Poisson systems. At present, the method cannot be directly applied to arbitrary Poisson manifolds because generic Poisson brackets lack the explicit Lie-algebraic coadjoint reconstruction formula \eqref{eq:coadjoint_reconstruction} that forms the core of our latent variable reconstruction procedure. Developing a generalized geometric momentum-fitting technique for non-Lie Poisson manifolds represents an intriguing theoretical challenge.
	
	\section{Conclusion and Future Work}
	\label{sec:conclusion}
	
	In this paper, we presented a physics-informed geometric learning framework designed to infer Lie--Poisson Hamiltonian dynamics directly from observable data, without requiring measurements of latent momentum variables. The core paradigm relies on the structural assumption that the system possesses underlying Lie--Poisson dynamics generated by a Hamiltonian. Crucially, our framework explicitly accommodates degenerate Hamiltonians, which lack dual Lagrangian representations or standard Euler--Poincar\'e formulations in observable velocity space. For such systems, the observable variables need not form a closed
	dynamical system and may not admit an autonomous evolution equation.
	LLPNNs circumvent this difficulty by reconstructing the latent
	Lie--Poisson dynamics through Noether invariants and coadjoint
	reconstruction. By embedding the Lie-algebraic structure directly into the neural network architecture, our framework guarantees exact structural preservation of Lie--Poisson brackets and achieves machine-precision conservation of Casimir invariants.
	
	We validated our approach across three challenging benchmark systems: generalized rigid-body motion on $\mathrm{SO}(3)$, a Kirchhoff underwater vehicle on $\mathrm{SE}(3)$, and an $N$-agent robotic swarm on $\mathrm{SE}(2)^N$, all of which exhibit complex non-linear behavior. Numerical experiments demonstrate that respecting the intrinsic geometric structure of the Poisson manifold yields superior long-term energy conservation and trajectory prediction accuracy compared to unconstrained baseline architectures. Using observable data alone, our framework successfully recovers full dynamical information and enables accurate future forecasting for both observable and latent variables, even under degenerate Hamiltonian conditions.
	
	The proposed framework reconstructs latent momentum trajectories through
	Noether invariants and the associated coadjoint action. Consequently,
	successful reconstruction requires that the observable variables contain
	sufficient information to determine the conserved quantity $p_0$ over a
	short fitting window. While this condition is satisfied for the examples
	considered in this paper, more complex observation operators may require
	longer fitting windows or additional measurements. Understanding
	observability properties of Lie--Poisson systems in the context of latent
	state reconstruction remains an open question, and an interesting direction for further study. 
	
	Several promising avenues exist for future research. First, extending the framework to model non-conservative effects---such as Rayleigh dissipation functions and external fluid disturbances---via structure-preserving variational integrators \cite{eldred2025variational} will enable accurate modeling in dissipative physical environments. Second, incorporating nonholonomic constraints into the Lie--Poisson formulation using structure-preserving integrators \cite{beckers2026structure} will generalize the framework to ground and aerial robotic platforms. Finally, extending the method to accommodate symmetry-breaking cost functions \cite{bloch2017optimal} represents both an important practical capability and a rich theoretical avenue for geometric machine learning.
	
	\section*{Acknowledgements}
	
	The author acknowledges fruitful and inspiring discussions with Drs. C. Eldred, F. Gay-Balmaz, A. Gruber, A. Sinclair, I. Tezaur, and D. V. Zenkov. This work was supported in part by the Shelby Foundation at the University of Alabama.
	
	
	\section*{CRediT Authorship Contribution Statement}
VP: Conceptualization, Methodology, Formal analysis, Software, Validation, Investigation, Data curation, Writing -- original draft, Writing -- review and editing, Visualization, Project administration.
	
	\section*{Declaration of Competing Interest}
	The author declares that he has no known competing financial interests or personal relationships that could have appeared to influence the work reported in this paper.
	
	\section*{Data Availability Statement}
	The source code, simulation scripts, and dataset generation routines for the Lie--Poisson Neural Network (LLPNN) architectures presented in this study are publicly available at \url{https://github.com/vputkaradze/LLPNNs}. All benchmark data can be reproduced using the provided codebase.
	
	\section*{Declaration of Generative AI and AI-Assisted Technologies}
	During the preparation of this work, the author used \emph{Gemini Pro}
	and \emph{ChatGPT} to assist with code development, improve code
	readability, and perform language editing of the manuscript. The author
	reviewed and edited all generated material and verified the accuracy of
	the mathematical, scientific, and technical content. The author takes
	full responsibility for the content of this publication.
	
	\bibliographystyle{unsrt}
	\bibliography{bibliography}

	\appendix 
	\section{Lie Groups and their actions} 
	\label{app:Lie_actions} 
	\subsection{General definitions} 
	In this short exposition, we closely follow the notation and definitions of \cite{MaRa2013}. A Lie group is a group $G$ that is also a smooth manifold, such that the group multiplication $(g,h)\mapsto gh$ and inversion $g\mapsto g^{-1}$ are smooth maps. The group structure satisfies the following standard properties:
	\begin{enumerate}
		\item There exists an identity element $e \in G$ such that $e g = g e = g$ for any $g \in G$, 
		\item For any $g,h \in G$, the product satisfies $g h \in G$, 
		\item For any $g \in G$, there exists an inverse element $g^{-1} \in G$ such that $g g^{-1} = g^{-1} g = e$. 
	\end{enumerate}
	As shown below, the tangent space at the identity is naturally endowed with a commutator bracket. To define this bracket, we first introduce the relevant group actions.
	
	\paragraph{Adjoint Group Actions} 
	Suppose we have smooth curves $g(s) \in G$ and $h(t) \in G$ passing through the identity $e$ at parameter zero, such that $g(0) = h(0) = e$. Let their velocity vectors at the identity be denoted by $g'(0) = \xi$ and $h'(0) = \eta$, which are elements of the tangent space $T_e G$. 
	
	\paragraph{The $\operatorname{AD}$ Operator} We define the conjugation map (inner automorphism) of $G$ by a fixed element $g \in G$ as: 
	\begin{equation}
		\operatorname{AD}_g h = g h g^{-1} \, . 
		\label{AD_def}
	\end{equation}
	
	\paragraph{The $\operatorname{Ad}$ Operator} Next, we define the Adjoint action of the group on the Lie algebra. By keeping $g$ fixed and differentiating \eqref{AD_def} with respect to the parameter $t$ along the curve $h(t)$, we obtain: 
	\begin{equation}
		\operatorname{Ad}_g \eta = \left. \frac{d}{dt} \left( g h(t) g^{-1} \right) \right|_{t=0} = 
		g \eta g^{-1} \, , \quad \text{where } \eta := h'(0) \in T_e G \, . 
		\label{Ad_def}
	\end{equation}
	This map, $\operatorname{Ad}_g: T_e G \to T_e G$, represents a linear representation of the group on its tangent space at the identity.
	
	\paragraph{The $\operatorname{ad}$ Operator} Now, we let $g = g(s)$ vary along its curve while keeping the algebra element $\eta$ fixed. Differentiating $\operatorname{Ad}_{g(s)} \eta$ with respect to $s$ at $s = 0$ yields the infinitesimal adjoint action, denoted by $\operatorname{ad}$: 
	\begin{equation}
		\operatorname{ad}_\xi \eta = \left. \frac{d}{ds} \left( \operatorname{Ad}_{g(s)} \eta \right) \right|_{s=0} = 
		\xi \eta - \eta \xi = [ \xi, \eta ] \, , \quad \text{where } \xi := g'(0) \in T_e G \, . 
		\label{ad_def}
	\end{equation}
	The map $\operatorname{ad}_\xi \eta$ naturally introduces the Lie bracket commutator to the tangent space at the identity. 
	
	\begin{definition}[Lie Algebra]
		The tangent space to the Lie group $G$ at the identity $g=e$ is defined as the Lie algebra $\mathfrak{g} \cong T_e G$ of the Lie group $G$, equipped with the Lie bracket $[\cdot, \cdot]$ defined in \eqref{ad_def}. 
	\end{definition}
	
	\paragraph{Coadjoint Actions}
	To study the dynamics of momenta, we transition to the dual space. Let $\langle \cdot, \cdot \rangle : \mathfrak{g}^* \times \mathfrak{g} \to \mathbb{R}$ denote the natural pairing between the Lie algebra $\mathfrak{g}$ and its dual space $\mathfrak{g}^*$. Following \cite{MaRa2013}, we define the dual operator $\operatorname{Ad}^*_g : \mathfrak{g}^* \to \mathfrak{g}^*$ as the algebraic dual of $\operatorname{Ad}_g$, which satisfies:
	\begin{equation}
		\left\langle \operatorname{Ad}^*_g \mu, \alpha \right\rangle = 
		\left\langle \mu, \operatorname{Ad}_g \alpha \right\rangle 
		\label{Ad_star_def} 
	\end{equation}
	for any $\mu \in \mathfrak{g}^*$ and $\alpha \in \mathfrak{g}$. Since the assignment $g \mapsto \operatorname{Ad}^*_g$ yields a right action, the left coadjoint action of $G$ on $\mathfrak{g}^*$ is defined using the group inverse:
	\begin{equation}
		(g, \mu) \mapsto \operatorname{Ad}^*_{g^{-1}} \mu \,.
		\label{left_coadjoint_action}
	\end{equation}
	Differentiating this representation yields the infinitesimal coadjoint action $\operatorname{ad}^*_\xi$ for any $\xi \in \mathfrak{g}$ and $\mu \in \mathfrak{g}^*$: 
	\begin{equation}
		\left\langle \operatorname{ad}^*_\xi \mu, \alpha \right\rangle = 
		\left\langle \mu, \operatorname{ad}_\xi \alpha \right\rangle = \left\langle \mu, [\xi, \alpha] \right\rangle . 
		\label{ad_star_def} 
	\end{equation}
	Calculating these operators ($\operatorname{AD}$, $\operatorname{Ad}$, $\operatorname{ad}$, $\operatorname{Ad}^*$, $\operatorname{ad}^*$) for concrete physical groups simplifies many aspects of geometric mechanics. For detailed derivations of these actions on common physical Lie groups, we refer the reader to \cite{Ho2011_pII}.
	\subsection{The Special Orthogonal Group $\mathrm{SO}(3)$}
	\label{subsec:so3_actions}
	
	The Special Orthogonal group in three dimensions is defined as:
	\begin{equation}
		\mathrm{SO}(3) = \{ R \in \mathbb{R}^{3 \times 3} \mid R^T R = \mathbb{I}_3 \, , \, \det(R) = 1 \} \,.
	\end{equation}
	The Lie algebra $\mathfrak{so}(3)$ consists of $3 \times 3$ skew-symmetric matrices. We identify $\mathfrak{so}(3)$ with $\mathbb{R}^3$ via the isomorphic ``hat'' map $\hat{\cdot}: \mathbb{R}^3 \to \mathfrak{so}(3)$:
	\begin{equation}
		\boldsymbol{\omega} = \begin{pmatrix} \omega_1 \\ \omega_2 \\ \omega_3 \end{pmatrix} \ \mapsto \ \hat{\boldsymbol{\omega}} = \begin{pmatrix} 0 & -\omega_3 & \omega_2 \\ \omega_3 & 0 & -\omega_1 \\ -\omega_2 & \omega_1 & 0 \end{pmatrix} \,,
	\end{equation}
	which satisfies $\hat{\boldsymbol{\omega}} \mathbf{x} = \boldsymbol{\omega} \times \mathbf{x}$ for any $\mathbf{x} \in \mathbb{R}^3$. We equip $\mathfrak{so}(3)$ with the standard inner product matching the Killing form scaling: $\langle \hat{\boldsymbol{\omega}}, \hat{\boldsymbol{\eta}} \rangle = \frac{1}{2} \operatorname{Tr}(\hat{\boldsymbol{\omega}}^T \hat{\boldsymbol{\eta}}) = \boldsymbol{\omega} \cdot \boldsymbol{\eta}$.
	
	\paragraph{Adjoint Operators}
	Let $R \in \mathrm{SO}(3)$ and $\boldsymbol{\omega}, \boldsymbol{\eta} \in \mathbb{R}^3$.
	\begin{itemize}
		\item \textbf{Conjugation ($\operatorname{AD}$):} For another group element $Q \in \mathrm{SO}(3)$, the conjugation action of $R$ is explicitly given by:
		\begin{equation}
			\operatorname{AD}_R Q = R Q R^T \,.
		\end{equation}
		Geometrically, if $Q$ represents a rotation by an angle $\theta$ about an axis $\mathbf{n}$, then $\operatorname{AD}_R Q$ represents a rotation by the same angle $\theta$ about the rotated axis $R\mathbf{n}$.
		\item \textbf{Adjoint Map ($\operatorname{Ad}$):} By identifying $\mathfrak{so}(3)$ with $\mathbb{R}^3$, the matrix conjugation $R \hat{\boldsymbol{\omega}} R^T = \widehat{R\boldsymbol{\omega}}$ yields:
		\begin{equation}
			\operatorname{Ad}_R \boldsymbol{\omega} = R \boldsymbol{\omega} \,.
			\label{eq:so3_Ad}
		\end{equation}
		\item \textbf{Infinitesimal Adjoint ($\operatorname{ad}$):} The Lie bracket in $\mathfrak{so}(3)$ corresponds to the standard cross product:
		\begin{equation}
			\operatorname{ad}_{\boldsymbol{\omega}} \boldsymbol{\eta} = [\hat{\boldsymbol{\omega}}, \hat{\boldsymbol{\eta}}] \ \Leftrightarrow \ \operatorname{ad}_{\boldsymbol{\omega}} \boldsymbol{\eta} = \boldsymbol{\omega} \times \boldsymbol{\eta} \,.
		\end{equation}
	\end{itemize}
	
	\paragraph{Coadjoint Operators}
	The dual space $\mathfrak{so}(3)^*$ is identified with $\mathbb{R}^3$. Let $\boldsymbol{\mu} \in \mathfrak{so}(3)^*$ represent the momentum vector.
	\begin{itemize}
		\item \textbf{Dual Adjoint Operator ($\operatorname{Ad}^*$):} Defined by the pairing $\langle \operatorname{Ad}^*_R \boldsymbol{\mu}, \boldsymbol{\omega} \rangle = \langle \boldsymbol{\mu}, \operatorname{Ad}_R \boldsymbol{\omega} \rangle$, we have:
		\begin{equation}
			\operatorname{Ad}^*_R \boldsymbol{\mu} = R^T \boldsymbol{\mu} \,.
		\end{equation}
		The corresponding left coadjoint action of $G$ on $\mathfrak{g}^*$ is given by $\operatorname{Ad}^*_{R^{-1}} \boldsymbol{\mu} = R \boldsymbol{\mu}$.
		\item \textbf{Infinitesimal Coadjoint ($\operatorname{ad}^*$):} Using the triple product identity $\boldsymbol{\mu} \cdot (\boldsymbol{\omega} \times \boldsymbol{\eta}) = (\boldsymbol{\mu} \times \boldsymbol{\omega}) \cdot \boldsymbol{\eta}$, the infinitesimal coadjoint action is:
		\begin{equation}
			\operatorname{ad}^*_{\boldsymbol{\omega}} \boldsymbol{\mu} = \boldsymbol{\mu} \times \boldsymbol{\omega} = -\boldsymbol{\omega} \times \boldsymbol{\mu} \,.
		\end{equation}
	\end{itemize}

	\subsection{The Special Euclidean Group $\mathrm{SE}(2)$}
	\label{subsec:se2_actions}
	
	The Special Euclidean group in two dimensions represents planar rigid body motions:
	\begin{equation}
		\mathrm{SE}(2) = \left\{ g = (R_\theta, \mathbf{x}) = \begin{pmatrix} R_\theta & \mathbf{x} \\ \mathbf{0}^T & 1 \end{pmatrix} \ \middle|\ R_\theta \in \mathrm{SO}(2), \ \mathbf{x} \in \mathbb{R}^2 \right\} \,,
	\end{equation}
	where $R_\theta = \begin{pmatrix} \cos\theta & -\sin\theta \\ \sin\theta & \cos\theta \end{pmatrix}$ and $\mathbf{x} = \begin{pmatrix} x \\ y \end{pmatrix}$. 
	The Lie algebra $\mathfrak{se}(2)$ is isomorphic to $\mathbb{R}^3$, with elements representing planar twist velocities $\xi = (\omega, \mathbf{v}) \in \mathbb{R} \times \mathbb{R}^2$. The matrix representation is:
	\begin{equation}
		\hat{\xi} = \begin{pmatrix} \omega \mathbb{J} & \mathbf{v} \\ \mathbf{0}^T & 0 \end{pmatrix} \,, \quad \text{where } \mathbb{J} = \begin{pmatrix} 0 & -1 \\ 1 & 0 \end{pmatrix} \,.
	\end{equation}
	The pairing between dual algebra element $\mu = (M, \mathbf{p}) \in \mathfrak{se}(2)^*$ and $\xi = (\omega, \mathbf{v}) \in \mathfrak{se}(2)$ is $\langle \mu, \xi \rangle = M\omega + \mathbf{p} \cdot \mathbf{v}$.
	
	\paragraph{Adjoint Operators}
	Let $g = (R_\theta, \mathbf{x}) \in \mathrm{SE}(2)$ and $\eta = (\omega, \mathbf{v}) \in \mathfrak{se}(2)$.
	\begin{itemize}
		\item \textbf{Conjugation ($\operatorname{AD}$):} For another group element $h = (R_\phi, \mathbf{y}) \in \mathrm{SE}(2)$, conjugation yields:
		\begin{equation}
			\operatorname{AD}_g h = g h g^{-1} = \left( R_\phi, \, R_\theta \mathbf{y} + (\mathbb{I}_2 - R_\phi)\mathbf{x} \right) \,.
			\label{eq:se2_AD}
		\end{equation}
		Because $\mathrm{SO}(2)$ is abelian, the rotational component $R_\theta R_\phi R_\theta^T$ simplifies directly to $R_\phi$.
		\item \textbf{Adjoint Map ($\operatorname{Ad}$):} Differentiating \eqref{eq:se2_AD} along a curve $h(t) = (R_{t\omega}, \mathbf{y}(t))$ passing through the identity with velocity $\eta = (\omega, \mathbf{v})$ at $t=0$ yields:
		\begin{equation}
			\operatorname{Ad}_g \eta = \left( \omega, \, R_\theta\mathbf{v} - \omega \mathbb{J}\mathbf{x} \right) \,.
		\end{equation}
		\item \textbf{Infinitesimal Adjoint ($\operatorname{ad}$):} For two algebra elements $\xi = (\omega_1, \mathbf{v}_1)$ and $\eta = (\omega_2, \mathbf{v}_2)$, the Lie bracket is:
		\begin{equation}
			\operatorname{ad}_{\xi} \eta = [\hat{\xi}, \hat{\eta}] \ \Leftrightarrow \ \operatorname{ad}_{\xi} \eta = \left( 0, \, \omega_1 \mathbb{J}\mathbf{v}_2 - \omega_2 \mathbb{J}\mathbf{v}_1 \right) \,.
		\end{equation}
	\end{itemize}
	
	\paragraph{Coadjoint Operators}
	Let $\mu = (M, \mathbf{p}) \in \mathfrak{se}(2)^*$ and $\xi = (\omega, \mathbf{v}) \in \mathfrak{se}(2)$.
	\begin{itemize}
		\item \textbf{Dual Adjoint Operator ($\operatorname{Ad}^*$):} Computing the dual operator via the pairing relation yields:
		\begin{equation}
			\operatorname{Ad}^*_g \mu = \left( M - (\mathbf{x} \times \mathbf{p})\cdot \mathbf{e}_3, \, R_\theta^T \mathbf{p} \right) \,,
		\end{equation}
		where $(\mathbf{x} \times \mathbf{p})\cdot \mathbf{e}_3 := \mathbf{p} \cdot \mathbb{J}\mathbf{x} = x p_y - y p_x$ is the 2D scalar cross product.
		\item \textbf{Infinitesimal Coadjoint ($\operatorname{ad}^*$):} Under the duality pairing, we find:
		\begin{equation}
			\operatorname{ad}^*_{\xi} \mu = \left( -(\mathbf{v} \times \mathbf{p})\cdot \mathbf{e}_3, \, -\omega \mathbb{J}\mathbf{p} \right) \,,
		\end{equation}
		where $(\mathbf{v} \times \mathbf{p})\cdot\mathbf{e}_3 = v_x p_y - v_y p_x$.
	\end{itemize}

	\subsection{The Special Euclidean Group $\mathrm{SE}(3)$}
	\label{subsec:se3_actions}
	
	The Special Euclidean group in three dimensions describes general spatial motions:
	\begin{equation}
		\mathrm{SE}(3) = \left\{ g = (R, \mathbf{x}) = \begin{pmatrix} R & \mathbf{x} \\ \mathbf{0}^T & 1 \end{pmatrix} \ \middle|\ R \in \mathrm{SO}(3), \ \mathbf{x} \in \mathbb{R}^3 \right\} \,.
	\end{equation}
	The Lie algebra $\mathfrak{se}(3)$ is isomorphic to $\mathbb{R}^6$. We identify any twist $\xi = (\boldsymbol{\omega}, \mathbf{v}) \in \mathbb{R}^3 \times \mathbb{R}^3$ with the matrix:
	\begin{equation}
		\hat{\xi} = \begin{pmatrix} \hat{\boldsymbol{\omega}} & \mathbf{v} \\ \mathbf{0}^T & 0 \end{pmatrix} \,.
	\end{equation}
	The dual space $\mathfrak{se}(3)^*$ is isomorphic to $\mathbb{R}^6$. For a wrench $\mu = (\mathbf{M}, \mathbf{p}) \in \mathfrak{se}(3)^*$, the pairing is given by $\langle \mu, \xi \rangle = \mathbf{M} \cdot \boldsymbol{\omega} + \mathbf{p} \cdot \mathbf{v}$.
	
	\paragraph{Adjoint Operators}
	Let $g = (R, \mathbf{x}) \in \mathrm{SE}(3)$ and $\eta = (\boldsymbol{\omega}, \mathbf{v}) \in \mathfrak{se}(3)$.
	\begin{itemize}
		\item \textbf{Conjugation ($\operatorname{AD}$):} For another group element $h = (Q, \mathbf{y}) \in \mathrm{SE}(3)$, conjugation yields:
		\begin{equation}
			\operatorname{AD}_g h = g h g^{-1} = \left( R Q R^T, \, R\mathbf{y} + (\mathbb{I}_3 - R Q R^T)\mathbf{x} \right) \,.
			\label{eq:se3_AD}
		\end{equation}
		\item \textbf{Adjoint Map ($\operatorname{Ad}$):} Differentiating \eqref{eq:se3_AD} along a curve $h(t) = (Q(t), \mathbf{y}(t))$ passing through the identity with velocity $\eta = (\boldsymbol{\omega}, \mathbf{v})$ at $t=0$ yields:
		\begin{equation}
			\operatorname{Ad}_g \eta = \left( R\boldsymbol{\omega}, \, R\mathbf{v} + \mathbf{x} \times (R\boldsymbol{\omega}) \right) \,,
		\end{equation}
		where we use the identity $R \hat{\boldsymbol{\omega}} R^T = \widehat{R\boldsymbol{\omega}}$ and $-(R\boldsymbol{\omega}) \times \mathbf{x} = \mathbf{x} \times (R\boldsymbol{\omega})$.
		\item \textbf{Infinitesimal Adjoint ($\operatorname{ad}$):} For two elements $\xi = (\boldsymbol{\omega}_1, \mathbf{v}_1)$ and $\eta = (\boldsymbol{\omega}_2, \mathbf{v}_2)$, the Lie bracket is:
		\begin{equation}
			\operatorname{ad}_{\xi} \eta = \left( \boldsymbol{\omega}_1 \times \boldsymbol{\omega}_2, \, \boldsymbol{\omega}_1 \times \mathbf{v}_2 + \mathbf{v}_1 \times \boldsymbol{\omega}_2 \right) \,.
		\end{equation}
	\end{itemize}
	
	\paragraph{Coadjoint Operators}
	Let $\mu = (\mathbf{M}, \mathbf{p}) \in \mathfrak{se}(3)^*$ and $\xi = (\boldsymbol{\omega}, \mathbf{v}) \in \mathfrak{se}(3)$.
	\begin{itemize}
		\item \textbf{Dual Adjoint Operator ($\operatorname{Ad}^*$):} Resolving the pairing equation gives:
		\begin{equation}
			\operatorname{Ad}^*_g \mu = \left( R^T \mathbf{M} - R^T\mathbf{x} \times R^T\mathbf{p}, \, R^T \mathbf{p} \right) \,,
		\end{equation}
		using the rotation-invariance of cross products $R^T(\mathbf{p} \times \mathbf{x}) = R^T\mathbf{p} \times R^T\mathbf{x}$.
		\item \textbf{Infinitesimal Coadjoint ($\operatorname{ad}^*$):} Splitting the dual pairing across angular and linear coordinates yields:
		\begin{equation}
			\hspace{-1cm} 
			\operatorname{ad}^*_{\xi} \mu = \left( \mathbf{M} \times \boldsymbol{\omega} - \mathbf{v} \times \mathbf{p}, \, \mathbf{p} \times \boldsymbol{\omega} \right) = \left( -\boldsymbol{\omega} \times \mathbf{M} - \mathbf{v} \times \mathbf{p}, \, -\boldsymbol{\omega} \times \mathbf{p} \right) \,.
		\end{equation}
	\end{itemize}
	
	\section{Mathematical Foundations of Lie-Poisson Systems} 
	\label{app:math_Lie_Poisson}
	
	\paragraph{Dynamics on Lie Groups} Let the configuration space of a mechanical system be modeled by a Lie group $G$. The trajectories of the system under a Lagrangian $L(g, \dot{g})$ are determined by Hamilton's variational principle. Provided that the fiber derivative (the mapping from velocities to momenta) $\dot{g} \mapsto p(g, \dot{g}) = \frac{\partial L}{\partial \dot{g}}(g, \dot{g})$ is a diffeomorphism for all $g \in G$, the system can be represented equivalently in terms of canonical Hamiltonian dynamics. This yields the Hamiltonian function $H(g,p) = \left\langle p, \dot{g}(g,p) \right\rangle - L(g, \dot{g}(g,p))$ defined on the cotangent bundle $T^*G$. 
	
	\paragraph{Symmetries on Lie Groups} We assume the Lagrangian exhibits left-invariance under the action of the group on itself, such that $L(hg, T_g L_h \dot{g}) = L(g, \dot{g})$ for every $h \in G$. This invariance immediately translates to the Hamiltonian, satisfying $H(hg, hp) = H(g, p)$ for all $h \in G$. Consequently, $H$ is entirely characterized by its values evaluated at the group identity $g=e$, allowing us to express the Hamiltonian as $H(g,p) = H(e, g^{-1} p) =: h(g^{-1} p)$. 
	Here, the quantity $\mu = g^{-1} p := T^*_e L_g(p)$ defines the reduced (or body-fixed) momentum, which resides in the dual space $\mathfrak{g}^*$ of the Lie algebra $\mathfrak{g}$ associated with $G$. The mapping $h: \mathfrak{g}^* \rightarrow \mathbb{R}$ represents the reduced Hamiltonian function. 
	Given this symmetry, the evolution of the full phase-space variables $(g(t), p(t))$ can be projected onto the lower-dimensional dual algebra space $\mathfrak{g}^*$ solely in terms of the reduced momentum $\mu(t) = g(t)^{-1} p(t)$. This reduction yields the Lie-Poisson equations for the evolution of $\mu(t)$.
	
	\paragraph{The Lie-Poisson Equations} Formulated in coordinate-free notation, the Lie-Poisson equations for a left-invariant system read:
	\begin{equation}\label{LP_intrinsic_left_inv} 
		\dot{\mu} = \operatorname{ad}^*_{ \frac{\partial h}{\partial \mu} } \mu \,,
	\end{equation} 
	where $\operatorname{ad}^*_{\xi}: \mathfrak{g}^* \rightarrow \mathfrak{g}^*$ denotes the infinitesimal coadjoint operator. It is defined through the pairing with the adjoint operator $\operatorname{ad}_{\xi}\eta = [\xi, \eta]$ via $\langle\operatorname{ad}^*_{\xi} \mu, \eta\rangle = \langle \mu, \operatorname{ad}_{\xi} \eta \rangle$ for any $\xi, \eta \in \mathfrak{g}$ and $\mu \in \mathfrak{g}^*$. The temporal evolution of an arbitrary smooth function $f(\mu)$ along the trajectories of \eqref{LP_intrinsic_left_inv} is governed by $\dot{f} = \{f, h\}$, where the underlying Poisson structure is determined by the minus Lie-Poisson bracket:
	\begin{equation} 
		\left\{ f, h \right\} = - \left\langle \mu, \left[ \frac{\partial f}{\partial \mu} , \frac{\partial h}{\partial \mu} \right] \right\rangle \,.
		\label{LP_bracket} 
	\end{equation} 
	While this bracket is non-canonical, it inherits its properties directly from the canonical Poisson bracket $\{\cdot, \cdot\}_{\rm can}$ on $T^*G$. Specifically, they are linked by:
	\[
	\left\{ f \circ \pi , h \circ \pi \right\}_{\rm can} = \left\{ f, h \right\} \circ \pi
	\]
	where $\pi: T^*G \rightarrow \mathfrak{g}^*$ is the projection map $\pi(g,p) = g^{-1} p = \mu$ that maps the full momentum to its body-fixed representation. Thus, $\pi$ acts as a Poisson map connecting $\{\cdot, \cdot\}_{\rm can}$ and $\{\cdot, \cdot\}$. This reduction scheme, which is a classic demonstration of Poisson reduction, explains the source of the Lie-Poisson brackets across diverse physical frameworks. 
	
	For systems that exhibit right invariance instead of left invariance, the reduced momentum is defined as $\mu = p g^{-1}$, which introduces a sign change on the right-hand sides of both \eqref{LP_eqs} and \eqref{LP_bracket} (leading to a plus Lie-Poisson bracket and a minus sign in the dynamical equation).
	
	\paragraph{The Flow of Lie-Poisson Systems} 
	Let $\operatorname{Ad}_g: \mathfrak{g} \rightarrow \mathfrak{g}$ denote the adjoint action of $G$ on its Lie algebra, defined by $\operatorname{Ad}_g \xi := \left. \frac{d}{d\varepsilon}\right|_{\varepsilon=0} g c_{\varepsilon} g^{-1}$, where $c_{\varepsilon} \in G$ is a curve satisfying $c_0 = e$ and $\dot{c}_0 = \xi$. Let $\operatorname{Ad}^*_g: \mathfrak{g}^* \rightarrow \mathfrak{g}^*$ represent the dual of the adjoint operator $\operatorname{Ad}_g$, defined via the pairing $\langle \operatorname{Ad}^*_g \mu , \xi \rangle = \langle \mu , \operatorname{Ad}_g\xi \rangle$. To ensure a proper left group action on the dual space, the left coadjoint action is defined by $(g, \mu) \mapsto \operatorname{Ad}^*_{g^{-1}} \mu$. For any curve $g(t) \in G$, and for any constant $p_0 \in \mathfrak{g}^*$, the time derivative of the dual adjoint operator satisfies:
	\[
	\frac{d}{dt} \operatorname{Ad}^*_{g(t)} p_0 = \operatorname{ad}^*_{ g(t)^{-1} \dot{g}(t)} \operatorname{Ad}^*_{g(t)} p_0\,.   
	\]
	A direct consequence of this identity is that a solution with conserved spatial momentum $p_0$ of the Lie-Poisson system \eqref{LP_intrinsic_left_inv} can be explicitly parameterized as:
	\[
	\phi_t(p_0) = \operatorname{Ad}^*_{g(t)} p_0\, ,  \quad\text{where}\quad g(t)^{-1} \dot{g}(t) = \frac{\partial h}{\partial \mu}(t) \,.
	\]
	This expression guarantees that the dynamical flow preserves the coadjoint orbits $\mathcal{O} = \{ \operatorname{Ad}^*_{g} p_0 \mid g \in G\} \subset \mathfrak{g}^*$ at all times. 
	Conversely, suppose $(\mu(t),g(t))$ is a solution of the extended Lie-Poisson system. Then, the spatial momentum $p_0 := \operatorname{Ad}^*_{g(t)^{-1}} \mu(t)$ is conserved for all times, satisfying:
	\begin{equation}
		\frac{d}{dt} \operatorname{Ad}^*_{g(t)^{-1}} \mu(t) = 0 \, . 
		\label{spatial_mom_conservation}
	\end{equation}
	This conservation of $p_0$ is a direct consequence of Noether's theorem associated with the left-invariance of the system. In practice, the coadjoint-dual relation $p_0 = \operatorname{Ad}^*_{g^{-1}} \mu$ maps the body-fixed momentum $\mu$ to the conserved spatial momentum $p_0$, while its inverse $\mu = \operatorname{Ad}^*_g p_0$ reconstructs the body-fixed momentum from this conserved quantity.
	
	
	\section{Reference Formulas for Specific Lie Groups}
	\label{subsec:explicit_formulas}
	
	To facilitate practical implementation, we compile the explicit forms of the conserved spatial momentum $p_0 = p_0(g, \mu)$ and its corresponding inversion $\mu = \mu(g, p_0)$ for three standard physical Lie groups.
	
	\subsection{The Special Orthogonal Group $\mathrm{SO}(3)$}
	\label{subsec:so3_actions_explicit}
	
	For rotational dynamics, let $g = R \in \mathrm{SO}(3)$ be the rotation matrix satisfying $R^T R = \mathbb{I}_3$ and $\det(R) = 1$. The body momentum is represented by $\boldsymbol{\mu} \in \mathbb{R}^3$ (body angular momentum), and the conserved spatial momentum is $\boldsymbol{p}_0 \in \mathbb{R}^3$ (spatial angular momentum).
	\begin{itemize}
		\item \textbf{Conserved Spatial Momentum $\boldsymbol{p}_0(R, \boldsymbol{\mu})$:}
		\begin{equation}
			\boldsymbol{p}_0 = R \boldsymbol{\mu}
		\end{equation}
		\item \textbf{Inversion to Body Momentum $\boldsymbol{\mu}(R, \boldsymbol{p}_0)$:}
		\begin{equation}
			\boldsymbol{\mu} = R^T \boldsymbol{p}_0
		\end{equation}
	\end{itemize}
	
	\subsection{The Special Euclidean Group $\mathrm{SE}(2)$}
	\label{subsec:se2_actions_explicit}
	
	For planar rigid motions, let $g = (R_\theta, \mathbf{x}) \in \mathrm{SE}(2)$, where $R_\theta \in \mathrm{SO}(2)$ is the rotation matrix of angle $\theta$, and $\mathbf{x} \in \mathbb{R}^2$ is the translation vector. The body momentum is $\mu = (J, \mathbf{P}) \in \mathbb{R} \times \mathbb{R}^2$ (body angular and linear momentum), and the conserved spatial momentum is $p_0 = (J_0, \mathbf{P}_0) \in \mathbb{R} \times \mathbb{R}^2$.
	\begin{itemize}
		\item \textbf{Conserved Spatial Momentum $p_0(g, \mu)$:}
		\begin{align}
			\mathbf{P}_0 &= R_\theta \mathbf{P} \\
			J_0 &= J + (\mathbf{x} \times \mathbf{P}_0 ) \cdot \mathbf{e}_z\, , 
		\end{align}
		where $(\mathbf{x} \times \mathbf{P}_0 ) \cdot \mathbf{e}_z:= x P_{0,y} - y P_{0,x}$ represents the standard 2D cross product.
		\item \textbf{Inversion to Body Momentum $\mu(g, p_0)$:}
		\begin{align}
			\mathbf{P} &= R_\theta^T \mathbf{P}_0 \\
			J &= J_0 - (\mathbf{x} \times \mathbf{P}_0)\cdot \mathbf{e}_z \, . 
		\end{align}
	\end{itemize}
	
	\subsection{The Special Euclidean Group $\mathrm{SE}(3)$}
	\label{subsec:se3_actions_explicit}
	
	For full 3D rigid motions, let $g = (R, \mathbf{x}) \in \mathrm{SE}(3)$, where $R \in \mathrm{SO}(3)$ is the rotation matrix and $\mathbf{x} \in \mathbb{R}^3$ is the translation vector. The body momentum is $\mu = (\mathbf{J}, \mathbf{P}) \in \mathbb{R}^3 \times \mathbb{R}^3$ (body angular and linear momentum), and the conserved spatial momentum is $p_0 = (\mathbf{J}_0, \mathbf{P}_0) \in \mathbb{R}^3 \times \mathbb{R}^3$.
	\begin{itemize}
		\item \textbf{Conserved Spatial Momentum $p_0(g, \mu)$:}
		\begin{align}
			\mathbf{P}_0 &= R \mathbf{P} \\
			\mathbf{J}_0 &= R \mathbf{J} + \mathbf{x} \times \mathbf{P}_0\, . 
		\end{align}
		\item \textbf{Inversion to Body Momentum $\mu(g, p_0)$:}
		\begin{align}
			\mathbf{P} &= R^T \mathbf{P}_0 \\
			\mathbf{J} &= R^T ( \mathbf{J}_0 - \mathbf{x} \times \mathbf{P}_0 )\, . 
		\end{align}
	\end{itemize}
\end{document}